\documentclass{article}
 
\usepackage{graphicx} 
\usepackage{subfigure}
\usepackage{xargs}
\usepackage[numbers]{natbib}
\usepackage[latin1]{inputenc}

\usepackage{appendix,color}
\RequirePackage{amsmath}
\RequirePackage{amssymb}

\newcommand{\BEAS}{\begin{eqnarray*}}
\newcommand{\EEAS}{\end{eqnarray*}}
\newcommand{\BEA}{\begin{eqnarray}}
\newcommand{\EEA}{\end{eqnarray}}
\newcommand{\BEQ}{\begin{equation}}
\newcommand{\EEQ}{\end{equation}}
\newcommand{\BIT}{\begin{itemize}}
\newcommand{\EIT}{\end{itemize}}
\newcommand{\BNUM}{\begin{enumerate}}
\newcommand{\ENUM}{\end{enumerate}}
\newcommand{\BA}{\begin{array}}
\newcommand{\EA}{\end{array}}

\newcommand{\tr}{\mathop{ \rm tr}}

\newcommand{\idm}{I}
\newcommand{\rb}{\mathbb{R}}

\newcommand{\BlackBox}{\rule{1.5ex}{1.5ex}}  

\newenvironment{proof}{\par\noindent{\bf Proof.\ }}{\hfill\BlackBox\\[2mm]}
\newtheorem{lemma}{Lemma}
\newtheorem{theorem}{Theorem}
\newtheorem{proposition}{Proposition}
\newtheorem{corollary}{Corollary}

\newcommand{\mysec}[1]{Section~\ref{sec:#1}}
\newcommand{\eq}[1]{Eq.~(\ref{eq:#1})}
\newcommand{\myfig}[1]{Figure~\ref{fig:#1}}
\newcommand{\defeq}{\stackrel{\rm def}{=}}
\newcommand{\rmd}{\mathrm{d}}

\newcommand{\eqsp}{\;}

\def \E{{\mathbb E}}
\def \P{{\mathbb P}}

\def \E{{\mathbb E}}
\def \P{{\mathbb P}}
\def \F{{\mathcal F}}
\def \H{{\mathcal H}}

\newcommandx\CE[4][1=,2=]{\ensuremath{{\mathbb E}_{#1}^{#2}\left[  #3 \bigm| #4 \right]}}
\newcommandx{\PVar}[1][1=]{\ensuremath{\operatorname{Var}_{#1}}}
\newcommandx{\PCov}[1][1=]{\ensuremath{\operatorname{Cov}}_{#1}}

\title{Non-strongly-convex smooth stochastic approximation
\\ with convergence rate $O(1/n)$}

\author{
Francis Bach\\
INRIA - Sierra Project-Team\\
Ecole Normale Sup\'erieure \\
Paris, France\\
\texttt{francis.bach@ens.fr} \\
\and
Eric Moulines\\
LTCI\\
Telecom ParisTech \\
Paris, France\\
\texttt{eric.moulines@enst.fr} 
}

\date{\today}

\begin{document}
\maketitle

\begin{abstract}

We consider the stochastic approximation problem where a convex function has to be minimized, given only the knowledge of unbiased estimates of its gradients at certain points, a framework which includes machine learning methods based on the minimization of the empirical risk. We focus on problems without strong convexity, for which all previously known algorithms achieve a convergence rate for function values of $O(1/\sqrt{n})$. We consider and analyze two algorithms that achieve a rate of $O(1/n)$ for classical   supervised learning problems. For least-squares regression, we show that   \emph{averaged} stochastic gradient descent \emph{with constant step-size} achieves the desired rate.  For logistic regression, this is achieved by a simple novel stochastic gradient algorithm that (a) constructs successive local quadratic approximations of the loss functions, while (b) preserving the same running time complexity as stochastic gradient descent. For these algorithms, we provide a non-asymptotic analysis of the generalization error (in expectation, and also in high probability for least-squares), and run extensive experiments
on standard machine learning benchmarks
showing that they often outperform existing approaches.

\end{abstract}

\section{Introduction}

Large-scale machine learning problems are becoming ubiquitous in many areas of science and engineering. Faced with large amounts of data, practitioners typically prefer algorithms that process each observation  only once,  or a few times. Stochastic approximation algorithms such as stochastic gradient descent (SGD) and its variants, although introduced more than 60 years ago~\cite{robbins1951stochastic}, still remain the most widely used
and studied method in this context (see, e.g.,~\cite{polyak1992acceleration,bottou-bousquet-2008b,shalev2007pegasos,nemirovski2009robust,gradsto,adagrad}).

We consider minimizing convex functions $f$, defined on a Euclidean space $\H$, given by $f(\theta) = \E \big[ \ell(y, \langle \theta, x \rangle ) \big]$, where  $(x,y)\in \H \times \rb$  denotes the data and $\ell$ denotes a loss function that is convex with respect to the second variable. This includes logistic and least-squares regression. In the stochastic approximation framework, independent and identically distributed pairs $(x_n,y_n)$ are observed sequentially and the predictor defined by $\theta$ is updated after each pair is seen.

 We partially understand the properties of $f$ that affect the problem difficulty. \emph{Strong convexity} (i.e., when $f$ is twice differentiable, a uniform strictly positive lower-bound  $\mu$ on Hessians of $f$) is a key property. Indeed, after $n$ observations and with the proper step-sizes, averaged SGD achieves the rate of $O(1/ \mu n)$ in the strongly-convex case~\cite{nemirovski2009robust,shalev2007pegasos}, while it achieves only $O(1/\sqrt{n})$ in the non-strongly-convex case~\cite{nemirovski2009robust}, with matching lower-bounds~\cite{nemirovsky1983problem,agarwal2012information}.

The main issue with strong convexity is that typical machine learning problems are high dimensional and have correlated variables so that the strong convexity constant $\mu$ is zero or very close to zero, and in any case smaller than $O(1/\sqrt{n})$. This then makes the non-strongly convex methods better.
In this paper, we aim at obtaining algorithms that may deal with arbitrarily small strong-convexity constants, but still achieve a rate of $O(1/n)$.

\emph{Smoothness} plays a central role in the context of deterministic optimization. The known convergence rates for smooth optimization are   better than for non-smooth optimization~(e.g., see~\cite{nesterov2004introductory}). However, for stochastic optimization the use of smoothness only leads to improvements on constants (e.g., see~\cite{lan2010optimal}) but not on the rate itself, which remains $O(1/\sqrt{n})$ for non-strongly-convex problems.

 We show that for the square loss and for the logistic loss, we may use the smoothness of the loss and obtain algorithms that have a convergence rate of $O(1/n)$ without any strong convexity assumptions. More precisely, for least-squares regression, we show  in \mysec{square} that   \emph{averaged} stochastic gradient descent \emph{with constant step-size} achieves the desired rate. For logistic regression this is achieved by a novel stochastic gradient algorithm that (a) constructs successive local quadratic approximations of the loss functions, while (b) preserving the same running time complexity as stochastic gradient descent (see \mysec{logistic}). For these algorithms, we provide a non-asymptotic analysis of their generalization error (in expectation, and also in high probability for least-squares), and run extensive experiments on standard machine learning benchmarks showing in \mysec{experiments} that they often outperform existing approaches.

\section{Constant-step-size least-mean-square algorithm}
\label{sec:square}

In this section, we consider   stochastic approximation for least-squares regression, where SGD is often referred to as the least-mean-square (LMS) algorithm. The novelty of our convergence result is the use of the constant step-size with averaging, leading to $O(1/n)$ rate without strong convexity.

\subsection{Convergence in expectation}
\label{sec:exp}

We make the following assumptions:
 
\begin{list}{\labelitemi}{\leftmargin=1.9em}
   \addtolength{\itemsep}{-.00\baselineskip}
\item[\textbf{(A1)}]  $\H$ is a $d$-dimensional Euclidean space, with $d \geqslant 1$.

\item[\textbf{(A2)}] The observations $(x_n,z_n) \in \H \times \H$ are independent and identically  distributed.

\item[\textbf{(A3)}] $\E \| x_n \|^2$  and  $\E \| z_n\|^2$ are finite. Denote by $H = \E ( x_n \otimes x_n)$ the covariance operator from $\H$ to $\H$. Without loss of generality, $H$ is assumed invertible (by projecting onto the minimal subspace where $x_n$ lies almost surely). However, its eigenvalues may be arbitrarily small.

\item[\textbf{(A4)}] The global minimum of $f(\theta) = (1/2) \E \big[  \langle \theta, x_n \rangle^2 - 2 \langle \theta, z_n \rangle \big]$ is attained at a certain $\theta_\ast \in \H$. We denote by $\xi_n = z_n - \langle \theta_\ast, x_n \rangle x_n $ the residual. We have
$\E \big[ \xi_n \big]= 0$, but in general, it is not true that $\CE{\xi_n}{ x_n} = 0$ (unless the model is well-specified).

\item[\textbf{(A5)}]  We study the stochastic gradient (a.k.a.~least mean square) recursion defined as
\BEQ
\label{eq:SGD-basic-recursion}
\theta_n = \theta_{n-1} - \gamma (    \langle \theta_{n-1}, x_n \rangle   x_n  - z_n)
= ( \idm - \gamma x_n \otimes x_n) \theta_{n-1} + \gamma z_n,
\EEQ started from  $\theta_0 \in \H$.
 We also consider the averaged iterates $\bar{\theta}_n = (n+1)^{-1} \sum_{k=0}^{n} \theta_k$.

 \item[\textbf{(A6)}] There exists $R >0$ and $\sigma >0$ such that
   $\E \big[ \xi_n \otimes \xi_n \big]\preccurlyeq \sigma^2 H$ and $
 \E \big( \|x_n\|^2 x_n \otimes x_n \big) \preccurlyeq R^2 H$,
where $\preccurlyeq$ denotes the the order between self-adjoint operators, i.e., $A \preccurlyeq B$ if and only if $B-A$ is positive semi-definite.
\end{list}

\paragraph{Discussion of assumptions.} Assumptions \textbf{(A1-5)} are standard in stochastic approximation~(see, e.g.,~\cite{kushner:yin:2003,gradsto}). Note that for least-squares problems, $z_n$ is of the form $y_n x_n$, where $y_n \in \rb$ is the response to be predicted as a linear function of $x_n$. We consider a slightly more general case than least-squares because we will need it for the quadratic approximation of the logistic loss in \mysec{log}. Note that in assumption \textbf{(A4)}, we do not assume that the model is well-specified.

Assumption \textbf{(A6)} is true for least-square regression with almost surely bounded data, since, if $\| x_n \|^2 \leqslant R^2$ almost surely, then $ \E \big( \|x_n\|^2 x_n \otimes x_n \big) \preccurlyeq \E \big( R^2 x_n \otimes x_n \big)  = R^2 H$; a similar inequality holds for the output variables $y_n$. Moreover, it also holds for data with infinite supports, such as Gaussians or mixtures of Gaussians
(where all covariance matrices of the mixture components are lower and upper bounded by a constant times the same matrix).
 Note that the finite-dimensionality assumption could be relaxed, but this would require notions similar to degrees of freedom~\cite{gu2002smoothing}, which is outside of the scope of this paper.

The goal of this section is to provide a bound on the expectation
$	
\E \big[ f( \bar{\theta}_n) - f(\theta_\ast) \big] ,
$
that (a) does not depend on the smallest non-zero eigenvalue of $H$ (which could be arbitrarily small) and (b) still scales as $O(1/n)$.

\begin{theorem}
\label{theo:lms}
Assume \textbf{(A1-6)}. For any constant step-size $\gamma < \frac{1}{R^2}$, we have

\BEQ
\label{eq:lms}
 \E \big[ f( \bar{\theta}_{n-1} ) - f(\theta_\ast) \big]
 \leqslant  \frac{1}{2 n} \bigg[
\frac{  \sigma \sqrt{d}  }{ 1- \sqrt{ \gamma R^2 }}+
 {R \| \theta_0 - \theta_\ast \| }   \frac{1}{  \sqrt{\gamma R^2}}  \bigg]^2.
\EEQ

When $\gamma = 1/(4R^2)$, we obtain $
 \E \big[ f( \bar{\theta}_{n-1} ) - f(\theta_\ast) \big]
 \leqslant
  \frac{2}{  n} \Big[   \sigma \sqrt{d}   + {R \| \theta_0 - \theta_\ast \| }   \Big]^2.
$
\end{theorem}

\paragraph{Proof technique.} We adapt and extend a proof technique from~\cite{aguech2000perturbation} which is based on non-asymptotic expansions in powers of $\gamma$. We also use a result from~\cite{polyak1992acceleration} which studied the recursion in \eq{SGD-basic-recursion}, with $x_n \otimes x_n$ replaced by its expectation $H$. See the appendix for details.

\paragraph{Optimality of bounds.} Our bound in \eq{lms} leads to a rate of $O(1/n)$, which is known to be optimal for least-squares regression (i.e., under reasonable assumptions, no algorithm, even  more complex than averaged SGD can have a better dependence in $n$)~\cite{tsybakov2003optimal}. The term $\sigma^2 d / n$ is also unimprovable.

\paragraph{Initial conditions.} If $\gamma$ is small, then the initial condition is forgotten more slowly. Note that with additional strong convexity assumptions, the initial condition would be forgotten faster (exponentially fast without averaging), which is one of the traditional uses of constant-step-size LMS~\cite{macchi1995adaptive}.

\paragraph{Specificity of constant step-sizes.} The non-averaged iterate sequence $(\theta_n)$ is a homogeneous Markov chain; under appropriate technical conditions,
this Markov chain has a unique stationary (invariant) distribution and  the sequence of iterates $(\theta_n)$ converges in distribution to this invariant distribution; see \citep[Chapter~17]{meyn:tweedie:2009}. Denote by $\pi_\gamma$ the invariant distribution. Assuming that the Markov Chain is Harris recurrent, the ergodic theorem for Harris Markov chain shows that $\bar{\theta}_{n-1} = n^{-1} \sum_{k=0}^{n-1} \theta_k$ converges almost-surely to $\bar{\theta}_\gamma \defeq \int \theta \pi_\gamma(\rmd \theta)$, which is the mean of the stationary distribution. Taking the expectation on both side of \eq{SGD-basic-recursion},
we get $\E[ \theta_n ] - \theta_* = (\idm - \gamma H) ( \E[ \theta_{n-1} ] - \theta_*)$, which shows, using that $\lim_{n \to \infty} \E[\theta_n]= \bar{\theta}_\gamma$ that $H \bar{\theta}_\gamma = H \theta_*$ and therefore $\bar{\theta}_\gamma= \theta_*$ since $H$ is invertible. Under slightly stronger assumptions, it can be shown that
\[
\textstyle
\lim_{n \to \infty} n \E [ (\bar{\theta}_n - \theta_*)^2]= \PVar[\pi_\gamma](\theta_0) + 2 \sum_{k=1}^\infty \PCov[\pi_\gamma](\theta_0,\theta_k) \eqsp,
\]
where $\PCov[\pi_\gamma](\theta_0,\theta_k)$ denotes the covariance of $\theta_0$ and $\theta_k$ when the Markov chain is started from stationarity. This implies that $\lim_{n \to \infty} n \E[ f(\bar{\theta}_n) - f(\theta_*)]$ has a finite limit. Therefore, this interpretation explains  why the averaging produces a sequence of estimators which converges to the solution~$\theta_*$ pointwise, and that the rate of convergence of $\E[ f(\theta_n) - f(\theta_*)]$ is of order $O(1/n)$. Note that for other losses than quadratic, the same properties hold \emph{except} that the mean under the stationary distribution does not coincide with $\theta_\ast$ and its distance to $\theta_\ast$ is typically of order $\gamma^2$ (see \mysec{logistic}).

\subsection{Convergence in higher orders}

\label{sec:high}

   We are now going to consider an extra assumption in order to bound the $p$-th moment of the excess risk and then get a high-probability bound. Let $p$ be a real number greater than $1$.
   \begin{list}{\labelitemi}{\leftmargin=1.9em}
   \addtolength{\itemsep}{-.05\baselineskip}
   \item[\textbf{(A7)}] There exists $R >0$, $
   \kappa>0$ and $\tau \geqslant \sigma >0$ such that, for all $n \geqslant 1$, $\| x_n\|^2 \leqslant R^2$ \mbox{ a.s.}, and
 \BEQ  \E \| \xi_n\|^p \leqslant \tau^p R^p   \quad \mbox{ and  } \quad  \E \big[ \xi_n \otimes \xi_n \big]\preccurlyeq \sigma^2 H  , \EEQ
\BEQ
\label{eq:kurtosis}
\text{$\forall z \in \H$},  \quad  \E \langle z, x_n \rangle^4 \leqslant \kappa \langle z, H z \rangle^2.
\EEQ
\end{list}
The last condition in \eq{kurtosis} says that the \emph{kurtosis} of the projection of the covariates $x_n$ on any direction $z \in \H$ is bounded. Note that computing the constant~$\kappa$ happens to be equivalent to the  optimization problem solved by the FastICA algorithm~\cite{hyvarinen1997fast}, which thus provides an estimate of $\kappa$. In Table~\ref{tab:data}, we provide such an estimate for the non-sparse datasets which we have used in experiments, while we consider only directions $z$ along the axes for high-dimensional  sparse datasets. For these datasets where a given variable is equal to zero except for a few observations, $\kappa$ is typically quite large. Adapting and analyzing normalized LMS techniques~\cite{nmls} to this set-up is likely to improve the theoretical robustness of the algorithm (but note that results in expectation from Theorem~\ref{theo:lms} do not use $\kappa$).
The next theorem provides a bound for the $p$-th moment of the excess risk.
\begin{theorem}
\label{theo:lmsp}

Assume \textbf{(A1-7)}. For any real $p \geqslant 1$, and for a   step-size $\gamma \leqslant 1/(12p \kappa R^2)$, we have:

\BEQ
\big( \E \big| f( \bar{\theta}_{n-1} ) - f(\theta_\ast) \big|^p \big)^{1/p}
 \leqslant  \frac{  p }{2 n }  \bigg(
 7 \tau \sqrt{d}  +  R \| \theta_0 - \theta_\ast\| \sqrt{ 3 + \frac{2}{\gamma p R^2} }
\bigg)^2.
  \EEQ
For $\gamma = 1/(12p \kappa R^2)$, we get:
$
\big( \E \big| f( \bar{\theta}_{n-1} ) - f(\theta_\ast) \big|^p \big)^{1/p}
 \leqslant  \frac{  p }{2 n }  \big(
 7 \tau \sqrt{d}  + 6 \sqrt{ \kappa} R \| \theta_0 - \theta_\ast\|
\big)^2
.
  $
\end{theorem}

Note that to control the $p$-th order moment, a smaller step-size is needed, which scales as $1/p$. We can now provide a high-probability bound; the tails decay polynomially as $1/(n \delta^{12 \gamma  \kappa R^2})$ and the smaller the step-size $\gamma$, the lighter the tails.
\begin{corollary}
\label{cor:bound}
For any step-size such that $\gamma \leqslant 1/(12 \kappa R^2)$, any $\delta \in (0,1)$,

\BEQ
\P \bigg(
f( \bar{\theta}_{n-1} ) - f(\theta_\ast)  \geqslant \frac{1}{n \delta^{12 \gamma  \kappa R^2}} \frac{\big[ 7 \tau \sqrt{d} + R \| \theta_0 - \theta_\ast\| ( \sqrt{3} + \sqrt{24 \kappa} ) \big]^2}{24 \gamma \kappa R^2 }
\bigg) \leqslant \delta \eqsp .
\EEQ
\end{corollary}

\section{Beyond least-squares: M-estimation}
\label{sec:logistic}

In \mysec{square}, we have shown that for least-squares regression, averaged SGD achieves a convergence rate  of $O(1/n)$ with no assumption regarding strong convexity. For all losses,  with a constant step-size $\gamma$,  the stationary distribution $\pi_\gamma$ corresponding to the homogeneous Markov chain $(\theta_n)$ does always satisfy
 $\int f'(\theta) \pi_\gamma(\rmd \theta)= 0$, where $f$ is the generalization error. When the gradient $f'$ is linear (i.e., $f$ is quadratic), then this implies that $f' ( \int \theta \pi_\gamma(\rmd \theta)) \!=\! 0$, i.e., the averaged recursion converges pathwise to $\bar{\theta}_\gamma = \int \theta \pi_\gamma(\rmd \theta)$ which coincides with the optimal value $\theta_\ast$ (defined through $f'(\theta_\ast)\!=\!0$).
When the gradient~$f'$ is no longer linear, then $\int f'(\theta) \pi_\gamma(\rmd \theta) \neq f'(\int \theta \pi_\gamma(\rmd \theta) )$. Therefore, for general $M$-estimation problems we should
expect that the averaged sequence still converges at rate $O(1/n)$ to the mean of the stationary distribution $\bar{\theta}_\gamma$, but not to the optimal predictor~$\theta_\ast$. Typically, the average distance between $\theta_n$ and $\theta_\ast$ is of order $\gamma$ (see \mysec{simulations} and~\cite{nedic2001convergence}), while for the averaged iterates that converge pointwise to $\bar{\theta}_\gamma$, it is of order
$\gamma^2$ for strongly convex problems under some additional smoothness conditions on the loss functions (these are satisfied, for example, by the logistic loss~\cite{bach2013adaptivity}).

Since quadratic functions may be optimized with rate $O(1/n)$ under weak conditions, we are going to use a quadratic approximation around a well chosen support point, which shares some similarity with the Newton procedure  (however, with a non trivial adaptation to the stochastic approximation framework). The  Newton step for $f$ around a certain point $\tilde{\theta}$ is equivalent to minimizing a quadratic surrogate $g$ of $f$ around $\tilde{\theta}$, i.e.,
$g(\theta) = f(\tilde{\theta}) + \langle f'(\tilde{\theta}),  {\theta} - \tilde{\theta} \rangle  + \frac{1}{2} \langle  {\theta} - \tilde{\theta}, f''(\tilde{\theta}) (  {\theta} - \tilde{\theta} ) \rangle$. If $f_n(\theta) \defeq \ell(y_n, \langle \theta, x_n \rangle)$, then  $g(\theta) = \E g_n( {\theta})$,
with $g_n ( {\theta}) = f(\tilde{\theta}) + \langle f_n'(\tilde{\theta}),  {\theta} - \tilde{\theta} \rangle  + \frac{1}{2} \langle  {\theta} - \tilde{\theta}, f_n''(\tilde{\theta}) (  {\theta} - \tilde{\theta} ) \rangle$; the Newton step may thus be solved approximately with stochastic approximation (here constant-step size LMS), with the following recursion:
\BEQ
\label{eq:newt}
\theta_n =
\theta_{n-1} - \gamma g_n' ( {\theta_{n-1}}) =
\theta_{n-1} - \gamma \big[ f_n'(\tilde{\theta})
+ f_n''(\tilde{\theta}) ( \theta_{n-1} - \tilde{\theta} ) \big].
\EEQ
This is equivalent to replacing the gradient $f_n'(\theta_{n-1})$ by its first-order approximation around $\tilde{\theta}$.
A crucial point is that for machine learning scenarios where $f_n$ is a loss associated to a single data point,  its complexity is only twice the complexity of a regular stochastic approximation step, since, with $f_n(\theta) = \ell( y_n , \langle x_n , \theta \rangle )$,   $f''_n(\theta) $ is a rank-one matrix.

\paragraph{Choice of support points for quadratic approximation.} An important aspect is the choice of the support point $\tilde{\theta}$. In this paper, we consider two strategies:

\begin{list}{\labelitemi}{\leftmargin=1.1em}
   \addtolength{\itemsep}{-.1\baselineskip}

\item[--] \textbf{Two-step procedure}:
for convex losses, averaged SGD with a step-size decaying at $O(1/\sqrt{n})$ achieves  a rate (up to logarithmic terms) of $O(1/\sqrt{n})$~\cite{nemirovski2009robust,gradsto}. We may thus use it to obtain a first decent estimate. The two-stage procedure is as follows (and uses $2n$ observations):
$n$ steps of averaged SGD with constant step size $\gamma \propto 1/{\sqrt{n}}$ to obtain $\tilde{\theta}$, and then averaged LMS for the Newton step around $\tilde{\theta}$. As shown below, this algorithm achieves the rate $O(1/n)$ for logistic regression. However, it is not the most efficient in practice.

\item[--] \textbf{Support point = current average iterate}: we simply consider the current averaged iterate $\bar{\theta}_{n-1}$ as the support point $\tilde{\theta}$, leading to the recursion:
\BEQ
\label{eq:E}
\theta_{n} = \theta_{n-1} - \gamma \big[ f_n'(\bar{\theta}_{n-1})
+ f_n''(\bar{\theta}_{n-1}) ( \theta_{n-1} - \bar{\theta}_{n-1} ) \big].
\EEQ
Although this algorithm has shown to be the most efficient in practice (see \mysec{simulations}) we currently have no proof of convergence. Given that  the behavior of the algorithms does not change much when the support point is updated less frequently than each iteration, there may be some connections to two-time-scale algorithms (see, e.g.,~\cite{borkar1997stochastic}). In \mysec{simulations}, we also consider several other strategies based on doubling tricks.

\end{list}

Interestingly, for non-quadratic functions, our algorithm imposes a new bias (by replacing the true gradient by an approximation which is only valid close to $\bar{\theta}_{n-1}$) in order to reach faster convergence (due to the linearity  of the underlying gradients).

\paragraph{Relationship with one-step-estimators.}
One-step estimators~(see, e.g., \cite{van2000asymptotic}) typically take any estimator with $O(1/n)$-convergence rate, and make a full Newton step to obtain  an efficient estimator (i.e., one that achieves the Cramer-Rao lower bound). Although our novel algorithm is largely inspired by one-step estimators, our situation is slightly different since our first estimator has only convergence rate $O(n^{-1/2})$ and is estimated on different observations.

 \subsection{Self-concordance and logistic regression}
\label{sec:log}

We make the following assumptions:

\begin{list}{\labelitemi}{\leftmargin=1.9em}
   \addtolength{\itemsep}{-.05\baselineskip}

\item[\textbf{(B1)}]  $\H$ is a $d$-dimensional Euclidean space, with $d \geqslant 1$.

\item[\textbf{(B2)}] The observations $(x_n,y_n) \in \H \times \{-1,1\}$ are independent and identically distributed.

\item[\textbf{(B3)}] We consider  $f(\theta)  = \E \big[ \ell( y_n , \langle x_n , \theta \rangle ) \big]$, with the following assumption on the loss function~$\ell$ (whenever we take derivatives of $\ell$, this will be with respect to the second variable):
$$
\forall (y,\hat{y}) \in \{-1,1\} \times \rb,  \ \ \   \ell'(y,\hat{y}) \leqslant 1, \ \ \ell''(y,\hat{y}) \leqslant 1/4,  \  \ |\ell'''(y,\hat{y})| \leqslant  \ell''(y,\hat{y}).
$$
We denote by $\theta_\ast$ a global minimizer of $f$, which we thus assume to exist, and we denote  by $H = f''(\theta_\ast)$ the Hessian operator at a global optimum $\theta_\ast$.

 \item[\textbf{(B4)}] We assume that there exists $R >0$, $\kappa>0$ and $\rho >0$ such that $ \| x_n\|^2 \leqslant R^2 \mbox{ almost surely}$, and
 \BEQ   \E \big [ x_n \otimes x_n \big] \preccurlyeq  \rho
 \E \big [ \ell''(y_n, \langle \theta_\ast, x_n \rangle)  x_n \otimes x_n \big] =
 \rho H ,
 \EEQ
 \BEQ
 \label{eq:kappalog}
\forall z \in \H, \theta \in \H, \ \E \big[ \ell''(y_n, \langle \theta, x_n \rangle)^2 \langle z, x_n \rangle^4  \big]\leqslant \kappa \big(
\E \big[ \ell''(y_n, \langle \theta, x_n \rangle)  \langle z, x_n \rangle^2  \big] \big)^2. 
\EEQ

\end{list}

Assumption \textbf{(B3)} is satisfied for the logistic loss and extends to all generalized linear models~(see more details in~\cite{bach2013adaptivity}), and the relationship between the third derivative and second derivative of the loss $\ell$ is often referred to as \emph{self-concordance}~(see~\cite{self,bach2010self} and references therein). Note moreover that we must have $\rho \geqslant 4$ and $\kappa \geqslant 1$.

A loose upper bound for $\rho$
is $1 / \inf_n \ell''(y_n,\langle \theta_\ast, x_n \rangle)$ but in practice, it is typically much smaller (see Table~\ref{tab:data}). The condition in \eq{kappalog}  is hard to check because it is uniform in $\theta$. With a slightly more complex proof, we could restrict $\theta$ to be close to $\theta_\ast$; with such constraints, the value of $\kappa$ we have found is close to the one from \mysec{high} (i.e., without the terms in $\ell''(y_n, \langle \theta, x_n \rangle)$).

\begin{theorem}
\label{theo:log}
Assume \textbf{(B1-4)}, and consider the vector $\zeta_n$ obtained as follows: (a) perform   $n$ steps of averaged stochastic gradient descent with constant step size $1/2R^2 \sqrt{n}$, to get $\tilde{\theta}_n$, and (b) perform $n$ step of averaged LMS with constant step-size $1/R^2$ for the quadratic approximation of $f$ around~$\tilde{\theta}_n$. If $n \geqslant ( 19  +  9 R \| \theta_0 - \theta_\ast\| )^4$, then

\BEQ
\E f(\zeta_n) - f(\theta_\ast) \leqslant \frac{\kappa^{3/2} \rho^3 d   }{n}    (    16 R  \| \theta_0 - \theta_\ast\| + 19)^4  .
\EEQ

\end{theorem}
We get an $O(1/n)$ convergence rate without assuming strong convexity, even locally, thus improving on results from~\cite{bach2013adaptivity} where the the rate is proportional to $ 1 / ( n \lambda_{\min}(H) )$. The proof    relies on self-concordance properties and the sharp analysis of the Newton step (see appendix).

\section{Experiments}
\label{sec:simulations}
\label{sec:experiments}

\subsection{Synthetic data}

\paragraph{Least-mean-square algorithm.}
We consider normally distributed inputs, with  covariance matrix~$H$ that has random eigenvectors and eigenvalues $1/k$, $k =1,\dots,d$. The outputs are generated from a linear function with homoscedastic noise with unit signal to noise-ratio. We consider $d=20$ and the least-mean-square algorithm with several settings of the step size $\gamma_n$, constant or proportional to $1/\sqrt{n}$. Here $R^2$ denotes the \emph{average radius of the data}, i.e., $R^2 = \tr H$.
In the left plot of \myfig{logistictoy}, we show the results, averaged over 10 replications.

 Without averaging, the algorithm with constant step-size does not converge pointwise (it oscillates), and its average excess risk decays as a linear function of $\gamma$ (indeed, the gap between each values of the constant step-size is close to $\log_{10}(4)$, which corresponds to a linear function in $\gamma$).

With averaging, the algorithm with constant step-size does  converge  at rate $O(1/n)$, and for all values of the constant $\gamma$, the rate is actually the same. Moreover (although it is not shown in the plots), the standard deviation is much lower.

With decaying step-size $\gamma_n = 1/(2 R^2 \sqrt{n})$ and without averaging, the convergence  rate is $O(1/\sqrt{n})$, and improves to $O(1/n)$ with averaging.

\paragraph{Logistic regression.}
We consider the same input data as for least-squares, but now generates outputs from the logistic probabilistic model. We compare several algorithms and display the results in \myfig{logistictoy} (middle and right plots).

On the middle plot, we consider SGD. Without averaging, the algorithm with constant step-size does not converge and its average excess risk reaches a constant value which is a linear function of $\gamma$ (indeed, the gap between each values of the constant step-size is close to $\log_{10}(4)$). With averaging, the algorithm does converge, but as opposed to least-squares, to a point which is not the optimal solution, with an error proportional to $\gamma^2$ (the gap between curves is twice as large).

On the right plot, we consider various variations of our Newton-approximation scheme. The ``2-step'' algorithm is the one for which our convergence rate holds ($n$ being the total number of examples, we perform $n/2$ steps of averaged SGD, then $n/2$ steps of LMS). Not surprisingly, it is not the best in practice (in particular at $n/2$, when starting the constant-size LMS, the performance worsens temporarily). It is classical to use doubling tricks to remedy this problem while preserving convergence rates~\cite{hazanbeyond}, this is done in ``2-step-dbl.'', which avoids the previous erratic behavior.

 We have also considered getting rid of the first stage where plain averaged stochastic gradient is used to obtain a support point for the quadratic approximation. We now consider only Newton-steps but change only these support points.
 We consider updating the support point at every iteration, i.e., the recursion from \eq{E}, while we also consider updating it every dyadic point (``dbl.-approx'').  The last two algorithms perform very similarly and achieve the $O(1/n)$ early. In all experiments on real data, we have considered the simplest variant (which corresponds to \eq{E}).


\begin{figure}

\begin{center}
\hspace*{-.5cm}
\includegraphics[scale=.475]{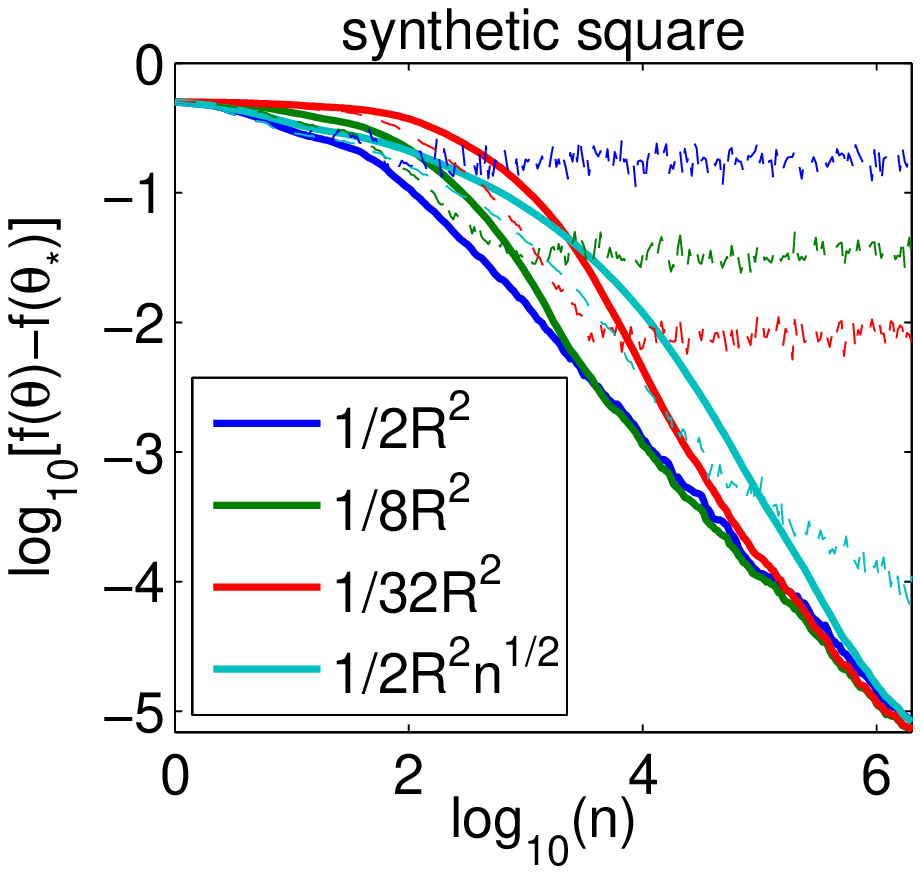}
\hspace*{.85cm}
\includegraphics[scale=.475]{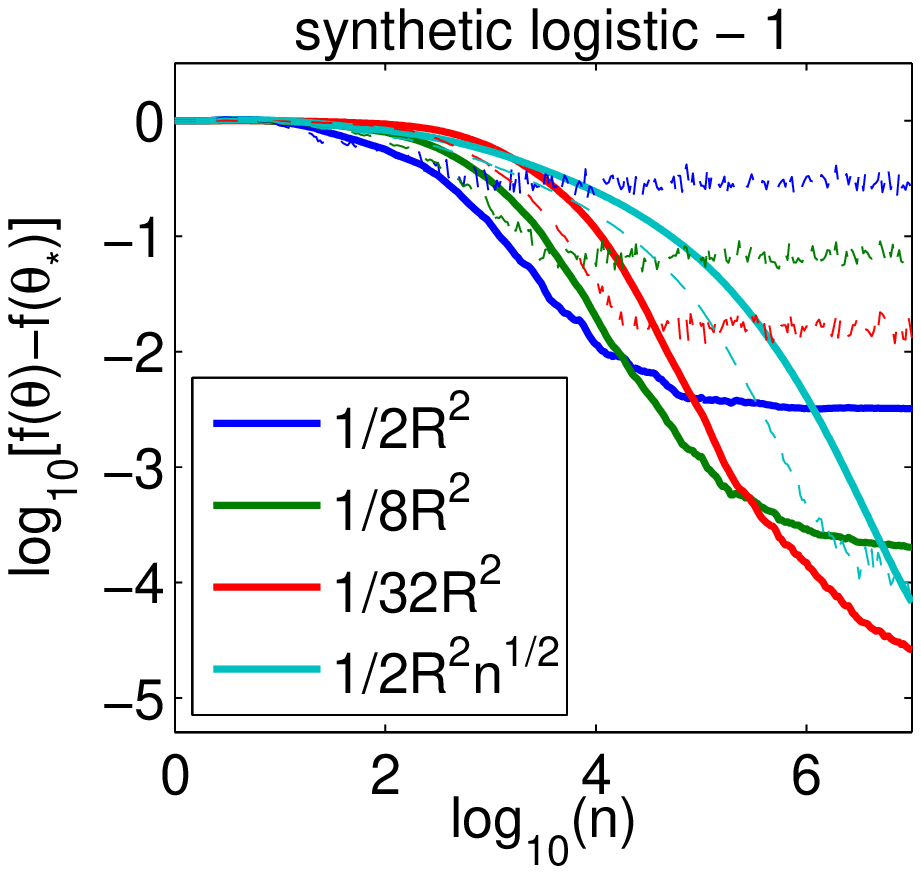}
\hspace*{.25cm}
\includegraphics[scale=.475]{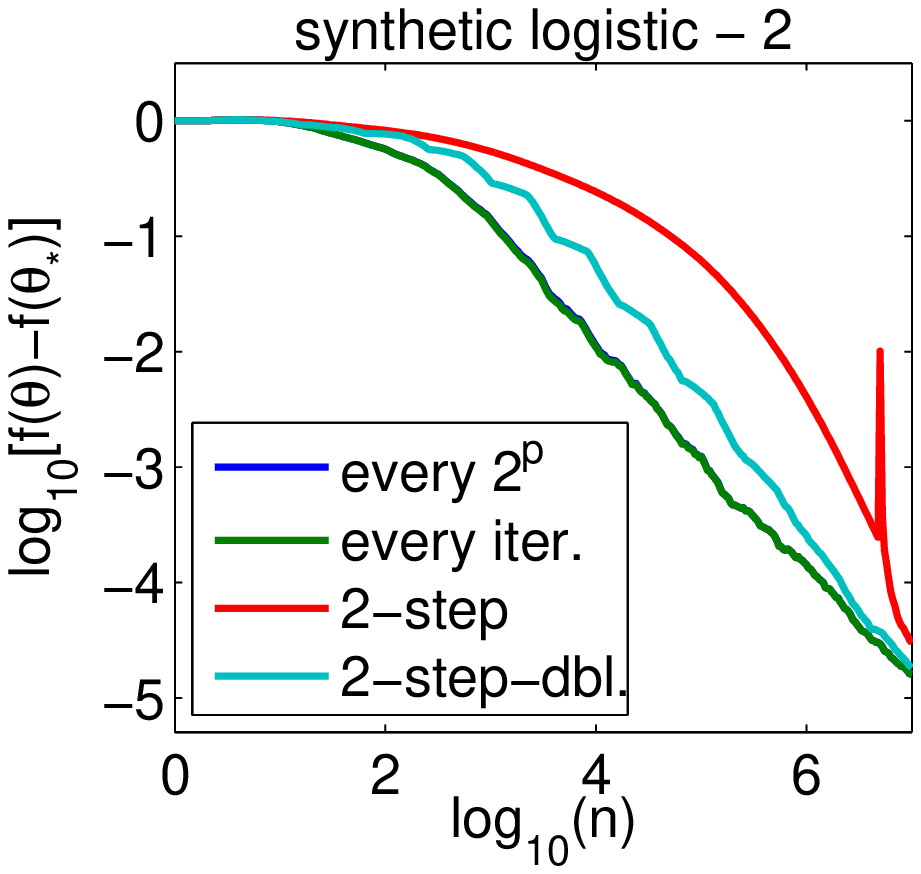}
\hspace*{-.5cm}
\end{center}

\caption{Synthetic data. Left: least-squares regression. Middle:  logistic regression with averaged SGD with various step-sizes, averaged  (plain) and non-averaged (dashed). Right: various Newton-based schemes for the same logistic regression problem. Best seen in color. See text for details.
\label{fig:logistictoy} }

\end{figure}

\subsection{Standard benchmarks}

We have considered 6 benchmark datasets which are often used in comparing large-scale optimization methods. The datasets are described in Table~\ref{tab:data} and vary in values of $d$, $n$ and sparsity levels. These are all \emph{finite} binary classification datasets with outputs in $\{-1,1\}$. For least-squares and logistic regression, we have   followed the following experimental protocol: (1) remove all outliers (i.e., sample points $x_n$ whose norm is greater than 5 times the average norm), (2) divide the dataset in two equal parts, one for training, one for testing, (3) sample within the training dataset with replacement, for 100 times the number of observations in the training set (this corresponds to $100$ effective passes; in all plots, a black dashed line marks the first effective pass), (4) compute averaged cost on training and testing data (based on 10 replications). All the costs are shown in log-scale, normalized to that the first iteration leads to $f(\theta_0) - f(\theta_\ast) = 1$.

All algorithms that we consider (ours and others) have a step-size, and typically a theoretical value that ensures convergence. We consider two settings: (1) one when this theoretical value is used, (2) one with the best testing error after one effective pass through the data (testing powers of $4$ times the theoretical step-size).

Here, we only consider \emph{covertype}, \emph{alpha}, \emph{sido} and \emph{news}, as well as test errors. For all training errors and the two other datasets (\emph{quantum}, \emph{rcv1}), see the appendix.

\urlstyle{same}

\paragraph{Least-squares regression.}
We compare three algorithms: averaged SGD with constant step-size, averaged SGD with   step-size decaying as $C/R^2 \sqrt{n}$, and the stochastic averaged gradient (SAG) method which is dedicated to finite training data sets~\cite{sag}, which has shown state-of-the-art performance in this set-up\footnote{The original algorithm from~\cite{sag} is considering only strongly convex problems, we have used the step-size of $1/16R^2$, which achieves fast convergence rates in all situations (see \url{http://research.microsoft.com/en-us/um/cambridge/events/mls2013/downloads/stochastic_gradient.pdf}).}. We show the results in the two left plots of \myfig{test1} and \myfig{test2}.

Averaged SGD with decaying step-size equal to $C/R^2 \sqrt{n}$ is slowest (except for \emph{sido}). In particular, when the best constant $C$ is used (right columns), the performance typically starts to increase significantly. With that step size, even after 100 passes, there is no sign of overfitting, even for the high-dimensional sparse datasets.

SAG and constant-step-size averaged SGD exhibit the best behavior, for the theoretical step-sizes and the best constants, with a significant advantage for constant-step-size SGD. The non-sparse datasets do not lead to overfitting, even close to the global optimum of the (unregularized) training objectives, while the sparse datasets do exhibit some overfitting after more than 10 passes.

\paragraph{Logistic regression.}
We also compare two additional algorithms: our Newton-based technique and ``Adagrad''~\cite{adagrad}, which is a stochastic gradient method with a form a diagonal scaling\footnote{Since a bound on $\|\theta_\ast\|$ is not available, we have used step-sizes proportional to $1 / \sup_n \| x_n \|_\infty$.} that allows to reduce the convergence rate (which is still in theory proportional to $O(1/\sqrt{n})$).
We show results in the two right plots of \myfig{test1} and \myfig{test2}.

 Averaged SGD with decaying step-size proportional to $1/R^2 \sqrt{n}$ has the same behavior than for least-squares (step-size harder to tune, always inferior performance except for \emph{sido}).

  SAG, constant-step-size SGD and the novel Newton technique tend to behave similarly (good with theoretical step-size, always among the best methods). They differ notably in some aspects: (1) SAG converges quicker for the training errors (shown in the appendix) while it is a bit slower for the testing error, (2) in some instances, constant-step-size averaged SGD does underfit (\emph{covertype}, \emph{alpha}, \emph{news}), which is consistent with the lack of convergence to the global optimum mentioned earlier, (3) the novel Newton approximation is consistently better.

 On the non-sparse datasets, Adagrad performs similarly to the Newton-type method (often better in early iterations and worse later), except for the \emph{alpha} dataset where the step-size is harder to tune (the best step-size tends to have early iterations that make the cost go up significantly). On sparse datasets like \emph{rcv1}, the performance is essentially the same as Newton. On the \emph{sido} data set, Adagrad (with fixed steps size, left column) achieves a good testing loss quickly then levels off, for reasons we cannot explain. On the \emph{news} dataset, it is inferior without parameter-tuning and a bit better with. Adagrad uses a diagonal rescaling; it could be combined with our technique, early experiments show that it improves results but that it is more sensitive to the choice of step-size.

 Overall, even with $d$ and $\kappa$ very large (where our bounds are vacuous), the performance of our algorithm still achieves the state of the art, while being more robust to the selection of the step-size: finer quantities likes degrees of freedom~\cite{gu2002smoothing} should be able to quantify more accurately the quality of the new algorithms.

\begin{table}

\caption{Datasets used in our experiments\label{tab:data}. We report the proportion of non-zero entries, as well as estimates for the constant $\kappa$ and $\rho$ used in our theoretical results, together with  the non-sharp constant which is typically used in analysis of logistic regression and which our analysis avoids (these are computed for non-sparse datasets only).}
\begin{center}
\begin{tabular}{|l|r|r|r|l|l|l|}
\hline
Name & $d$ & $n$ & sparsity &  $ {\kappa}$ & $ {\rho}$ & $1 / \inf_n \ell''(y_n,\langle \theta_\ast, x_n \rangle)$\\
\hline
\emph{quantum} & 79 & 50 000 &    100 \% & 5.8  $\times 10^{2}$ & 16  & 8.5 $\times 10^2$ \\
\emph{covertype} & 55 & 581 012 &   100 \% &  9.6 $\times  10^{2}$ & 160 & 3 $\times 10^{12}$ \\
\emph{alpha} & 501 & 500 000 &   100 \% &  6 & 18  & 8 $\times 10^4$\\
\emph{sido} &   4 933& 12 678 &    10 \% &  1.3 $\times 10^{4}$ & $\times $ & $\times$\\
\emph{rcv1} &         47 237& 20 242 &   0.2 \%  & 2 $\times 10^{4}$ & $\times $& $\times$\\
\emph{news} &       1 355 192 & 19 996  &   0.03 \% &    2 $\times 10^{4}$ & $\times $& $\times$\\
\hline
\end{tabular}

\end{center}
\end{table}

\begin{figure}

\begin{center}

\hspace*{-.5cm}
\includegraphics[scale=.43]{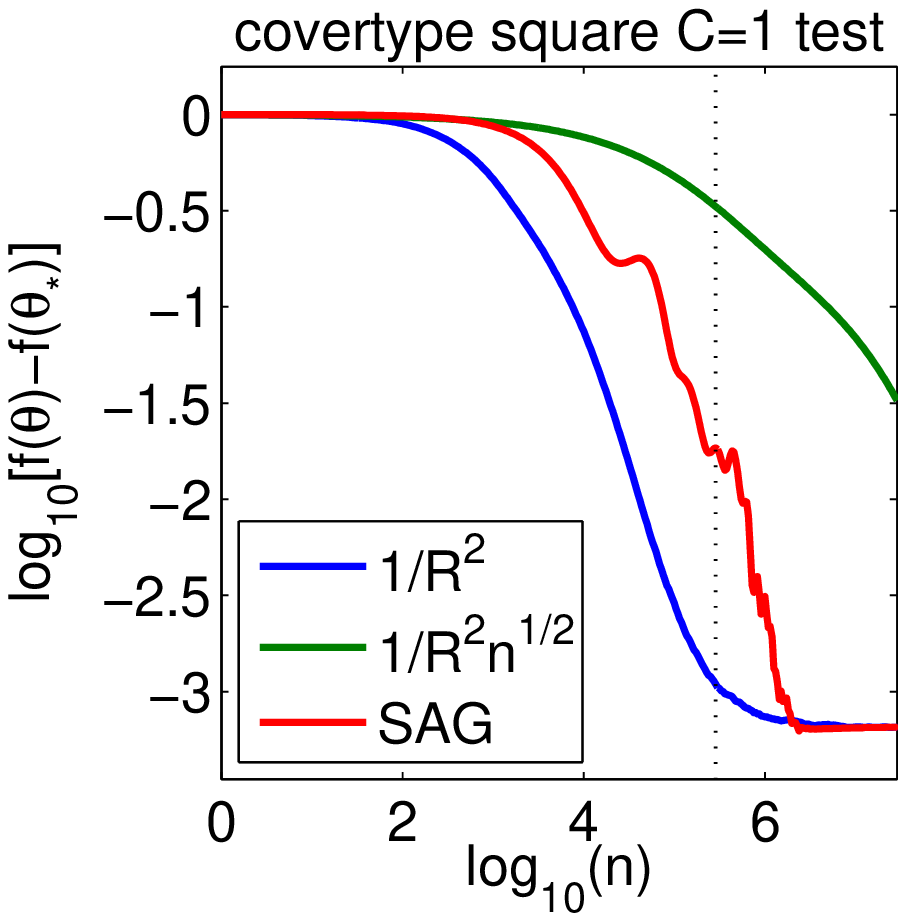}
\hspace*{-.25cm}
\includegraphics[scale=.43]{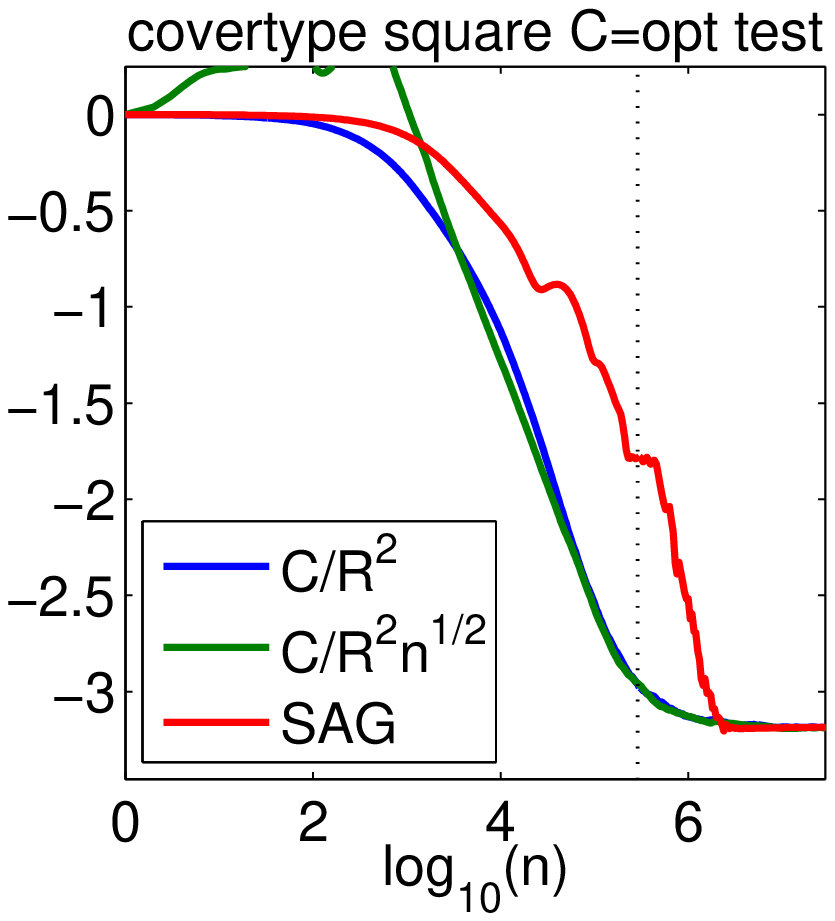}
\includegraphics[scale=.43]{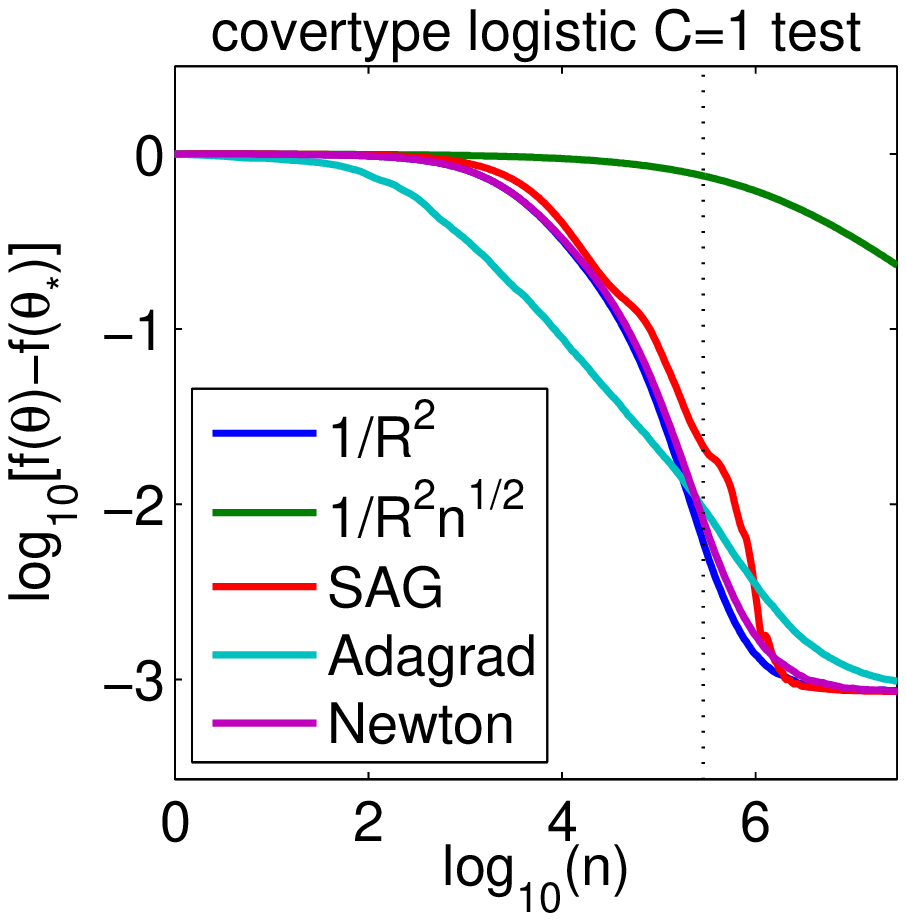}
\hspace*{-.25cm}
\includegraphics[scale=.43]{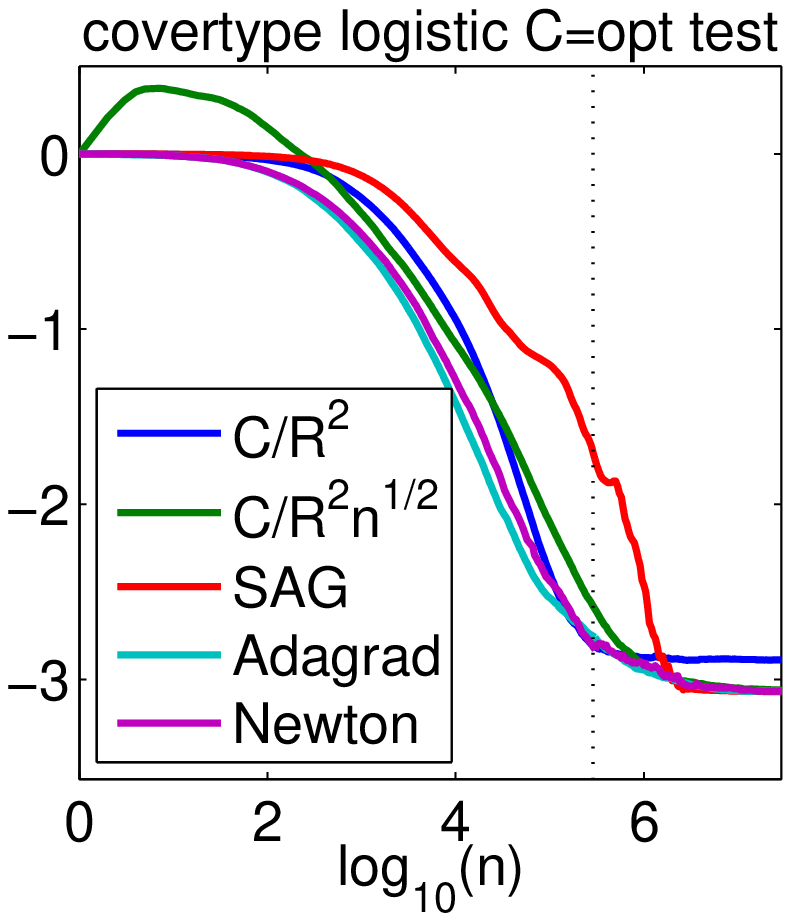}
\hspace*{-.5cm}

\hspace*{-.5cm}
\includegraphics[scale=.43]{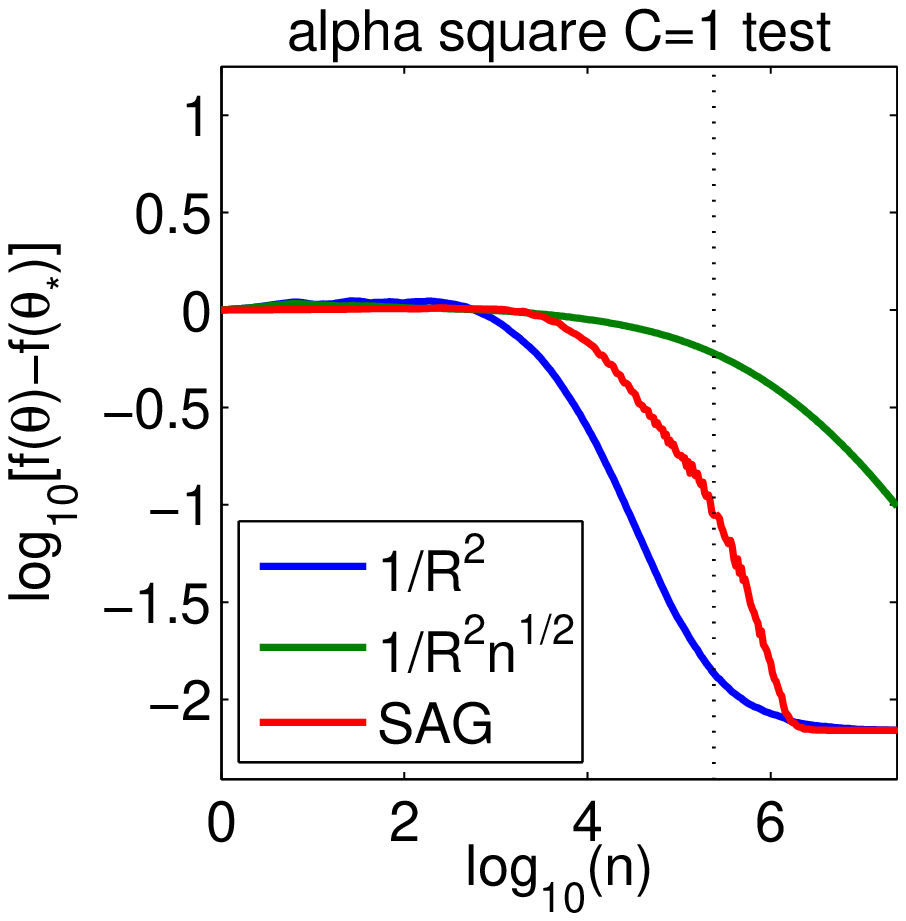}
\hspace*{-.25cm}
\includegraphics[scale=.43]{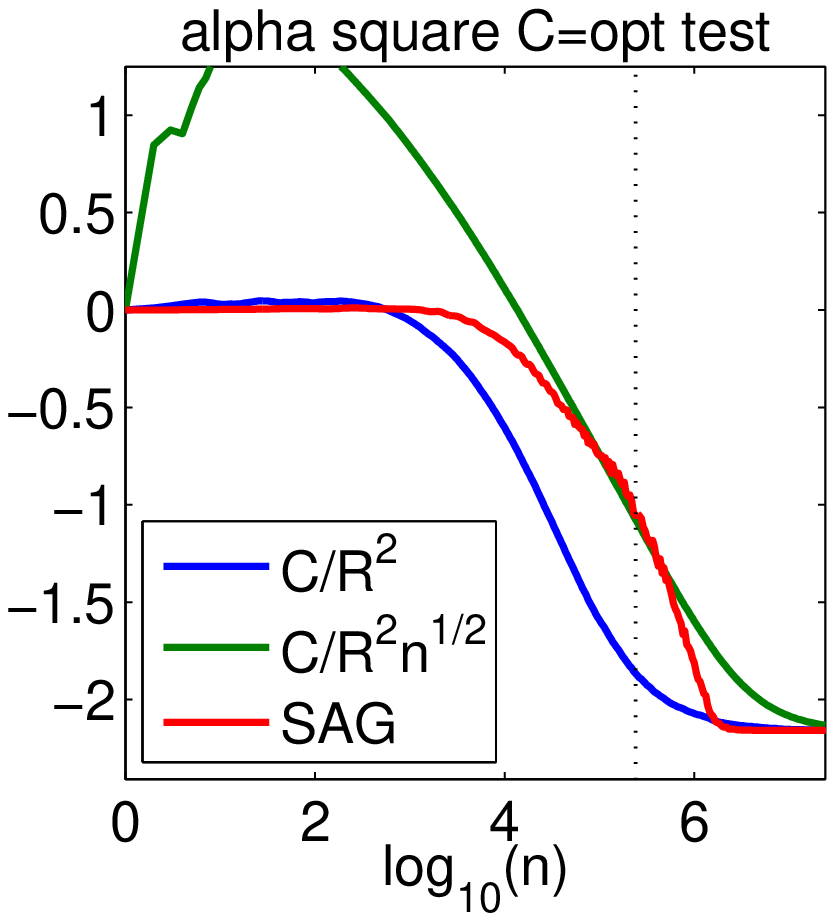}
\includegraphics[scale=.43]{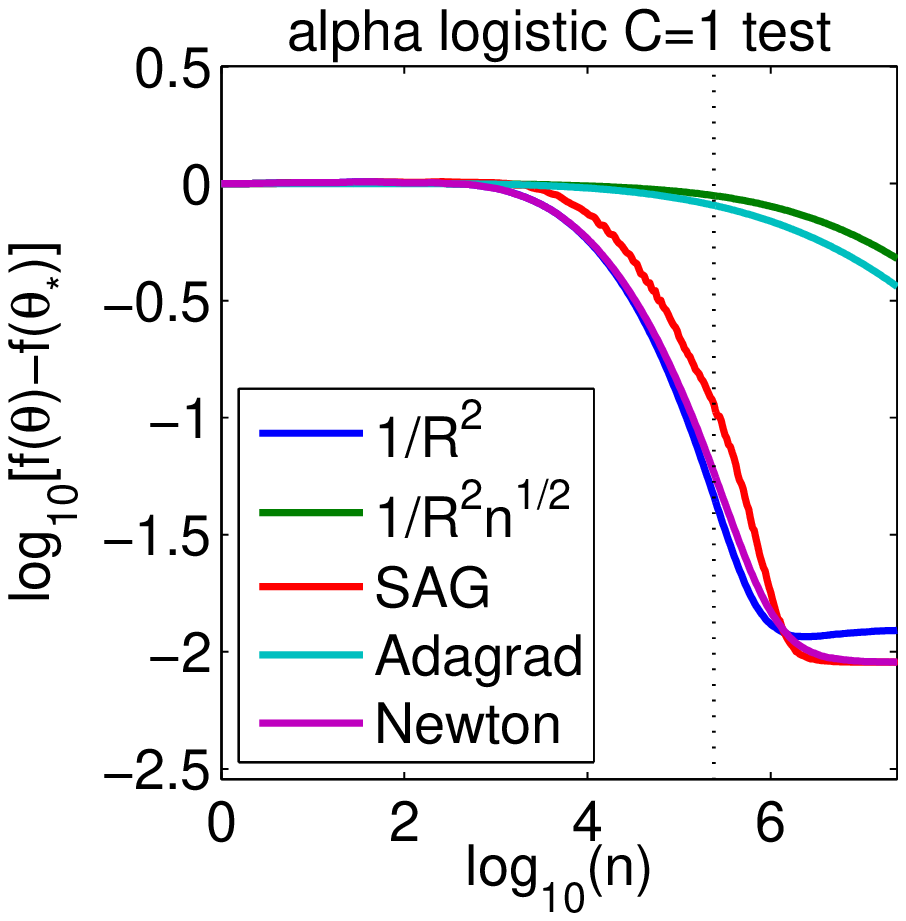}
\hspace*{-.25cm}
\includegraphics[scale=.43]{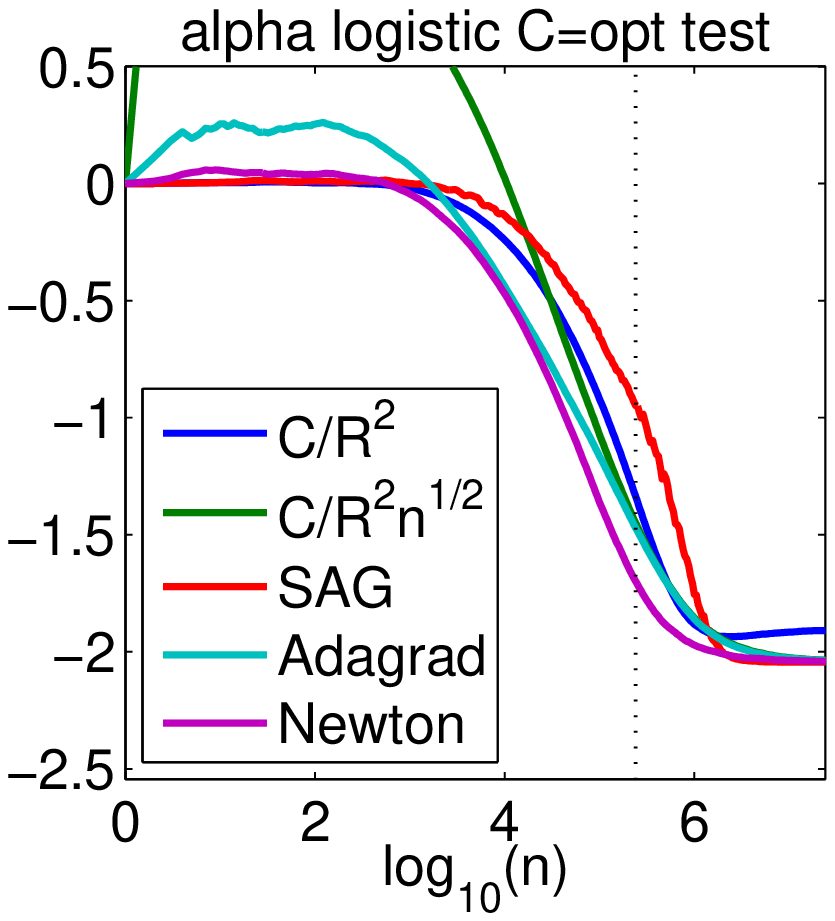}
\hspace*{-.5cm}

\end{center}

\caption{Test  performance for least-square regression (two left  plots) and  logistic regression (two right plots). From top to bottom:  \emph{covertype}, \emph{alpha}. Left: theoretical steps, right: steps optimized for performance after one effective pass through the data. Best seen in color. }
\label{fig:test1}

\end{figure}

\begin{figure}

\begin{center}
\hspace*{-.5cm}
\includegraphics[scale=.43]{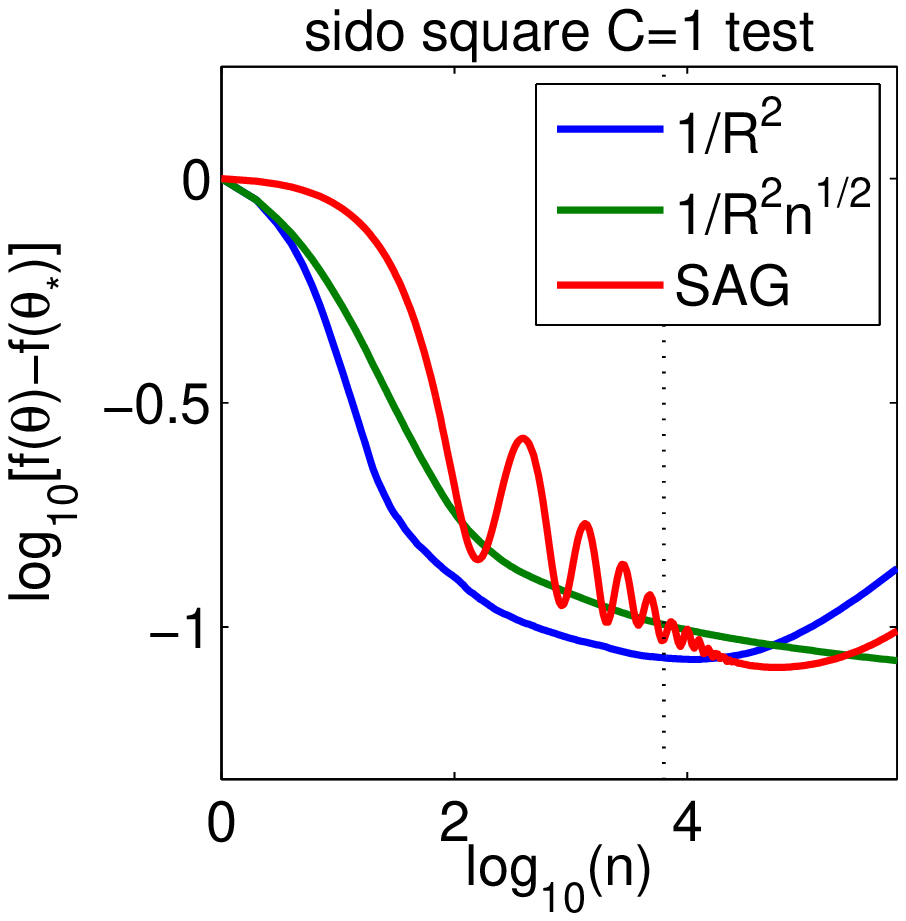}
\hspace*{-.25cm}
\includegraphics[scale=.43]{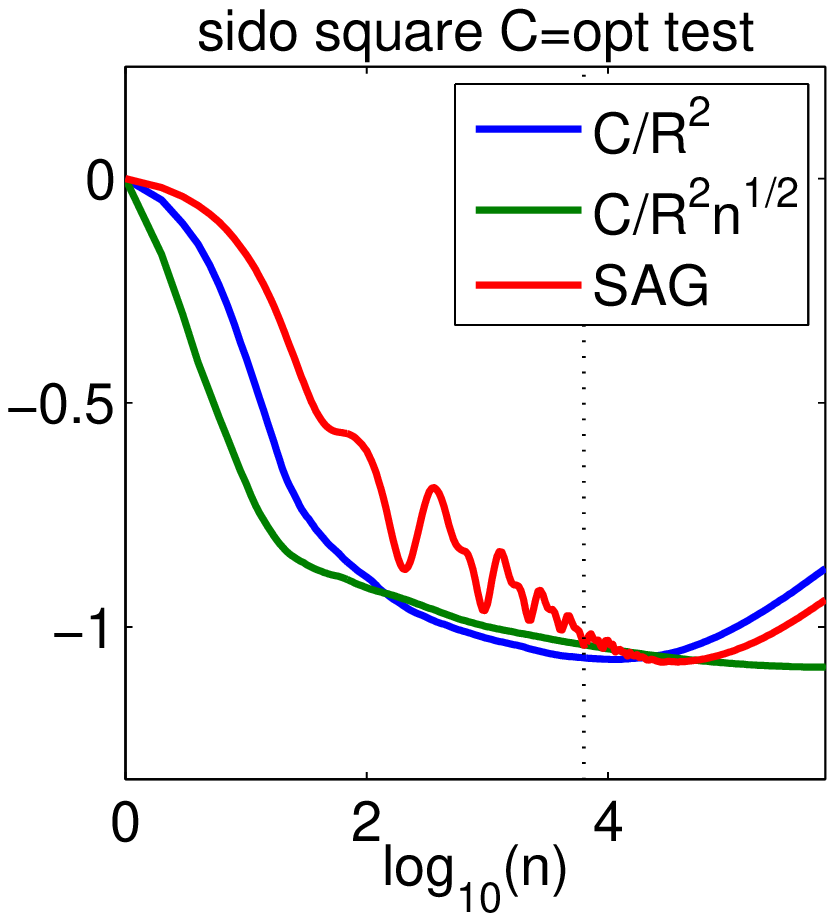}
\includegraphics[scale=.43]{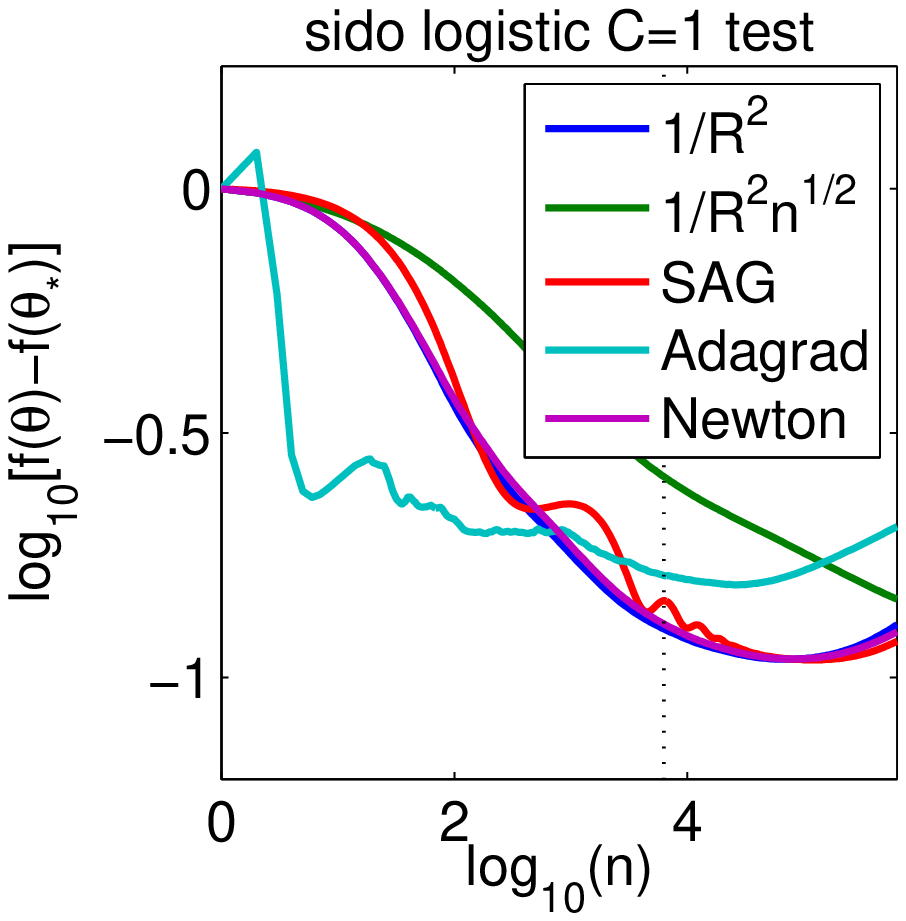}
\hspace*{-.25cm}
\includegraphics[scale=.43]{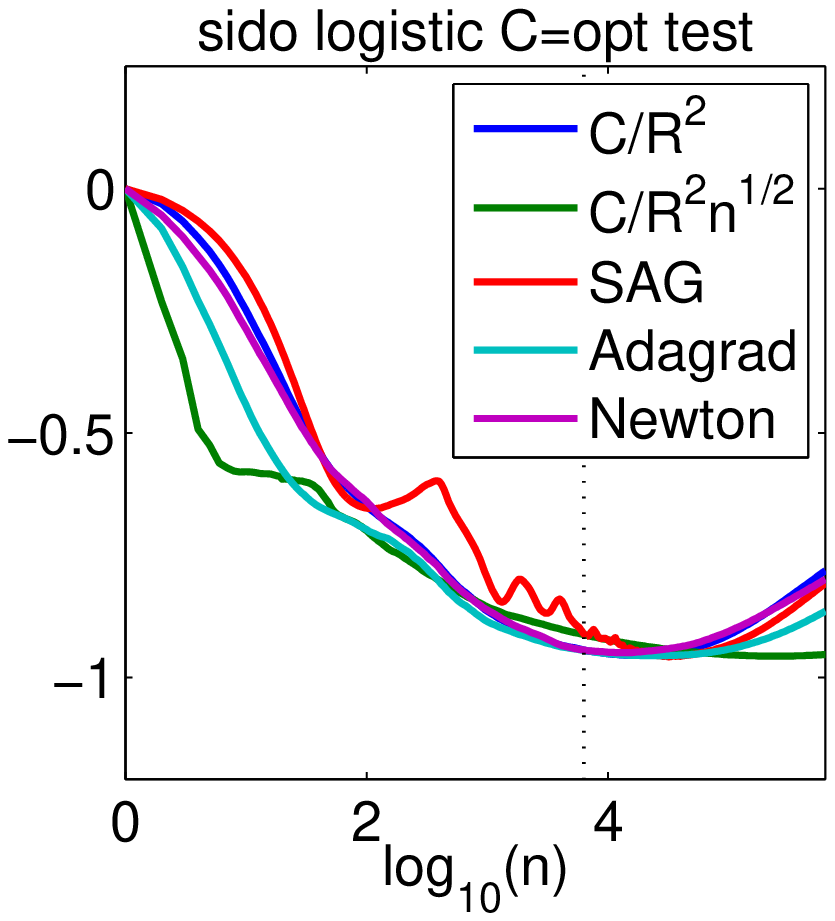}
\hspace*{-.5cm}

\hspace*{-.5cm}
\includegraphics[scale=.43]{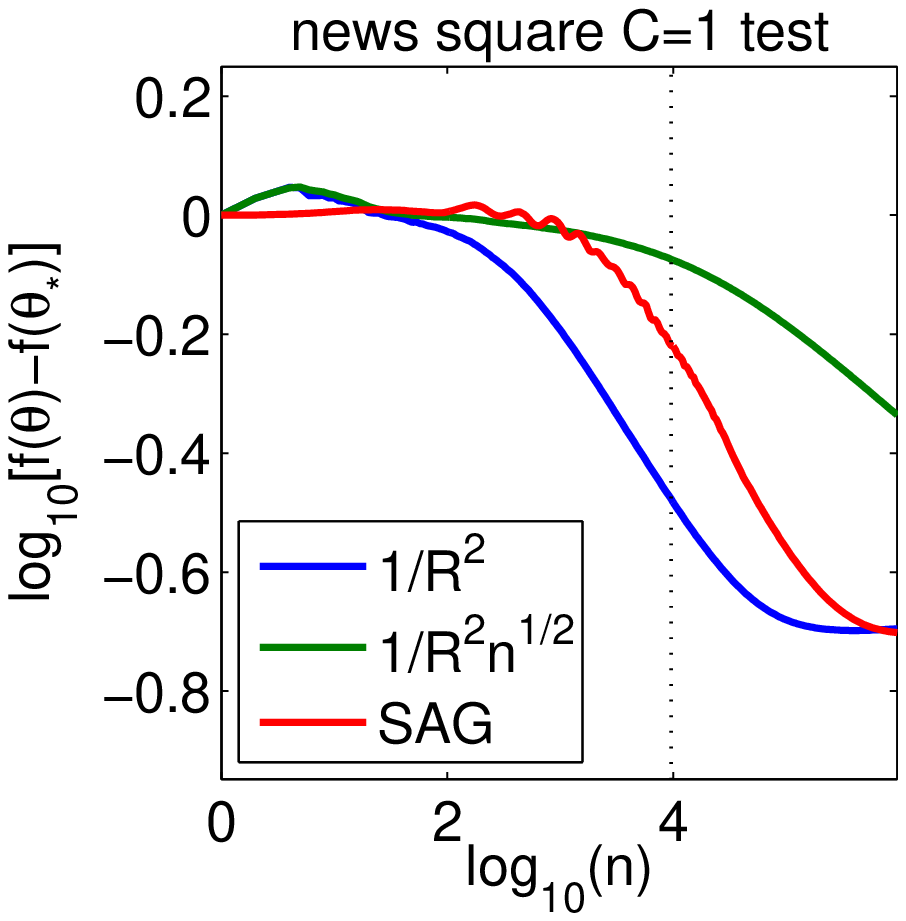}
\hspace*{-.25cm}
\includegraphics[scale=.43]{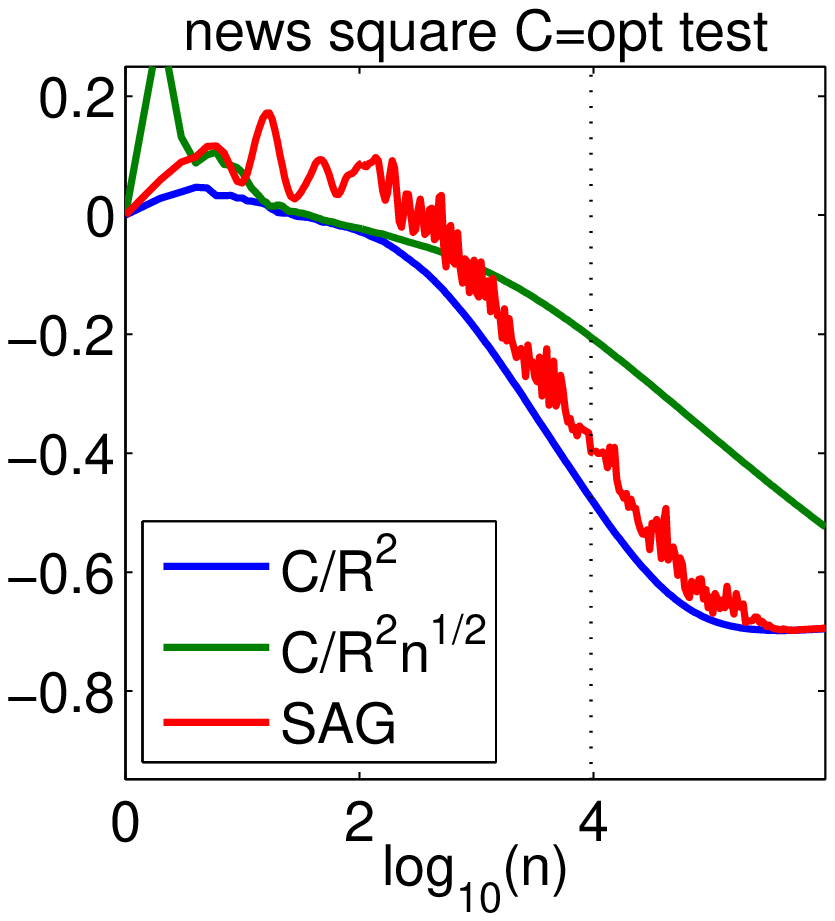}
\includegraphics[scale=.43]{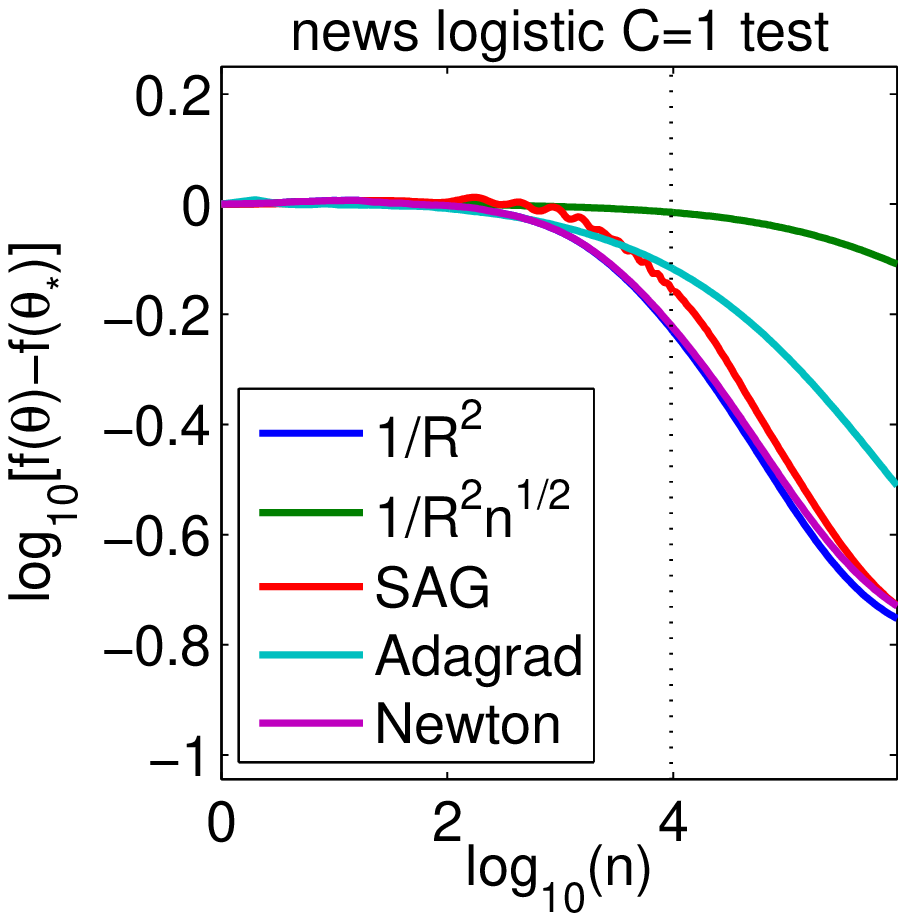}
\hspace*{-.25cm}
\includegraphics[scale=.43]{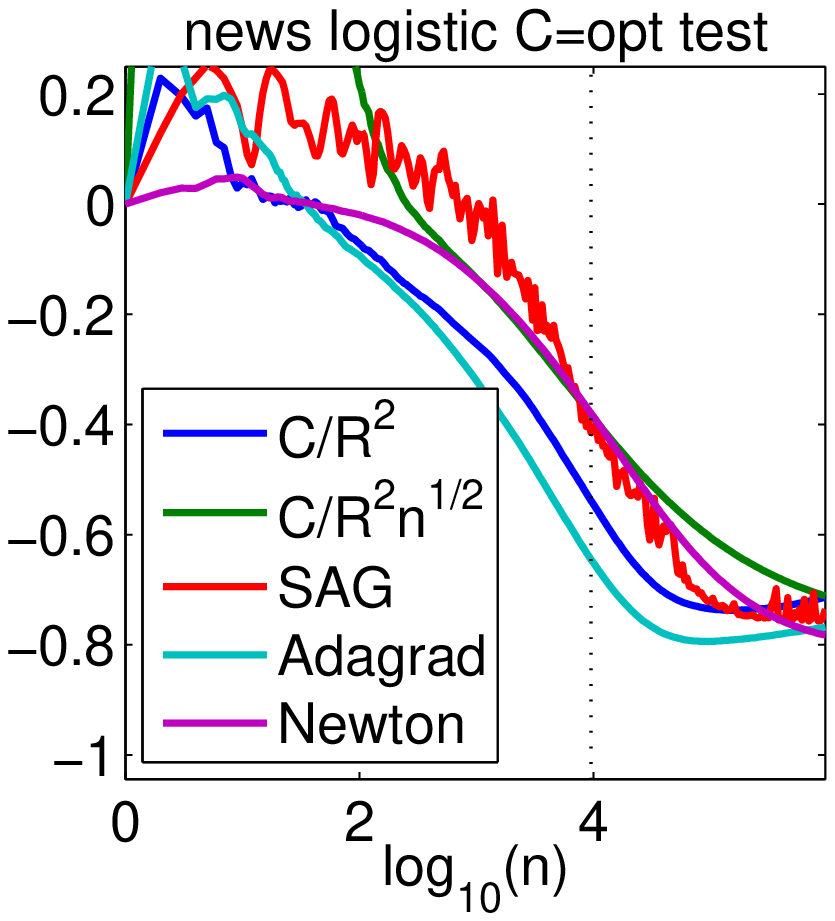}
\hspace*{-.5cm}

\end{center}

\caption{Test  performance for least-square regression (two left  plots) and  logistic regression (two right plots). From top to bottom: \emph{sido}, \emph{news}. Left: theoretical steps, right: steps optimized for performance after one effective pass through the data. Best seen in color. }
\label{fig:test2}

\end{figure}

\section{Conclusion}

In this paper, we have presented two stochastic approximation algorithms that can achieve rates of $O(1/n)$ for logistic and least-squares regression, without strong-convexity assumptions. Our analysis reinforces the key role of averaging in obtaining fast rates, in particular with large step-sizes.
Our work can naturally be extended in several ways:
(a) an analysis of the algorithm that updates the support point of the quadratic approximation at every iteration, (b) proximal extensions (easy to implement, but potentially harder to analyze); (c) adaptive ways to find the constant-step-size; (d) step-sizes that depend on the iterates to increase robustness, like in normalized LMS~\cite{nmls}, and (e) non-parametric analysis to improve our theoretical results for large values of $d$.

\subsection*{Acknowledgements}
 
This work was partially supported by the European Research Council (SIERRA Project 239993).
The authors would like to thank Simon Lacoste-Julien and Mark Schmidt  for discussions related to this work.

\newpage

\appendix

In the appendix we provide proofs of all three theorems, as well as additional experimental results (all training objectives, and two additional datasets \emph{quantum} and \emph{rcv1}).

\paragraph{Notations.} Throughout this appendix material, we are going to use the notation $\| X\|_p =\big[  \E ( \| X\|^p ) \big]^{1/p}$ for any random vector $X$ and real number $p \geqslant 1$. By Minkowski's inequality, we have the triangle inequality $\| X + Y \|_p \leqslant \| X\|_p + \| Y\|_p$ whenever the expression makes sense.

\appendix
\section{Proof of Theorem~\ref{theo:lms}}

We first denote by $\eta_n = \theta_n - \theta_\ast \in \H$ the deviation to $\theta_\ast$. Since we consider quadratic functions, it satisfies a simplified recursion:
\BEA
\nonumber \eta_n & = & \eta_{n-1} -  \gamma ( x_n \otimes x_n ) \theta_n + \gamma \xi_n \\
\label{eq:eta}
  & = &  \big( \idm -  \gamma x_n \otimes x_n \big)  \eta_{n-1}  + \gamma \xi_n.
\EEA
We also consider $\bar{\eta}_n = \frac{1}{n+1} \sum_{k=0}^n \eta_k = \bar{\theta}_n - \theta_\ast $ the averaged iterate.
We have
$f(\theta_n) - f(\theta_\ast) = \frac{1}{2} \langle \eta_n, H \eta_n \rangle$
and $f(\bar{\theta}_n) - f(\theta_\ast) = \frac{1}{2} \langle \bar{\eta}_n, H \eta_n \rangle$.

The crux of the proof is to consider the same recursion as \eq{eta}, but replacing $x_n \otimes x_n$ by its expectation $H$ (which is related to fixed design analysis in linear regression). This  is of course only an approximation, and thus one has to study the remainder term; it happens to satisfy a similar recursion, on which we can apply the same technique, and so on. This proof technique is taken from~\cite{aguech2000perturbation}. Here we push it to arbitrary orders with explicit constants for \emph{averaged} constant-step-size stochastic gradient descent.

\paragraph{Consequences of assumptions.}
Note that  Assumption \textbf{(A6)}  implies that $\E \| x_n\|^2 \leqslant R^2$
(indeed, taking the trace of $\E \big( \|x_n\|^2 x_n \otimes x_n \big) \preccurlyeq R^2 H$, we get $\E \| x_n \|^4 \leqslant R^2 \E \|x_n\|^2$, and we always have by Cauchy-Schwarz inequality, $\E \|x_n\|^2 \leqslant \sqrt{ \E \| x_n\|^4}
\leqslant R \sqrt{\E \|x_n\|^2 }$). This then implies that $\tr H \leqslant R^2$ and thus $H \preccurlyeq ( \tr H) \idm \preccurlyeq R^2 \idm$. Thus, whenever $\gamma \leqslant 1/R^2$, we have $\gamma H \preccurlyeq \idm$, for the order between positive definite matrices.

We denote by $\F_n$ the $\sigma$-algebra generated by $(x_1,z_1,\dots,x_n,z_n)$.
Both $\theta_n$ and $\bar{\theta}_{n}$ are $\F_n$-measurable.

\subsection{Two main lemmas}

The proof relies on two lemmas, one that provides a weak result essentially equivalent (but more specific and simpler because the step-size is constant) to non-strongly-convex results from~\cite{gradsto}, and one that replaces $x_n \otimes x_n$ by its expectation $H$ in \eq{eta}, which may then be seen as a non-asymptotic counterpart to the similar set-tup in~\cite{polyak1992acceleration}.

 \begin{lemma}
 \label{lemma:1}
Assume $(x_n,\xi_n) \in \H \times \H$ are $\mathcal{F}_n$-measurable for a sequence of increasing $\sigma$-fields $(\mathcal{F}_n)$, $n \geqslant 1$. Assume that $\E [ \xi_n | \F_{n-1}] = 0$,  $\E \big[  \| \xi_n\|^2 | \F_{n-1} \big] $ is finite and $\E  \big[ \big( \|x_n\|^2 x_n \otimes x_n \big)  | \F_{n-1} \big] \preccurlyeq R^2 H$, with $ \E \big[ x_n \otimes x_n \big | \F_{n-1} ] = H$ for all $n \geqslant 1$, for some $R>0$ and invertible operator $H$. Consider the recursion $\alpha_n = \big( \idm - \gamma x_n \otimes x_n \big)  \alpha_{n-1}  + \gamma \xi_n$, with $\gamma R^2 \leqslant 1$. Then:
$$
( 1 - \gamma R^2) \E \big[ \langle \bar{\alpha}_{n-1}, H \bar{\alpha}_{n-1} \rangle \big] + \frac{1}{2 n \gamma  } \E \| \alpha_n \|^2 \leqslant  \frac{1}{2 n \gamma   }  \| \alpha_0 \|^2 +  \frac{ \gamma}{ n}  \sum_{k=1}^n  \E \| \xi_k\|^2.
$$
\end{lemma}
\begin{proof}
We follow the proof technique of~\cite{gradsto} (which relies only on smoothness) and get:
\BEAS
\| \alpha_{n} \|^2
& = & \| \alpha_{n-1}\|^2 + \gamma^2    \| \xi_n  - ( x_n \otimes x_n ) \alpha_{n-1} \|^2
+ 2 \gamma \langle \alpha_{n-1},  \xi_n  - ( x_n \otimes x_n ) \alpha_{n-1}  \rangle \\
& \leqslant & \| \alpha_{n-1}\|^2 +  \bigg\{ 2 \gamma^2  \| \xi_n\|^2 + 2 \gamma^2  \|  ( x_n \otimes x_n ) \alpha_{n-1} \|^2 \bigg\}
+ 2 \gamma \langle \alpha_{n-1},  \xi_n  - ( x_n \otimes x_n ) \alpha_{n-1}  \rangle .
\EEAS
By taking expectations, we obtain:
\BEAS
\E \big[ \| \alpha_n \|^2 | \F_{n-1} \big]
& \leqslant & 
 \| \alpha_{n-1}\|^2 + 2 \gamma^2   \| \xi_n\|^2 + 2 \gamma^2 
 \langle \alpha_{n-1},  \E \big[ \|x_n\|^2 x_n \otimes x_n \big]  \alpha_{n-1}  \rangle   - 2 \gamma \langle \alpha_{n-1},  H \alpha_{n-1}  \rangle 
\\
& \leqslant & 
 \| \alpha_{n-1}\|^2 + 2 \gamma^2   \| \xi_n\|^2 + 2 \gamma^2 
 R^2\langle \alpha_{n-1},  H \alpha_{n-1}  \rangle   - 2 \gamma \langle \alpha_{n-1},  H \alpha_{n-1}  \rangle 
\\
& = & 
 \| \alpha_{n-1}\|^2 + 2 \gamma^2   \| \xi_n\|^2 + 2\gamma^2   R^2\langle \alpha_{n-1},  H \alpha_{n-1}  \rangle   - 2 \gamma \langle \alpha_{n-1},  H \alpha_{n-1}  \rangle 
\\
& \leqslant & 
 \| \alpha_{n-1}\|^2 + 2 \gamma^2   \| \xi_n\|^2 -  2 \gamma ( 1 - \gamma R^2) \langle \alpha_{n-1},  H \alpha_{n-1}  \rangle.
\EEAS
By taking another expectation, we get
$$
\E \| \alpha_n \|^2 \leqslant 
 \E \| \alpha_{n-1}\|^2 + 2 \gamma^2  \E  \| \xi_n\|^2 -  2 \gamma ( 1 - \gamma R^2) \E \langle \alpha_{n-1},  H \alpha_{n-1}  \rangle.
$$
This leads to the desired result, because, by convexity,
$ \langle \bar{\alpha}_{n-1}, H \bar{\alpha}_{n-1} \rangle \leqslant \frac{1}{n} \sum_{k=0}^{n-1}  \langle  {\alpha}_{k}, H  {\alpha}_{k} \rangle$.\end{proof}

\begin{lemma}
 \label{lemma:2}
 Assume $ \xi_n \in \H  $ is $\mathcal{F}_n$-measurable for a sequence of increasing $\sigma$-fields $(\mathcal{F}_n)$, $n \geqslant 1$.
Assume $\E [ \xi_n | \F_{n-1} ] = 0$,  $\E \big[ \| \xi_n\|^2 \big]$ is finite, and for all $n \geqslant 1$, $ \E \big[ \xi_n \otimes \xi_n  \big]\preccurlyeq C$. Consider the recursion $\alpha_n = \big( \idm - \gamma H \big)  \alpha_{n-1}  + \gamma \xi_n$, with $\gamma H \preccurlyeq  \idm$ for some invertible $H$. Then:
\BEQ
\E [ {\alpha}_n \otimes {\alpha}_{n} ]    =  ( \idm - \gamma H)^n \alpha_0 \otimes \alpha_0
( \idm - \gamma H)^n
+\gamma^2 \sum_{k=1}^n ( \idm - \gamma H)^{n-k} C  ( \idm - \gamma H)^{n-k},
\EEQ
\BEQ
\E \big[ \langle \bar{\alpha}_{n-1}, H \bar{\alpha}_{n-1} \rangle \big]   \leqslant  \frac{1}{n \gamma} \| \alpha_0 \|^2  + \frac{ \tr CH^{-1} }{n} .
\EEQ
\end{lemma}
\begin{proof}
The proof relies on the fact that cost functions are quadratic and our recursions are thus linear, allowing to obtain $\alpha_n$ in closed form. 
The sequence $(\alpha_n)$ satisfies a linear recursion, from which we get, for all $n \geqslant 1$:
\BEAS
\alpha_n & = & ( \idm - \gamma H)^n \alpha_0 + \gamma \sum_{k=1}^n ( \idm - \gamma H)^{n-k} \xi_k,
\EEAS
which leads to the first result using classical martingale second moment expansions  (which amount to considering $\xi_i$, $i=1,\dots,n$ independent, so that the variance of the sum is the sum of variances).
Moreover, using the identity $\sum_{k=0}^{n-1} ( \idm - \gamma H)^k = \big( \idm - ( \idm - \gamma H)^n \big) \big(\gamma H \big)^{-1}$, we get:
\BEAS
\bar{\alpha}_{n-1}
& = & \frac{1}{n}\sum_{k=0}^{n-1} ( \idm - \gamma H)^k   \alpha_0
+ \frac{\gamma}{n} \sum_{k=1}^{n-1} \sum_{j=1}^k ( \idm - \gamma H)^{k  - j } \xi_j
\\
& = & \frac{1}{n} \big( \idm - ( \idm - \gamma H)^n \big) \big(\gamma H \big)^{-1} \alpha_0
+ \frac{\gamma}{n} \sum_{k=1}^{n-1} \sum_{j=1}^k ( \idm - \gamma H)^{k  - j } \xi_j
\\
& = & \frac{1}{n}\big( \idm - ( \idm - \gamma H)^n \big) \big(\gamma H \big)^{-1}  \alpha_0
+   \frac{\gamma}{n}\sum_{j=1}^{n-1}  \bigg( \sum_{k=j}^{n-1}  ( \idm - \gamma H)^{k  - j } \bigg) \xi_j
\\
& = & \frac{1}{n} \big( \idm - ( \idm - \gamma H)^n \big) \big(\gamma H \big)^{-1} \alpha_0
+   \frac{\gamma}{n}\sum_{j=1}^{n-1}  \bigg( \sum_{k=0}^{n-1-j}  ( \idm - \gamma H)^{k  } \bigg) \xi_j
\\
& = & \frac{1}{n} \big( \idm - ( \idm - \gamma H)^n \big) \big(\gamma H \big)^{-1} \alpha_0
+   \frac{\gamma}{n}\sum_{j=1}^{n-1}   \big( \idm - ( \idm - \gamma H)^{n-j} \big) \big(\gamma H \big)^{-1}  \xi_j.
\EEAS
We then get, using standard martingale square moment inequalities (which here also amount to considering $\xi_i$, $i=1,\dots,n$ independent, so that the variance of the sum is the sum of variances):
\BEAS
\E \langle \bar{\alpha}_{n-1} , H \bar{\alpha}_{n-1} \rangle
& = & \frac{1}{n \gamma} \langle
\alpha_0,   { \big[ \idm - ( \idm - \gamma H)^n \big]^2 } \big( n \gamma H\big)^{-1} \alpha_0 \rangle \\
& & \hspace*{3cm}
+  \frac{1}{n^2}
\sum_{j=1}^{n-1} \tr  \big(   {\idm - ( \idm - \gamma H)^{n-j} } \big)^2 H^{-1} C \\
 & \leqslant & \frac{1}{n \gamma} \| \alpha_0 \|^2  + \frac{1}{n} \tr H^{-1} C,
 \EEAS
because for all $u \in [0,1]$, $\frac{ (1 - (1 - u)^n )^2}{nu} \leqslant 1$ (see Lemma~\ref{lemma:easy} in \mysec{easy}), and the second term is the sum of terms which are all less than  $ \tr H^{-1} C$. 

Note that we may replace the term $ \frac{1}{n \gamma} \| \alpha_0 \|^2  $
by $\displaystyle \frac{1}{n^2 \gamma^2} \langle \alpha_0, H^{-1} \alpha_0 \rangle$, which is only interesting when $\langle \alpha_0, H^{-1} \alpha_0 \rangle$ is small.
\end{proof}

\subsection{Proof principle}
The proof relies on an expansion of $\eta_n$ and $\bar{\eta}_{n-1}$ as polynomials in $\gamma$ due to~\cite{aguech2000perturbation}. This expansion is done separately for the noise process (i.e., when assuming $\eta_0=0$) and for the noise-free process that depends only on the initial conditions (i.e., when assuming that $\sigma=0$).  The bounds may then be added. 

Indeed, we have $\eta_n = M_{1}^n \eta_0 + \gamma \sum_{k=1}^n M_{k+1}^n \xi_k$, with $M_{i}^j = ( \idm - \gamma x_j \otimes x_j ) \cdots  ( \idm - \gamma x_{i} \otimes x_{i} )$ and $M_i^{i-1} = \idm$, and thus
$\bar{\eta}_n = \frac{1}{n+1} \sum_{i=0}^{n} 
\bigg[ 
M_{1}^i\eta_0 + \gamma \sum_{k=1}^i M_{k+1}^i \xi_k
\bigg]
= \frac{1}{n+1} \sum_{i=0}^{n} 
M_{1}^i \eta_0 
+ \gamma \sum_{k=1}^n \bigg(
\sum_{i=k}^n M_{k+1}^i
\bigg) \xi_k,
$
leading to
$$\| H^{1/2} \bar{\eta}_n \|_p
\leqslant
\bigg\| \frac{1}{n+1} \sum_{i=0}^{n} 
M_{1}^j \eta_0 \bigg\|_p
+ \bigg\| \gamma \sum_{k=1}^n \bigg(
\sum_{i=k}^n M_{k+1}^i
\bigg) \xi_k \bigg\|_p,
$$
for any $p \geqslant 2$ for which it is defined: the left term depends only on initial conditions and the right term depends only on the noise process (note the similarity with bias-variance decompositions).

\subsection{Initial conditions}
\label{app:A4} 
In this section, we assume that $\xi_n$ is uniformly equal to zero, and that $\gamma R^2 \leqslant 1$.

We thus have $\eta_n = ( \idm - \gamma x_n \otimes x_n) \eta_{n-1}$ and thus
\BEAS
\| \eta_n\|^2 & = & \| \eta_{n-1}\|^2 - 2 \gamma \langle \eta_{n-1}, ( x_n \otimes x_n )
\eta_{n-1} \rangle + \gamma^2 \langle \eta_{n-1}, ( x_n \otimes x_n )^2
\eta_{n-1} \rangle .
\EEAS
By taking expectations (first given $\F_{n-1}$, then unconditionally), we get:
\BEAS
\E \| \eta_n\|^2 &  \leqslant & \E \| \eta_{n-1}\|^2 - 2 \gamma  \E \langle \eta_{n-1}, H
\eta_{n-1} \rangle + \gamma^2 R^2 \E \langle \eta_{n-1}, H
\eta_{n-1} \rangle \mbox{ using } \E \|x_n\|^2 x_n \otimes x_n \preccurlyeq R^2 H, \\
&  \leqslant & \E \| \eta_{n-1}\|^2 -   \gamma  \E \langle \eta_{n-1}, H
\eta_{n-1} \rangle \mbox{ using } \gamma R^2 \leqslant 1 ,
\EEAS
from which we obtain, by summing from $1$ to $n$ and using convexity (note that Lemma~\ref{lemma:1} could  be used directly as well):
$$
\E \langle \bar{\eta}_{n-1}, H \bar{\eta}_{n-1} \rangle
\leqslant \frac{\| \eta_0\|^2}{n \gamma}.
$$
Here, it would be interesting to explore conditions under which the initial conditions may be forgotten at a rate $O(1/n^2)$, as obtained by~\cite{gradsto} in the strongly convex case.

\subsection{Noise process}
\label{app:A5} 
In this section, we assume that $\eta_0 = \theta_0 - \theta_\ast = 0 $ and $\gamma R^2 \leqslant 1$ (which implies $\gamma H \preccurlyeq \idm$). Following~\cite{aguech2000perturbation}, we recursively define the sequences $(\eta_n^r)_{n \geqslant 0}$ for $r \geqslant 0$ (and their averaged counterparts~$\bar{\eta}_n^r$):
\BIT
\item[--] The sequence $(\eta_n^0)$ is defined as $\eta^0_0 = \eta_0  =  0 $ and for $n \geqslant 1$,
$\eta^0_n = ( \idm - \gamma H ) \eta^0_{n-1} + \gamma \xi_n$.

\item[--] The sequence $(\eta_n^r)$ is defined from $(\eta_n^{r-1})$ as $\eta_0^r=0$ and, for all $n \geqslant 1$: 
\BEQ
\label{eq:etarec}
\eta_n^r = ( \idm - \gamma H)\eta_{n-1}^r + \gamma ( H  -x_n \otimes x_n) \eta_{n-1}^{r-1}.
\EEQ
\EIT

\paragraph{Recursion for expansion.}
We now show that the sequence $\eta_n - \sum_{i=0}^r \eta_n^i$ then satisfies the following recursion, for any $r \geqslant 0$ (which is of the same type than $(\eta_n)$):
\BEQ
\label{eq:etar}
\eta_n - \sum_{i=0}^r \eta_n^i = ( \idm - \gamma x_n \otimes x_n) \bigg(
\eta_{n-1} - \sum_{i=0}^r \eta_{n-1}^r
\bigg) + \gamma( H - x_n \otimes x_n )\eta_{n-1}^r.
\EEQ

In order to prove \eq{etar} by recursion, we have, for $r=0$,
\BEAS
\eta_n - \eta^0_n & = &   ( \idm - \gamma x_n \otimes x_n )  \eta_{n-1}   - ( \idm - \gamma H) \eta_{n-1}^0  \\
& = &   ( \idm - \gamma x_n \otimes x_n )  ( \eta_{n-1}  - \eta_{n-1}^0  ) + \gamma( H - x_n \otimes x_n) \eta_{n-1}^0,\EEAS
and, to go from $r$ to $r+1$:
\BEAS
\eta_n - \sum_{i=0}^{r+1}\eta_n^i
& = & 
( \idm - \gamma x_n \otimes x_n) \bigg(
\eta_{n-1} - \sum_{i=0}^r \eta_{n-1}^i
\bigg) + \gamma( H - x_n \otimes x_n )\eta_{n-1}^r  \\
& & \hspace*{3cm}  - 
 ( \idm - \gamma H)\eta_{n-1}^{r+1} - \gamma ( H  -x_n \otimes x_n) \eta_{n-1}^{r} \\
 & = & ( \idm - \gamma x_n \otimes x_n) \bigg(
\eta_{n-1} - \sum_{i=0}^{r+1} \eta_{n-1}^i
\bigg) + \gamma ( H - x_n \otimes x_n) \eta_{n-1}^{r+1}.
\EEAS

\paragraph{Bound on covariance operators.}
 We now show that we also have a bound on the covariance operator of $\eta_{n-1}^r$, for any $r \geqslant 0$ and $n \geqslant 2$:
\BEQ
\label{eq:etanoise}
\E \big[ \eta_{n-1}^r \otimes \eta_{n-1}^r \big] \preccurlyeq \gamma^{r+1} R^{2r}\sigma^2  \idm.
\EEQ

In order to prove \eq{etanoise} by recursion, we get for $r=0$:
\BEAS
\E \big[  \eta_{n-1}^0 \otimes \eta_{n-1}^0 \big] & \preccurlyeq
&  \gamma^2 \sigma^2 \sum_{k=1}^{n-1} ( \idm - \gamma H)^{2n-2-2k} H \\
 & \preccurlyeq
&   \gamma^2 \sigma^2 \big( { \idm -  ( \idm - \gamma H)^{2n-2} }\big)
\big({ \idm - ( \idm - \gamma H)^2}\big)^{-1} H \\
 & =
&    \gamma^2 \sigma^2 \big( { \idm -  ( \idm - \gamma H)^{2n-2} }\big)
\big({ 2 \gamma H - \gamma^2 H^2 }\big)^{-1} H
\\
& \preccurlyeq & 
\gamma^2 \sigma^2 \big( { \idm -  ( \idm - \gamma H)^{2n-2} }\big)
\big({   \gamma H  }\big)^{-1} H
\preccurlyeq \gamma \sigma^2  \idm .\EEAS

In order to go from $r$ to $r+1$, we have, using  Lemma~\ref{lemma:2}
and the fact that $\eta_{k-1}^r$ and $x_k$ are independent:
 
\BEAS
& & \E \big[ \eta_{n-1}^{r+1} \otimes \eta_{n-1}^{r+1} \big]   \\
& \preccurlyeq
& \gamma^2  \E \bigg[\sum_{k=1}^{n-1} ( \idm - \gamma H)^{n-1-k}
( H - x_k \otimes x_k)
\E \big[ \eta_{k-1}^r \otimes \eta_{k-1}^r\big]
( H - x_k \otimes x_k)
 ( \idm - \gamma H)^{n-1-k} \bigg] \\
  & \preccurlyeq
& \gamma^{r+3} R^{2r}  \sigma^2   \E \bigg[ \sum_{k=1}^{n-1} ( \idm - \gamma H)^{n-1-k}
( H - x_k \otimes x_k)^2
 ( \idm - \gamma H)^{n-1-k}  \bigg]  \mbox{ using the result for } r, \\
  & \preccurlyeq
&\gamma^{r+3} R^{2r+2}  \sigma^2    \sum_{k=1}^{n-1} ( \idm - \gamma H)^{2n-2-2k} H
\mbox{ using } \E ( x_k \otimes x_k - H)^2 \preccurlyeq \E \| x_k \|^2 x_k \otimes x_k \preccurlyeq R^2 H,
\\
  & \preccurlyeq
& \gamma^{r+2} R^{2r+2}  \sigma^2    \idm
.
\EEAS

\paragraph{Putting things together.}
We may apply Lemma~\ref{lemma:1} to the sequence $\big( \eta_n - \sum_{i=0}^r \eta_n^i \big)$, to get
\BEAS
\E \bigg\langle  \bar{\eta}_{n-1} - \sum_{i=0}^r \bar{\eta}_{n-1}^i , H \big(\bar{\eta}_{n-1} - \sum_{i=0}^r \bar{\eta}_{n-1}^i \big) \bigg\rangle
& \leqslant &  \frac{1}{1-\gamma R^2}  
\frac{\gamma}{n} \sum_{k=2}^n \E \| ( H - x_k \otimes x_k) \eta_{k-1}^r \|^2 
\\
& \leqslant &  \frac{1}{1-\gamma R^2}  \gamma^{r+2} \sigma^2 R^{2r+4}  .
\EEAS

We may now apply Lemma~\ref{lemma:2} to \eq{etarec}, to get, with a noise process $\xi_n^r = ( H - x_n \otimes x_n) \eta_{n-1}^{r-1}$ which is such that
$$
\E \big[ \xi_n^r  \otimes \xi_n^r  \big] \preccurlyeq \gamma^{r} R^{2r}  \sigma^2  H,
$$
\BEAS
\E \langle \bar{\eta}_{n-1}^r, H \bar{\eta}_{n-1}^r
\rangle & \leqslant &  \frac{1}{n} \gamma^{r} R^{2r}  d \sigma^2     .
\EEAS

We thus get, using Minkowski's inequality (i.e., triangle inequality for the norms $\| \cdot \|_p$):
\BEAS
\big(\E \langle \bar{\eta}_{n-1}, H  \bar{\eta}_{n-1} \rangle \big)^{1/2}
& \leqslant & \big(\frac{1}{1-\gamma R^2}  \gamma^{r+2} \sigma^2 R^{2r+4}
\big)^{1/2}   + \frac{\sigma \sqrt{d}}{\sqrt{n}}\sum_{i=0}^r  \gamma^{i/2} R^i
\\
& \leqslant &  \big(\frac{1}{1-\gamma R^2}  \gamma^{r+2} \sigma^2 R^{2r+4}
\big)^{1/2}      +    \frac{\sigma \sqrt{d}}{\sqrt{n}}
  \frac{1 - ( \sqrt{ \gamma R^2 } )^{r+1}}{ 1- \sqrt{ \gamma R^2 }}.\EEAS

This implies that for any $ \gamma R^2 < 1$, we obtain, by letting $r$ tend to $+\infty$:
$$
\big( \E \langle \bar{\eta}_{n-1}, H  \bar{\eta}_{n-1} \rangle \big)^{1/2}
\leqslant   \frac{\sigma \sqrt{d}}{\sqrt{n}}
\frac{1  }{ 1- \sqrt{ \gamma R^2 }}
.
$$

\subsection{Final bound}
We can now take results from Appendices~\ref{app:A4} and \ref{app:A5}, to get
$$\big( \E \langle \bar{\eta}_{n-1}, H  \bar{\eta}_{n-1} \rangle \big)^{1/2}
\leqslant   \frac{\sigma \sqrt{d}}{\sqrt{n}}
\frac{1  }{ 1- \sqrt{ \gamma R^2 }}
+ \frac{\| \eta_0\|^2}{n \gamma},
$$
which leads to the desired result.

\subsection{Proof of Lemma~\ref{lemma:easy} }
\label{sec:easy}
In this section, we state and prove a simple lemma.
\begin{lemma}
\label{lemma:easy} 
For any $u \in [0,1]$ and $n >0$, $ { (1 - (1 - u)^n )^2}{} \leqslant nu$.
\end{lemma}
\begin{proof}
Since $u \in [0,1]$, we have, $1 - (1-u)^n \leqslant 1$. Moreover, 
$n(1-u)^{n-1} \leqslant n$, and by integrating between $0$ and $u$, we get $1-(1-u)^n \leqslant nu$. By multiplying the two previous inequalities, we get the desired result.
\end{proof}

 \section{Proof of Theorem~\ref{theo:lmsp}}

Throughout the proof, we use the notation for $X \in \H$ a random vector, and $p$ any \emph{real} number greater than $1$, $\| X\|_p = \big( \E \| X\|^p \big)^{1/p}$. We first recall the Burkholder-Rosenthal-Pinelis (BRP) inequality~\cite[Theorem 4.1]{pinelis}. Let $p \in \rb$, $p \geqslant 2$ and $(\mathcal{F}_n)_{n \geqslant 0}$ be a sequence of increasing $\sigma$-fields, and $(x_n)_{n \geqslant 1}$ an adapted sequence of elements of $\H$, such that
$\E \big[ x_n | \F_{n-1} \big] = 0$, and $\| x_n\|_p$ is finite. Then,
\BEA
\label{eq:brp}
\bigg\| \sup_{k\in\{1,\dots,n\}} \bigg\| \sum_{j=1}^k x_j  \bigg \| \bigg\|_p
& \leqslant &  \sqrt{p} \bigg\|
\sum_{k=1}^n \E \big[ \| x_k\|^2 | \F_{k-1} \big]
\bigg\|_{p/2}^{1/2}
+ p \bigg\| \sup_{k\in\{1,\dots,n\}} \| x_k \| \bigg\|_p \\
\nonumber & \leqslant &  \sqrt{p} \bigg\|
\sum_{k=1}^n \E \big[ \| x_k\|^2 | \F_{k-1} \big]
\bigg\|_{p/2}^{1/2}
+ p \bigg\| \sup_{k\in\{1,\dots,n\}} \| x_k \|^2 \bigg\|_{p/2}^{1/2}.
\EEA
 
 We use the same notations than the proof of Theorem~\ref{theo:lms}, and the same proof principle: (a) splitting the contributions of the initial conditions and the noise, (b) providing a direct argument for the initial condition, and (c) performing an expansion for the noise contribution.

\paragraph{Consequences of assumptions.}
Note that by Cauchy-Schwarz inequality, assumption \textbf{(A7)} implies for all $z,t \in \H$, $\E  \langle z, x_n \rangle^2 \langle t, x_n \rangle^2 \leqslant \kappa \langle z, H z \rangle \langle t, H t\rangle$. It in turn implies that for all positive semi-definite self-adjoint operators $M,N$,
$\E  \langle x_n,M  x_n \rangle \langle x_n,N  x_n \rangle \leqslant \kappa  \tr ( MH) \tr (NH)$.

\subsection{Contribution of initial conditions}
\label{sec:B1}

When the noise is assumed to be zero, we  have $\eta_n = ( \idm - \gamma x_n \otimes x_n) \eta_{n-1}$  almost surely, and thus, since $ 0 \preccurlyeq \gamma x_n \otimes x_n \preccurlyeq \idm$,  $\|\eta_n\| \leqslant \|\eta_0\|$ almost surely, and
\BEAS
\| \eta_n\|^2 & = & \| \eta_{n-1}\|^2 - 2 \gamma \langle \eta_{n-1}, ( x_n \otimes x_n )
\eta_{n-1} \rangle + \gamma^2 \langle \eta_{n-1}, ( x_n \otimes x_n )^2
\eta_{n-1} \rangle \\
 & \leqslant & \| \eta_{n-1}\|^2 - 2 \gamma \langle \eta_{n-1}, ( x_n \otimes x_n )
\eta_{n-1} \rangle + \gamma \langle \eta_{n-1}, ( x_n \otimes x_n )
\eta_{n-1} \rangle  \\
& & \hspace*{6cm}  \mbox{ using } \|x_n\|^2 \leqslant R^2 \mbox{ and } \gamma R^2 \leqslant 1, \\
& = & \| \eta_{n-1}\|^2 -  \gamma \langle \eta_{n-1}, ( x_n \otimes x_n )
\eta_{n-1} \rangle ,
\EEAS
which we may write as
$$
\| \eta_n\|^2 - \| \eta_{n-1}\|^2 + \gamma  
\langle \eta_{n-1}, H 
\eta_{n-1} \rangle \leqslant   \gamma \langle \eta_{n-1}, ( H -  x_n \otimes x_n )
\eta_{n-1} \rangle \defeq M_n.
$$
We thus have:
$$
A_n \defeq \|\eta_n\|^2 + \gamma \sum_{k=1}^n \langle \eta_{k-1}, H
\eta_{k-1} \rangle \leqslant \| \eta_0\|^2 + \sum_{k=1}^n M_k.
$$
Note that we have
\BEAS
\E [ M_n^2 | \F_{n-1} ] & \leqslant &  \E \big[
\gamma^2 \langle \eta_{n-1},  ( x_n \otimes x_n )\eta_{n-1} \rangle ^2  | \F_{n-1} \big]
\leqslant \gamma^2R^2 \| \eta_0\|^2  \langle \eta_{n-1},   H \eta_{n-1} \rangle,
 \EEAS
 and $|M_n| \leqslant \gamma \| \eta_0\|^2 R^2  $.
We may now apply the Burkholder-Rosenthal-Pinelis inequality in \eq{brp}, to get:
\BEAS
\big\| A_n  \big\|_p
& \leqslant & \| \eta_0\|^2 + \sqrt{p} \bigg\|
\sum_{k=1}^n \E \big[ M_k^2 | \F_{k-1} \big]
\bigg\|_{p/2}^{1/2}
+ p \bigg\| \sup_{k\in\{1,\dots,n\}} |M_k| \bigg\|_p \\
& \leqslant & \| \eta_0\|^2 + \gamma \sqrt{p} \bigg\| \| \eta_0 \|^2 R^2 
\sum_{k=1}^n  \langle \eta_{k-1},  H \eta_{k-1} \rangle
\bigg\|_{p/2}^{1/2}
+ p \gamma R^2  \| \eta_0 \|^2 \\
& \leqslant & \| \eta_0\|^2 + \gamma^{1/2} R \sqrt{p} \| \eta_0 \|  \big\|A_n 
\big\|_{p/2}^{1/2}
+ p \gamma R^2  \| \eta_0 \|^2 \\
& \leqslant & \| \eta_0\|^2 + \gamma^{1/2} R \sqrt{p} \| \eta_0 \|  \big\|A_n 
\big\|_{p}^{1/2}
+ p \gamma R^2  \| \eta_0 \|^2.
\EEAS
We have used above that (a) $\sum_{k=1}^n  \langle \eta_{k-1},  H \eta_{k-1} \rangle \leqslant \frac{A_n}{\gamma}$ and that (b) $\big\|A_n 
\big\|_{p/2} \leqslant \big\|A_n 
\big\|_{p}$.
This leads to
 $$ \bigg( \big\| A_n  \big\|_p^{1/2} - \frac{1}{2} \gamma^{1/2} R \sqrt{p} \| \eta_0 \| 
 \bigg)^2 \leqslant \| \eta_0\|^2 + \frac{5p}{4} \gamma R^2  \| \eta_0 \|^2,
$$
which leads to
$$  \big\| A_n  \big\|_p^{1/2} - \frac{1}{2} \gamma^{1/2} R \sqrt{p} \| \eta_0 \| 
  = \| \eta_0\|  \sqrt{ 1  + \frac{5p \gamma R^2}{4} },
$$
\BEAS
\| A_n  \big\|_p & \leqslant &  \| \eta_0 \|^2 \bigg( 2 + \frac{5p \gamma R^2}{2} + \frac{p \gamma R^2}{2}  \bigg) \leqslant \| \eta_0 \|^2 ( 2 + 3 p \gamma R^2 ).
\EEAS
Finally, we obtain, for any $p\geqslant 2$
$$
\big\| \big\langle \bar{\eta}_{n-1}, H \bar{\eta}_{n-1} \big\rangle 
\big\|_p \leqslant \frac{\| \eta_0 \|^2}{n \gamma } ( 2 + 3 p \gamma R^2 ),
$$
i.e., by a change of variable $p \rightarrow \frac{p}{2}$, for any $p \geqslant 4$, we get
$$
\big\|H^{1/2} \bar{\eta}_{n-1}
\big\|_{p} \leqslant \big\| \langle  \bar{\eta}_{n-1}, H \bar{\eta}_{n-1} \rangle \|_{p/2}^{1/2} =  \frac{\| \eta_0 \|}{\sqrt{n \gamma} } \sqrt{ 2 + \frac{3 p}{2} \gamma R^2}.
$$

By using monotonicity of norms, we get, for any $p \in  [2,4]$:
$$
\big\|H^{1/2} \bar{\eta}_{n-1} 
\big\|_{p} \leqslant  \big\|H^{1/2} \bar{\eta}_{n-1} 
\big\|_{4} \leqslant  \frac{\| \eta_0 \|}{\sqrt{n \gamma} } \sqrt{ 2 + 6 \gamma R^2}
\leqslant \frac{\| \eta_0 \|}{\sqrt{n \gamma} } \sqrt{ 2 + 3 p \gamma R^2},
$$
which  is also valid for $p>4$.

Note that the constants in the bound above could be improved by using a proof by recursion.

\subsection{Contribution of the noise}
\label{sec:B2}
We follow the same proof technique than for Theorem~\ref{theo:lms} and consider the expansion based on the sequences $(\eta_{n}^r)_{n}$, for $r \geqslant 0$. We need (a) bounds on $\eta_n^0$, (b) a recursion on the magnitude (in $\| \cdot \|_p$ norm) of $\eta_n^r$ and (c) a control of the error made in the expansions.

\paragraph{Bound on $\bar{\eta}_n^0$.}
We start by a lemma similar to Lemma~\ref{lemma:2} but for all moments.  This will show a bound for the sequence $\bar{\eta}_n^0$. 

\begin{lemma}
 \label{lemma:2moments_weak}
 Assume $ \xi_n \in \H  $ is $\mathcal{F}_n$-measurable for a sequence of increasing $\sigma$-fields $(\mathcal{F}_n)$, $n \geqslant 1$.
Assume $\E [ \xi_n | \F_{n-1} ] = 0$,  $\E \big[ \| \xi_n\|^2 | \F_{n-1} \big]$ is finite. 
Assume moreover that for all $n \geqslant 1$, $ \E \big[ \xi_n \otimes \xi_n | \F_{n-1} \big]\preccurlyeq C$ and $\|  \xi_n\|_p   \leqslant \tau  R $ almost surely
for some $p \geqslant 2$.

Consider the recursion $\alpha_n = \big( \idm - \gamma H \big)  \alpha_{n-1}  + \gamma \xi_n$, with $\alpha_0=0$ and $\gamma H \preccurlyeq  \idm$.  Let $p \in \rb$, $p \geqslant 2$. Then:
\BEQ
\| H^{1/2}\bar{\alpha}_{n-1} \|_p  \leqslant 
   \frac{\sqrt{p}}{\sqrt{n}} \sqrt{  \tr C H^{-1} }
+  \frac{ \sqrt{\gamma} p R \tau }{\sqrt{n}}.
\EEQ

\end{lemma}
\begin{proof} 
We have, from the proof of Lemma~\ref{lemma:2}:
\BEAS
\bar{\alpha}_{n-1}& = &    \frac{\gamma}{n}\sum_{j=1}^{n-1}  \big(   {\idm - ( \idm - \gamma H)^{n-j} }\big) \big({\gamma H}\big)^{-1} \xi_j \\
\| H^{1/2}\bar{\alpha}_{n-1} \|_p & \leqslant &  \frac{\gamma}{n} \bigg\| \sum_{j=1}^{n-1}  \big(   {\idm - ( \idm - \gamma H)^{n-j} }\big) \big({\gamma H}\big)^{-1}  H^{1/2}\xi_j  \bigg\|_p \\
& \leqslant &    \frac{\gamma}{n} \bigg\| \sum_{j=1}^{n-1}  \beta_j \bigg\|_p,
\EEAS
with $\beta_j =  \big(   {\idm - ( \idm - \gamma H)^{n-j} }\big) \big({\gamma H}\big)^{-1} H^{1/2} \xi_j $. We have
\BEAS
\sum_{j=1}^{n-1} \E [ \| \beta_j \|^2  | \F_{j-1} ]
& =  &
\sum_{j=1}^{n-1} \tr \E \big[ \xi_j \otimes \xi_j | \F_{j-1} \big] H \bigg(  \frac{\idm - ( \idm - \gamma H)^{n-j} }{\gamma H}\bigg) ^2 \\
& \leqslant & \gamma^{-2} \sum_{j=1}^{n-1}   \E \big[ \langle \xi_j, H^{-1} \xi_j \rangle | \F_{j-1} \big] \leqslant n \gamma^{-2} \tr CH^{-1}   ,
 \EEAS
and 
\BEAS 
\| \beta_j \|_p
& \leqslant &   \lambda_{\max} \big[   \big(   {\idm - ( \idm - \gamma H)^{n-j} }\big) \big({\gamma H}\big)^{-1} H^{1/2}  \big] 
\|  \xi_j \|_p  \\
& \leqslant &   \gamma^{-1/2} \|  \xi_j \|_p  \max_{ u \in (0,1] } \frac{ 1 - ( 1-u)^{n-j} }{u^{1/2}}  \\
& \leqslant &   \frac{ \sqrt{n-j}}{\sqrt{\gamma}} \|  \xi_j \|_p
\leqslant  \frac{ \tau R \sqrt{n}}{\sqrt{\gamma}} ,
\EEAS
using Lemma~\ref{lemma:easy} in \mysec{easy}, and assumption \textbf{(A7)}. 

Using  Burkholder-Rosenthal-Pinelis inequality in \eq{brp}, we  then obtain
$$
\bigg\| \sup_{k \in \{1,\dots,n-1\}} \| H^{1/2}\bar{\alpha}_{k} \|  \bigg\|_p   \leqslant 
    \frac{\sqrt{p}}{\sqrt{n} }  \sqrt{ \tr CH^{-1} } +  p \frac{ \sqrt{\gamma }}{\sqrt{n}}\sigma R,$$
leading to the desired result.
\end{proof}

 \paragraph{Bounds on $ {\eta}_n^0$.}
Following the same proof technique as above, we have
$$\eta_n^0 =   \gamma \sum_{j=1}^n
( \idm - \gamma H)^{n-j}  \xi_{j},$$
from which we get, for any positive semidefinite operator $M$ such that $\tr M = 1$, using BRP's inequality:
\BEAS
\bigg\| \sup_{k\in\{1,\dots,n\}} \big\| M^{1/2} \eta_{k}^0 \| \bigg\|_p
& \leqslant &  \sqrt{p} \gamma \sigma \bigg\|
\sum_{j=1}^n   \tr H   ( \idm - \gamma H)^{n-j} M  ( \idm - \gamma H)^{n-j} 
\bigg\|_{p/2}^{1/2}
+ p \gamma  \tau R \\
& \leqslant &  \sqrt{p} \gamma \sigma \bigg\|
\frac{1}{\gamma} \tr M
\bigg\|_{p/2}^{1/2}
+ p \gamma  \tau R
\\
& \leqslant &  \frac{1}{R}  \sqrt{ p \gamma R^2}
( \sigma + \tau  \sqrt{ p \gamma R^2}),
\EEAS
leading to
\BEQ
\label{eq:M}
\sup_{ \tr M = 1} \bigg\| \sup_{k\in\{1,\dots,n\}} \big\| M^{1/2} \eta_{k}^0 \| \bigg\|_p
\leqslant  \frac{1}{R}  \sqrt{ p \gamma R^2}
( \sigma + \tau  \sqrt{ p \gamma R^2}).
\EEQ

\paragraph{Recursion on bounds on $\eta_n^r$.} We introduce the following quantity to control
the deviations of~$\eta_n^r$:
$$A_{r} = \sup_{\tr M =1} \bigg\| \sup_{k\in\{1,\dots,n\}} \big\| M^{1/2} \eta_{k}^r \big\| \bigg\|_p.$$
We have from \eq{M}, $A_0 \leqslant\frac{1}{R}  \sqrt{ p \gamma R^2}
( \sigma + \tau  \sqrt{ p \gamma R^2})$.

Since $\eta_n^r = ( \idm - \gamma H)\eta_{n-1}^r + \gamma ( H  -x_n \otimes x_n) \eta_{n-1}^{r-1}$, for all $n \geqslant 1$, we have the closed form expression
$$
\eta_{n}^r  = \gamma \sum_{k=2}^n ( \idm - \gamma H)^{n-k}  ( H  -x_k \otimes x_k) \eta_{k-1}^{r-1},
$$
and we may use BRP's inequality in \eq{brp} to get, for any $M$ such that $\tr M = 1$:
\BEAS
A_{r}
& \leqslant & B + C ,
\EEAS
with 
\BEAS
B & = &   \sqrt{p}\gamma
\bigg\|
\sum_{k=2}^n \langle  \eta_{k-1}^{r-1},
\E \big[ 
 ( H  -x_k \otimes x_k)    ( \idm - \gamma H)^{n-k} M   ( \idm - \gamma H)^{n-k} ( H  -x_k \otimes x_k) \big]
\eta_{k-1}^{r-1} \rangle 
\bigg\|_{p/2}^{1/2} \\
& \leqslant & \sqrt{p}\gamma
\bigg\|
\sum_{k=2}^n \langle  \eta_{k-1}^{r-1},\E \big[
 ( x_k \otimes x_k)    ( \idm - \gamma H)^{n-k} M   ( \idm - \gamma H)^{n-k}   (x_k \otimes x_k) \big]
\eta_{k-1}^{r-1} \rangle 
\bigg\|_{p/2}^{1/2}\\
&& \mbox{ using } \E \tr N ( H  -x_k \otimes x_k) M ( H  -x_k \otimes x_k) 
\leqslant \E \tr N  (x_k \otimes x_k)    H   (x_k \otimes x_k ), \\
& \leqslant & \sqrt{p}\gamma
\bigg\|
\sum_{k=2}^n \kappa \langle  \eta_{k-1}^{r-1}, H \eta_{k-1}^{r-1} \rangle 
\tr H   ( \idm - \gamma H)^{n-k} M   ( \idm - \gamma H)^{n-k}  \bigg\|_{p/2}^{1/2}
\\
& & \mbox{ using the kurtosis property,} \\
& \leqslant & \sqrt{p}\gamma \sqrt{\kappa} A_{r-1}
\bigg( \sum_{k=2}^n  
\tr H   ( \idm - \gamma H)^{n-k} M   ( \idm - \gamma H)^{n-k} 
\bigg)^{1/2} \\
&&\mbox{ using }  \langle  \eta_{k-1}^{r-1}, H \eta_{k-1}^{r-1} \rangle 
\leqslant \sup_{ k \in \{1,\dots,n\}} \langle  \eta_{k-1}^{r-1}, H \eta_{k-1}^{r-1} \rangle ,\\
& \leqslant & \sqrt{p}\gamma \sqrt{\kappa}  RA_{r-1} \bigg( \frac{1}{\gamma} \tr M\bigg)^{1/2}  =  \sqrt{p \gamma R^2   \kappa} A_{r-1},
\EEAS
and
\BEAS
C & = & p \gamma \bigg\|
\sup_{k \in \{2,\dots, n\}}
\langle  \eta_{k-1}^{r-1},
  ( H  -x_k \otimes x_k)    ( \idm - \gamma H)^{n-k} M   ( \idm - \gamma H)^{n-k}   ( H  -x_k \otimes x_k)  
\eta_{k-1}^{r-1} \rangle 
\bigg\|_{p/2}^{1/2} \\
& \leqslant & 
p \gamma \bigg\|
\sup_{k \in \{2,\dots, n\}}
\langle  \eta_{k-1}^{r-1},
  ( x_k \otimes x_k )    ( \idm - \gamma H)^{n-k} M   ( \idm - \gamma H)^{n-k}   ( x_k \otimes x_k)
\eta_{k-1}^{r-1} \rangle 
\bigg\|_{p/2}^{1/2} \\
& & 
+ p \gamma \bigg\|
\sup_{k \in \{2,\dots, n\}}
\langle  \eta_{k-1}^{r-1},
 H    ( \idm - \gamma H)^{n-k} M   ( \idm - \gamma H)^{n-k}    H
\eta_{k-1}^{r-1} \rangle 
\bigg\|_{p/2}^{1/2} \\
&&\mbox{ using Minkowski's inequality},
\\
& \leqslant & 
+ p \gamma \bigg(
\sum_{k  = 2}^n
\E \big[
\langle  \eta_{k-1}^{r-1},
  ( x_k \otimes x_k)    ( \idm - \gamma H)^{n-k} M   ( \idm - \gamma H)^{n-k}   (  x_k \otimes x_k)  
\eta_{k-1}^{r-1} \rangle^{p/2} \big]
\bigg)^{1/p}
\\
& & + p \gamma R^2 A_{r-1}\\
& & \mbox{  bounding the supremum by a sum, and using }
 H    ( \idm - \gamma H)^{n-k} M   ( \idm - \gamma H)^{n-k}    H \preccurlyeq H^2,
 \\
& \leqslant &  p \gamma \bigg(
\sum_{k  = 2}^n
\E \big[ \langle  \eta_{k-1}^{r-1}, x_k \rangle^p
\langle  x_k,   ( \idm - \gamma H)^{n-k} M   ( \idm - \gamma H)^{n-k}   x_k  \rangle^{p/2} \big]
\bigg)^{1/p} + p \gamma R^2  A_{r-1}
\\
& \leqslant & p \gamma  RA_{r-1} \bigg(
\sum_{k  = 2}^n
\E \big[  
\langle  x_k,   ( \idm - \gamma H)^{n-k} M   ( \idm - \gamma H)^{n-k}     x_k  \rangle^{p/2} \big]
\bigg)^{1/p} + p \gamma R^2  A_{r-1}  \\
&& \mbox{ by conditioning with respect to } x_k, \\
& \leqslant &
  p \gamma RA_{r-1}  \bigg(
\sum_{k  = 2}^n
\E \big[  
\langle  x_k,   ( \idm - \gamma H)^{n-k} M   ( \idm - \gamma H)^{n-k}    x_k  \rangle (R^2 \tr M)^{p/2-1} \big]
\bigg)^{1/p} + p \gamma R^2   A_{r-1}  \\
\\
&&\mbox{bounding } \langle  x_k,   ( \idm - \gamma H)^{n-k} M   ( \idm - \gamma H)^{n-k}    x_k  \rangle  \mbox{ by }  R^2 \tr M,\\
& \leqslant &  p \gamma R^2   A_{r-1}   
+ p \gamma RA_{r-1}   \bigg(
(R^2)^{p/2-1}
\sum_{k  = 2}^n
\tr H   ( \idm - \gamma H)^{n-k} M   ( \idm - \gamma H)^{n-k} 
\bigg)^{1/p}
\\
& \leqslant &   p \gamma R^2   A_{r-1}  + p \gamma RA_{r-1}   \bigg(
(R^2)^{p/2-1} R^2 \gamma^{-1}
\bigg)^{1/p} =    p \gamma R^2   A_{r-1}  + p (\gamma R^2)^{1-1/p} RA_{r-1} .\EEAS

This implies that
$ A_r  
\leqslant A_0 \big( \sqrt{p \gamma R^2   \kappa}  + p \gamma R^2   + p (\gamma R^2)^{1-1/p} \big)^r $, which in turn implies, from \eq{M},
\BEA
\label{eq:AA}
 A_r  
\leqslant  \sqrt{ p \gamma R^2} ( \sigma + \tau  \sqrt{ p \gamma R^2}  )
 \big( \sqrt{p \gamma R^2   \kappa}  + p \gamma R^2   + p (\gamma R^2)^{1-1/p} \big)^r .
\EEA

The condition on $\gamma$ will come from the requirement that $\sqrt{p \gamma R^2   \kappa}  + p \gamma R^2   + p (\gamma R^2)^{1-1/p} < 1$.

\paragraph{Bound on $\| H^{1/2} \bar{\eta}_{n-1}^r \|$.}

We have the closed-form expression:
\BEAS
\bar{\eta}_{n-1}^{r} & = &  \frac{\gamma}{n} \sum_{j=2}^{n-1} \frac{\idm - ( \idm - \gamma H)^{n-j} }{\gamma H} ( H - x_j \otimes x_j) \eta_{j-1}^{r-1},
\EEAS
 leading to, using BRP's inequality in \eq{brp}, similar arguments than in the previous bounds,  $\big(  {\idm - ( \idm - \gamma H)^{n-j} } \big)^2 \big({\gamma H} \big)^{-2} H \preccurlyeq  H^{-1} \gamma$ and 
 $\big(  {\idm - ( \idm - \gamma H)^{n-j} } \big)^2 \big({\gamma H} \big)^{-2} H \preccurlyeq  \frac{n}{\gamma} \idm$:
\BEAS
& & \| H^{1/2} \bar{\eta}_{n-1}^{r} \|_p \\
& \leqslant &  \frac{\gamma \sqrt{p}}{n}
\bigg\|
\sum_{j=2}^{n-1} \langle \eta_{j-1}^{r-1}, \E \bigg[  ( H - x_j \otimes x_j) 
\big(  {\idm - ( \idm - \gamma H)^{n-j} } \big)^2 \big({\gamma H} \big)^{-2}  H
 ( H - x_j \otimes x_j) \bigg] \eta_{j-1}^{r-1} \rangle
\bigg\|_{p/2}^{1/2} \\
& & + \frac{\gamma  {p}}{n} \bigg\|
\sup_{ j \in \{2,\dots,n-1\}} \langle \eta_{j-1}^{r-1},  ( H - x_j \otimes x_j) 
\big(  {\idm - ( \idm - \gamma H)^{n-j} } \big)^2 \big({\gamma H} \big)^{-2} H
 ( H - x_j \otimes x_j)  \eta_{j-1}^{r-1} \rangle
\bigg\|_{p/2}^{1/2}\\
& \leqslant &  \frac{  \gamma \sqrt{p}}{n}
\bigg\|
\sum_{j=2}^{n-1} \langle \eta_{j-1}^{r-1}, \E \bigg[  ( H - x_j \otimes x_j) \gamma^{-2} H^{-1}
 ( H - x_j \otimes x_j) \bigg] \eta_{j-1}^{r-1} \rangle
\bigg\|_{p/2}^{1/2} \\
& & + \frac{\gamma  {p}}{n} \bigg\|
\sup_{ j \in \{2,\dots,n-1\}} \langle \eta_{j-1}^{r-1},    H
\frac{n}{\gamma}
 H \eta_{j-1}^{r-1} \rangle
\bigg\|_{p/2}^{1/2}\\
& & +
\frac{\gamma  {p}}{n} \bigg\|
\sup_{ j \in \{2,\dots,n-1\}} \langle \eta_{j-1}^{r-1},  (  x_j \otimes x_j) 
\big(  {\idm - ( \idm - \gamma H)^{n-j} } \big)^2 \big({\gamma H} \big)^{-2} H
 (  x_j \otimes x_j)  \eta_{j-1}^{r-1} \rangle
\bigg\|_{p/2}^{1/2}
\\
& \leqslant &  \frac{  \sqrt{p}}{n}
\bigg\|
\sum_{j=2}^{n-1} \kappa \langle  \eta_{j-1}^{r-1}, H\eta_{j-1}^{r-1} \rangle   \tr  H^{-1} H 
\bigg\|_{p/2}^{1/2} +  \frac{\sqrt{\gamma} {p} R^2 }{\sqrt{n} }     A_{r-1} 
\\
& & 
+
\frac{\gamma  {p}}{n} \bigg(
\sum_{j=2}^{n-1}  \E \bigg[ \langle \eta_{j-1}^{r-1},   x_j \rangle^p 
\langle x_j,
\big(  {\idm - ( \idm - \gamma H)^{n-j} } \big)^2 \big({\gamma H} \big)^{-2} H
 x_j \rangle^{p/2} \bigg] 
\bigg )^{1/p}
\\
 & \leqslant &  \frac{  \sqrt{p \kappa d }}{\sqrt{n}} R A_{r-1}
+ \frac{\sqrt{\gamma} {p} R^2 }{\sqrt{n} }     A_{r-1}  \\
& & 
+
\frac{\gamma  {p}}{n} RA_{r-1} \bigg(
\sum_{j=2}^{n-1}  \E \bigg[  \langle x_j,
\big(  {\idm - ( \idm - \gamma H)^{n-j} } \big)^2 \big({\gamma H} \big)^{-2}  H
 x_j \rangle^{p/2} \bigg] 
\bigg )^{1/p}
\\
 & \leqslant &   \frac{\sqrt{p}}{\sqrt{n}}  R  A_{r-1}  
( \sqrt{ \gamma p R^2} + \sqrt{\kappa d } ) +
\frac{\gamma  {p}}{n} R A_{r-1} \bigg(
\sum_{j=2}^{n-1}  \frac{d}{\gamma^2} 
\big(\frac{n}{\gamma}R^2 \big)^{p/2-1}
\bigg )^{1/p}
\\
& \leqslant &    \frac{\sqrt{p}}{\sqrt{n}}  R A_{r-1}  
( \sqrt{ \gamma p R^2} + \sqrt{\kappa d } ) 
+
\frac{\gamma  {p}}{n} RA_{r-1} 
n^{1/2} \gamma^{-1/2 - 1/p} R^{1 - 2 / p} d^{1/p}\\
& = &    \frac{\sqrt{p}}{\sqrt{n}}  R A_{r-1}  
( \sqrt{ \gamma p R^2} + \sqrt{\kappa d } ) 
+
\frac{   {p}}{\sqrt{n}} RA_{r-1} ( \gamma R^2)^{1/2 - 1 / p} d^{1/p} \\
& = &    \frac{\sqrt{p}}{\sqrt{n}}  R  A_{r-1}  
\bigg[ \sqrt{ \gamma p R^2} + \sqrt{\kappa d }  + \sqrt{p} ( \gamma R^2)^{1/2 - 1 / p} d^{1/p}\bigg]\\
& \leqslant &    \frac{\sqrt{p}}{\sqrt{n}}   
  \sqrt{ p \gamma R^2}( \sigma + \tau  \sqrt{ p \gamma R^2} )
( \sqrt{ \gamma p R^2} + \sqrt{\kappa d }  + \sqrt{p}d^{1/p} ( \gamma R^2)^{1/2 - 1 / p} \big)
\big( \sqrt{p \gamma R^2   \kappa}  + p \gamma R^2   + p (\gamma R^2)^{1-1/p} \big)^{r-1}
,
\EEAS
using \eq{AA}.

We may then impose a restriction on $\gamma R^2$, i.e., 
$  \gamma R^2   \leqslant   \frac{1}{\alpha\kappa p } $ with $\alpha > 1$.
We then have
\BEAS
\sqrt{p \gamma R^2   \kappa}  + p \gamma R^2   + p (\gamma R^2)^{1-1/p} 
& \leqslant & \frac{1}{\sqrt{\alpha}} + \frac{1}{\alpha} +  ( \alpha p )^{-1+1/p}  \\
& \leqslant & \frac{1}{\sqrt{\alpha}} + \frac{1}{\alpha} + ( \alpha )^{-1 } (\alpha p)^{1/p} \\
& \leqslant & \frac{1}{\sqrt{\alpha}} + \frac{1}{\alpha} + ( \alpha )^{-1 } (\alpha 2)^{1/2} \mbox{ if } \alpha > 2,\\
& = & \frac{1}{\sqrt{\alpha}} + \frac{1}{\alpha} + \sqrt{\frac{2}{\alpha}}.
\EEAS
With $\alpha = 12$, we obtain a bound of $0.781 \leqslant \frac{8}{10}$ above.

This leads to the bound
\BEA
\nonumber \| H^{1/2} \bar{\eta}_{n-1}^{r} \|_p
& \leqslant &
\frac{\sqrt{p}}{\sqrt{n}}   \sqrt{ p \gamma R^2}
\big( \sigma + \frac{\tau}{\sqrt{12} \sqrt{\kappa}} \big)
( \frac{1}{ \sqrt{12 \kappa} } + \sqrt{\kappa d }  + 
d^{1/p} \sqrt{p} \big( \frac{1}{12p} \big)^{ ( 1/2-1/p)/(1-1/p)}\big)
\big(  8/10 \big)^{r-1}\\
\nonumber& \leqslant &
\frac{\sqrt{p}}{\sqrt{n}}   \sqrt{ p \gamma R^2}
\tau ( 1 + \frac{1}{\sqrt{12}} )
\sqrt{\kappa d } ( \frac{1}{\sqrt{12} } +  1 + \sup_{p \geqslant 2} \sqrt{p} \big( \frac{1}{12p} \big)^{ ( 1/2-1/p)/(1-1/p)}\big)
\big(  8/10 \big)^{r-1}\\
\nonumber& \leqslant &
\frac{\sqrt{p}}{\sqrt{n}}   \sqrt{ p \gamma R^2}
\tau ( 1 + \frac{1}{\sqrt{12}} )
\sqrt{\kappa d } ( \frac{1}{\sqrt{12} } +  1 + \sqrt{2} \big)
\big(  8/10 \big)^{r-1}\\
\label{eq:BB} & \leqslant & \frac{7}{2}
\frac{\sqrt{p}}{\sqrt{n}}   \sqrt{ p \gamma R^2}
\tau\sqrt{\kappa d }  \big(  8/10 \big)^{r-1}.
\EEA

\paragraph{Bound on $\| H^{1/2} ( \bar{\eta}_{n-1} - \sum_{i=0}^r \bar{\eta}_{n-1}^i )  \|_p$.}
From \eq{etar} and the fact that $ 0 \preccurlyeq \idm - \gamma x_n  \otimes x_n \preccurlyeq \idm$ almost surely, we get:
\BEAS
\nonumber
  \bigg\| \eta_n - \sum_{i=0}^r \eta_n^i\bigg\|_p
& \leqslant & 
 \bigg\| \eta_{n-1} - \sum_{i=0}^r \eta_{n-1}^i\bigg\|_p
+ \gamma  \big\| 
( H - x_n \otimes x_n ) \eta_{n-1}^r \big\|_p \\
\nonumber& \leqslant & 
 \bigg\| \eta_{n-1} - \sum_{i=0}^r \eta_{n-1}^i\bigg\|_p
+ \gamma  R \big\| \langle x_n , \eta_{n-1}^r \rangle\big\|_p \\
& \leqslant & 
 \bigg\| \eta_{n-1} - \sum_{i=0}^r \eta_{n-1}^i\bigg\|_p
+ \gamma  R^2 A_r .
\EEAS
This implies that
\BEA
\label{eq:CC} \bigg\| H^{1/2} ( \bar{\eta}_{n-1} - \sum_{i=0}^r \bar{\eta}_{n-1}^i )  \bigg\|_p
& \leqslant & R  \bigg\| \eta_n - \sum_{i=0}^r \eta_n^i\bigg\|_p
\leqslant n\gamma  R^3 A_r .
\EEA

\paragraph{Putting things together.}
We get by combining Lemma~\ref{lemma:2moments_weak} with \eq{BB} and \eq{CC} and then  letting $r$ tends to infinity,
\BEAS
\| H^{1/2} \bar{\eta}_{n-1}^{r} \|_p
& \leqslant & 
 \sum_{i=1}^r 
\big\|  H^{1/2} \bar{\eta}_{n-1}^i  \big\|_p + \big\|  H^{1/2} \bar{\eta}_{n-1}^0 \big\|_p  + \bigg\| H^{1/2} ( \bar{\eta}_{n-1} - \sum_{i=0}^r \bar{\eta}_{n-1}^i )  \bigg\|_p \\
& \leqslant & \bigg\{ \frac{7}{2}
\frac{\sqrt{p}}{\sqrt{n}}   \sqrt{ p \gamma R^2}
\tau
\sqrt{\kappa d } \frac{1 - (8/10)^r }{1 - 8/10} \bigg\}
+ \bigg\{ \frac{\sqrt{p d} \sigma}{\sqrt{n}}  
+  \frac{ \sqrt{\gamma} p R \tau }{\sqrt{n}} \bigg\}  + O ( (8/10)^r)\\
& \leqslant &  \frac{\sqrt{p d} \sigma}{\sqrt{n}}  
+  \frac{\sqrt{p}}{\sqrt{n}}   \sqrt{ p \gamma R^2}
\tau
\sqrt{\kappa d } ( 1 + \frac{7}{2} \frac{10}{2} ) \leqslant \frac{\sqrt{p d} \sigma}{\sqrt{n}}  
+  18.5 \frac{\sqrt{p}}{\sqrt{n}}   \sqrt{ p \gamma R^2}
\tau
\sqrt{\kappa d } \\
& \leqslant & \frac{\sqrt{p d} \sigma}{\sqrt{n}}  
+  \frac{18.5}{\sqrt{12}} \frac{\sqrt{pd} \tau }{\sqrt{n}}   \leqslant 
 \frac{\sqrt{p d} (\sigma + 6 \tau)}{\sqrt{n}}  \leqslant  7 \frac{\sqrt{p d} \tau}{\sqrt{n}}  . \EEAS

\subsection{Final bound}
For $\gamma \leqslant \frac{1}{12 \kappa p R^2}$, we obtain, from the last equations of \mysec{B1} and \mysec{B2},
\BEAS
\| H^{1/2} \bar{\eta}_{n-1}^{r} \|_p
& \leqslant &
7 \frac{\sqrt{p d} \tau}{\sqrt{n}} + \frac{\| \eta_0 \|}{\sqrt{n \gamma} } \sqrt{ 2 + 3 p \gamma R^2} \\
& \leqslant &
\frac{\sqrt{p}}{\sqrt{n}}
\bigg( 7 \sqrt{d} \tau + R \| \eta_0\| \sqrt{ 3 + \frac{2}{\gamma p R^2} } \bigg).
\EEAS

Moreover, when $\gamma =  \frac{1}{12 \kappa p R^2}$, we have:
 \BEAS
\| H^{1/2} \bar{\eta}_{n-1}^{r} \|_p
 & \leqslant &
7 \frac{\sqrt{p d} \tau}{\sqrt{n}} + \frac{\| \eta_0 \|}{\sqrt{n } }   \sqrt{ 12 \kappa p R^2} \sqrt{ 2 +  \frac{1}{4}}
\\
& \leqslant &
7 \frac{\sqrt{p d} \tau}{\sqrt{n}} + \frac{6 R \| \eta_0 \|}{\sqrt{n } }   \sqrt{   \kappa p } 
= \frac{ \sqrt{p}}{\sqrt{n}}  \bigg(
 7 \sqrt{d} \tau + 6 \sqrt{ \kappa} R \| \theta_0 - \theta_\ast\|
\bigg).
\EEAS

 \subsection{Proof of Corollary~\ref{cor:bound}}

 We have from the previous proposition, for $\gamma \leqslant \frac{1}{12 \kappa p R^2}$:
 \BEAS
 \big( \E \big| f( \bar{\theta}_{n-1} ) - f(\theta_\ast) \big|^p \big)^{1/p}
 & \leqslant  & \frac{  1 }{2 n }  \bigg(
 \sqrt{p} \big[  7 \tau \sqrt{d}  +  R \| \theta_0 - \theta_\ast\| \sqrt{ 3} \big]  +   R \| \theta_0 - \theta_\ast\| \sqrt{ \frac{2}{\gamma  R^2} }
\bigg)^2 \\
& \leqslant  & \frac{  1 }{2 n }  \bigg(
 \sqrt{p}\  \square  +   \triangle   \sqrt{ \frac{1}{\eta} }
\bigg)^2 ,
\EEAS
with $\eta = 12 \gamma \kappa R^2 \leqslant 1/p $, and $\square =  7 \tau \sqrt{d}  +  R \| \theta_0 - \theta_\ast\| \sqrt{ 3} $ and $ \triangle = R \| \theta_0 - \theta_\ast\| \sqrt{24 \kappa  }$.

This leads to, using Markov's inequality:
\BEAS
\P \bigg(
f( \bar{\theta}_{n-1} ) - f(\theta_\ast)  \geqslant \frac{t}{2n} 
\bigg) \leqslant \bigg(
\frac{\sqrt{p}\  \square  +   \triangle   \sqrt{ 1/\eta }}{\sqrt{t}}
\bigg)^{2p}.
\EEAS
 This leads to, with $p = \frac{1}{\eta}$,
 \BEAS
\P \bigg(
f( \bar{\theta}_{n-1} ) - f(\theta_\ast)  \geqslant \frac{t}{2n} 
\bigg) \leqslant \bigg(
\frac{  (\square  +   \triangle)^2   }{ {\eta t}}
\bigg)^{1/\eta}.
\EEAS
This leads to
\BEQ
\P \bigg(
f( \bar{\theta}_{n-1} ) - f(\theta_\ast)  \geqslant \frac{t}{2n} \big[ 7 \tau \sqrt{d} + R \| \theta_0 - \theta_\ast\| ( \sqrt{3} + \sqrt{24 \kappa} ) \big]^2
\bigg) \leqslant \bigg(
\frac{ 1}  { 12 \gamma \kappa R^2 t }
\bigg)^{1/(12 \gamma \kappa R^2 )}.
\EEQ

Thus the large deviations decay as power of $t$, with a power that decays as $1/( 12 \gamma \kappa R^2 )$. If $\gamma$ is small, the deviations are lighter.

In order to get the desired result, we simply take $t = \frac{1}{12 \gamma \kappa R^2} \delta^{- 12 \kappa \gamma R^2}$.

\section{Proof of Theorem \ref{theo:log}}

The proof relies mostly on properties of approximate Newton steps: $\theta_1 \defeq \tilde{\theta}_n$ is an approximate minimizer of $f$, and  $\theta_3 \defeq \zeta_n$ is an approximate minimizer of the  associated quadratic problem.

In terms of convergence rates, $\theta_1$ will be $(1/\sqrt{n})$-optimal, while $\theta_3$ will be $(1/n)$-optimal for the quadratic problem because of previous results on averaged LMS. A classical property is that a single Newton step   squares the error. Therefore, the full Newton step should have an error which is the square of the one of $\theta_1$, i.e., $O(1/n)$. Overall, since $\theta_3$ approaches the full Newton step with rate $O(1/n)$, this makes a bound of $O(1/n)$.

 In \mysec{appnew}, we provide a general \emph{deterministic} result on the Newton step, while in \mysec{sto}, we combine with two stochastic approximation results, making the informal reasoning above more precise.

\subsection{Approximate Newton step}
\label{sec:appnew}
In this section, we study the effect of an approximate Newton step. We consider $\theta_1 \in \H$, the Newton iterate $\theta_2 = \theta_1 - f''(\theta_1)^{-1} f'(\theta_1)$, and an approximation $\theta_3$ of $\theta_2$. In the next proposition, we provide a bound on $f(\theta_3) -f(\theta_\ast)$, under different conditions, whether $\theta_1$ is close to optimal for~$f$, and/or $\theta_3$ is close to optimal for the quadratic approximation around $\theta_1$ (i.e., close to $\theta_2$). 
\eq{always} corresponds to the least-favorable situations where both errors are small, while \eq{almost1} and \eq{almost2} consider cases where $\theta_1$ is sufficiently good.
See proof in \mysec{proofprop1}. These three cases are necessary for the probabilistic control.

\begin{proposition}[Approximate Newton step]
\label{prop:approx_newton}
Assume \textbf{(B3-4)}, and $\theta_1, \theta_2, \theta_3 \in \H$ such that   $f(\theta_1) - f(\theta_\ast) = \varepsilon_1$,  $\theta_2 = \theta_1 - f''(\theta_1)^{-1} f'(\theta_1)$ and $ \frac{1}{2} \langle \theta_3 - \theta_2 , f''(\theta_1)(\theta_3 - \theta_2) \rangle=  \varepsilon_2$. 
Then, if $t^2 = \varepsilon_1 \kappa \rho$,
\BEQ
\label{eq:always}
 f(\theta_3) - f(\theta_\ast) 
\leqslant \varepsilon_1   +  \sqrt{  2 \rho  \varepsilon_2} e^{ \sqrt{3 + t^2} t /  2}   +\sqrt{\rho}  e^{  \sqrt{3 + t^2} t   }  ( \sqrt{3} + 2 t ) \sqrt{\varepsilon_1}   .
\EEQ
If $ t = \sqrt{ \varepsilon_1 \kappa \rho} \leqslant 1/16 $, then
\BEQ
\label{eq:almost1}
f(\theta_3) - f(\theta_\ast) \leqslant  57
 \kappa \rho \varepsilon_1^2
+  2\sqrt{ \rho \varepsilon_2} .\EEQ
Moreover, if $ t = \sqrt{ \varepsilon_1 \kappa \rho} \leqslant 1/16 $ and $\varepsilon_2 \kappa \rho \leqslant 1/16$, then
\BEQ
\label{eq:almost2}
f(\theta_3) - f(\theta_\ast) \leqslant  57
 \kappa \rho \varepsilon_1^2
+ 
12
\varepsilon_2 .
\EEQ
\end{proposition}

Note that in the favorable situation in \eq{almost1}, we get error of the form $O(\varepsilon_1^2 + \varepsilon_2)$. It essentially suffices now to show that in our set-up, in a probabilistic sense to be determined,
$\varepsilon_1 = O(1/\sqrt{n})$ and $\varepsilon_2 = O(1/ {n})$, while controlling the unfavorable situations.

\subsection{Stochastic analysis}
\label{sec:sto}

We consider the following two-step algorithm:
\BIT
\item[--] Starting from any initialization $\theta_0$, run $n$ iterations of averaged stochastic gradient descent to  get $\theta_1$,
\item[--] Run from $\theta_1$ $n$ steps of LMS on the quadratic approximation around $\theta_1$, to get $\theta_3$, which is an approximation of the Newton step $\theta_2$.
\EIT

We consider the events
$$A_1 = \bigg\{  f(\theta_1) - f(\theta_\ast) \leqslant \frac{1}{16^2} (\kappa \rho)^{-1}  \bigg\}
= \bigg\{ \varepsilon_1 \leqslant \frac{1}{16^2} (\kappa \rho)^{-1}  \bigg\} $$

and $$A_2  = \bigg\{ \frac{1}{2} \langle \theta_3 - \theta_2 , f''(\theta_1)(\theta_3 - \theta_2) \rangle\leqslant  \frac{1}{16} (\kappa \rho)^{-1}  \bigg\}
= \bigg\{ \varepsilon_2 \leqslant  \frac{1}{16} (\kappa \rho)^{-1}  \bigg\}.$$

We denote by $\mathcal{G}_1$ the $\sigma$-field generated by the first $n$ observations (the ones used to define $\theta_1$).
We have, by separating all events, i.e.,  using $1 = 1_{A_1} 1_{A_2} + 1_{A_1} 1_{A_2^{\sf c} }  + 1_{A_1^{\sf c}}$:
\BEA
\nonumber &&\E \big[ f(\theta_3) - f(\theta_\ast)  \big| \mathcal{G}_1 \big ] \\
\nonumber & = &  \E \big[ 1_{A_1} 1_{A_2} \big( f(\theta_3) - f(\theta_\ast)  \big) \big|  \mathcal{G}_1 \big ]
+ \E \big[ 1_{A_1}  1_{A_2^{\sf c}} \big( f(\theta_3) - f(\theta_\ast)  \big) \big|  \mathcal{G}_1 \big ]
+ \E \big[ 1_{A_1^{\sf c}} \big( f(\theta_3) - f(\theta_\ast)  \big) \big|  \mathcal{G}_1 \big ]
 \\
\nonumber & \leqslant &  \E \big[ 1_{A_1} 1_{A_2} \big( 57
 \kappa \rho \varepsilon_1^2
+ 
12
\varepsilon_2  \big) \big|  \mathcal{G}_1 \big ]
+ \E \big[ 1_{A_1}  1_{A_2^{\sf c}} \big(57
 \kappa \rho \varepsilon_1^2
+  2\sqrt{  \rho \varepsilon_2}  \big) \big|  \mathcal{G}_1 \big ] \\
\nonumber & & 
+ \E \big[ 1_{A_1^{\sf c}} \big(
\varepsilon_1   +  \sqrt{  2 \rho  \varepsilon_2} e^{ \sqrt{3 + t^2} t /  2}   +\sqrt{\rho}  e^{  \sqrt{3 + t^2} t   }  ( \sqrt{3} + 2 t ) \sqrt{\varepsilon_1}   \big) \big|  \mathcal{G}_1 \big ] \mbox{ using Prop.~\ref{prop:approx_newton},}
 \\
\nonumber & \leqslant & 57
 \kappa \rho \varepsilon_1^2
 + 12  \E \big[  1_{A_1} \varepsilon_2   \big|  \mathcal{G}_1 \big ]
+ \E \big[ 1_{A_1}  1_{A_2^{\sf c}} \big(   2\sqrt{ \rho \varepsilon_2}  \big) \big|  \mathcal{G}_1 \big ] \\
\nonumber& & 
+   1_{A_1^{\sf c}} \bigg(\varepsilon_1   +  \E \big[ \sqrt{\varepsilon_2} | \mathcal{G}_1 \big] \sqrt{  2 \rho  } e^{ \sqrt{3 + t^2}t/  2}   +\sqrt{\rho}  e^{  \sqrt{3 + t^2} t   }  ( \sqrt{3} + 2 t ) \sqrt{\varepsilon_1}   \bigg) 
\\
\nonumber& \leqslant & 57
 \kappa \rho \varepsilon_1^2
 + 12 \times 1_{A_1} \E \big[  \varepsilon_2   \big|  \mathcal{G}_1 \big ]
+ 2  \sqrt{ \rho} 1_{A_1}  \sqrt{\P(A_2^{\sf c}| \mathcal{G}_1)}  \sqrt{ \E \big[ \varepsilon_2 \big|  \mathcal{G}_1 \big] } \mbox{ using Cauchy-Schwarz inequality},\\
\nonumber& & 
+   1_{A_1^{\sf c}} \bigg(\varepsilon_1   +  \E \big[ \sqrt{\varepsilon_2} | \mathcal{G}_1 \big] \sqrt{  2 \rho  } e^{ \sqrt{3 + t^2}t/  2}   +\sqrt{\rho}  e^{  \sqrt{3 + t^2} t   }  ( \sqrt{3} + 2 t ) \sqrt{\varepsilon_1}   \bigg) 
\\
\nonumber& =  & 57
 \kappa \rho \varepsilon_1^2
 + 12 \times 1_{A_1} \E \big[  \varepsilon_2   \big|  \mathcal{G}_1 \big ]
+  2\sqrt{  \rho} 1_{A_1}  \sqrt{\P( \{ \varepsilon_2 \geqslant \frac{1}{16 \kappa \rho} | \mathcal{G}_1)}  \sqrt{ \E \big[ \varepsilon_2 \big|  \mathcal{G}_1 \big] } \\
\nonumber& & 
+   1_{A_1^{\sf c}} \bigg(\varepsilon_1   +  \E \big[ \sqrt{\varepsilon_2} | \mathcal{G}_1 \big] \sqrt{  2 \rho  } e^{ \sqrt{3 + t^2}t/  2}   +\sqrt{\rho}  e^{  \sqrt{3 + t^2} t   }  ( \sqrt{3} + 2 t ) \sqrt{\varepsilon_1}   \bigg)  \\
\nonumber& \leqslant & 57
 \kappa \rho \varepsilon_1^2
 + 12 \times 1_{A_1} \E \big[  \varepsilon_2   \big|  \mathcal{G}_1 \big ]
+ 2  \sqrt{  \rho} 1_{A_1}  \sqrt{ 16 \kappa \rho \E (  \varepsilon_2  | \mathcal{G}_1) }   \sqrt{ \E \big[ \varepsilon_2 \big|  \mathcal{G}_1 \big] } 
\mbox{ using Markov's inequality}, \\
\nonumber& & 
+   1_{A_1^{\sf c}} \bigg(\varepsilon_1   +  \E \big[ \sqrt{\varepsilon_2} | \mathcal{G}_1 \big] \sqrt{  2 \rho  } e^{ \sqrt{3 + t^2}t/  2}   +\sqrt{\rho}  e^{  \sqrt{3 + t^2} t   }  ( \sqrt{3} + 2 t ) \sqrt{\varepsilon_1}   \bigg) 
\\
\nonumber& = & 57
 \kappa \rho \varepsilon_1^2
 +   1_{A_1} \E \big[  \varepsilon_2   \big|  \mathcal{G}_1 \big ]
 ( 12 + 2 \sqrt{\rho} \sqrt{16 \kappa \rho} )
 \\
\label{eq:newA} & & 
+   1_{A_1^{\sf c}} \bigg(\varepsilon_1   +  \E \big[ \sqrt{\varepsilon_2} | \mathcal{G}_1 \big] \sqrt{  2 \rho  } e^{ \sqrt{3 + t^2}t/  2}   +\sqrt{\rho}  e^{  \sqrt{3 + t^2} t   }  ( \sqrt{3} + 2 t ) \sqrt{\varepsilon_1}   \bigg) .
\EEA

We now need to control $\varepsilon_2$, i.e., the error made by the LMS algorithm started from $\theta_1$.

\paragraph{LMS on the second-order Taylor approximation.}
We consider the quadratic approximation around  $\theta_1 \in \H$, and write is as an expectation, i.e.,
\BEAS g(\theta) & =  &  f(\theta_1) + \langle f'(\theta_1), \theta- \theta_1 \rangle
+ \frac{1}{2} \langle \theta - \theta_1 , f''(\theta_1) ( \theta - \theta_1) \rangle
\\
& = &  f(\theta_1) + \E \big[
\big\langle \ell'(y, \langle x, \theta_1\rangle ) x, \theta- \theta_1 \big\rangle
\big]
+ \frac{1}{2} \E \big[ \big\langle
\theta - \theta_1 , \ell''(y, \langle x,  \theta_1 \rangle ) x \otimes x ( \theta - \theta_1)
\big\rangle \big]
\\
& = &  f(\theta_1) +
\big\langle  \E \big[ \ell'(y, \langle x, \theta_1 \rangle ) x \big], \theta- \theta_1 \big\rangle
+ \frac{1}{2}  \big\langle
\theta - \theta_1 , \E \big[ \ell''(y, \langle x,  \theta_1 \rangle ) x \otimes x \big]( \theta - \theta_1)
\big\rangle .
\EEAS
We consider $\tilde{x}_n = \sqrt{ \ell''(y_n, \langle x_n,  \theta_1 \rangle )}  x_n$ and $\tilde{z}_n =
- \ell'(y_n, \langle x_n, \theta_1 \rangle ) x_n $, so that
\BEAS
 g(\theta) & =  & f(\theta_1) + \E \bigg[  \frac{1}{2} \langle \theta - \theta_1 , \tilde{x}_n \rangle^2  -  \langle\tilde{z}_n , \theta - \theta_1 \rangle \bigg].
 \EEAS
 We denote by $\theta_2 = \theta_1 - f''(\theta_1)^{-1} f'(\theta_1)$ the output of the Newton step, i.e., the global minimizer of~$g$, and $\tilde{\xi}_n = \tilde{z}_n - \langle \theta_2 - \theta_1, \tilde{x}_n \rangle \tilde{x}_n$ the residual.

 We have $\E \tilde{\xi}_n = 0$,  $\E \big[ \tilde{x}_n \otimes \tilde{x}_n \big]  = f''(\theta_1)$,
 and, for any $ z \in \H$:
 \BEAS
\big(  \E \big[ \langle z, \tilde{\xi}_n \rangle \big]^2 \big)^{1/2}
 & = &  \bigg( \E  \big[ \big(  \ell''(y_n, \langle x_n,  \theta_1 \rangle )  \langle \theta_2 - \theta_1, x_n \rangle + \ell'(y_n, \langle x_n,  \theta_1 \rangle \big) 
\langle z, x_n \rangle  
 \big]^2 \bigg)^{1/2}
 \\
  & \leqslant &  \bigg( \E  \big[ \big(  \ell''(y_n, \langle x_n,  \theta_1 \rangle )  \langle \theta_2 - \theta_1, x_n \rangle  
  \langle z , x_n \rangle  
 \big]^2 \bigg)^{1/2}
 + \bigg( \E  \big[   \ell'(y_n, \langle x_n,  \theta_1 \rangle  
\langle z, x_n \rangle  
 \big]^2 \bigg)^{1/2}
 \\
 & & \mbox{using the triangle inequality,}\\
& \leqslant &  \sqrt{\kappa} \sqrt{ \langle z, f''(\theta_1) z \rangle 
\langle \theta_2 - \theta_1 , f''(\theta_1) (\theta_2 - \theta_1)  \rangle }  + \bigg( \E  \big[   
\langle z, x_n \rangle  
 \big]^2 \bigg)^{1/2} \\
& \leqslant &  \sqrt{\kappa} \sqrt{ \langle z, f''(\theta_1) z \rangle 
\langle \theta_2 - \theta_1 , f''(\theta_1) (\theta_2 - \theta_1)  \rangle }  + 
\sqrt{\rho} \sqrt{ \langle z, f''(\theta_\ast) z \rangle } \\
& \leqslant & 
\sqrt{ \langle z, f''(\theta_1) z \rangle }
\bigg[
 \sqrt{\kappa} \| f''(\theta_1)^{-1/2} f'(\theta_1) \| + \sqrt{\rho} e^{\sqrt{ \kappa \rho} d_1/2}
\bigg]
 ,
  \EEAS
 where we denote  $d_{1}^2 = \langle \theta_1 - \theta_\ast, H( \theta_1 - \theta_\ast ) \rangle$, and we have used assumption \textbf{(B4)} ,   $| \ell' | \leqslant 1$
 and Prop.~\ref{prop:hessians} relating $H$ and $f''(\theta_1)$. This leads to
$$
\E \big[ \tilde{\xi}_n \otimes \tilde{\xi}_n \big] \preccurlyeq 
\bigg[
 \sqrt{\kappa} \| f''(\theta_1)^{-1/2} f'(\theta_1) \| + \sqrt{\rho} e^{\sqrt{ \kappa \rho} d_1/2}
\bigg]^2  f''(\theta_1) .
$$
Thus, we have:
\BIT
\item[--]
$\E \big[ \tilde{\xi}_n \otimes \tilde{\xi}_n \big] \preccurlyeq 
\sigma^2  f''(\theta_1) $ with $\sigma = 
 \sqrt{\kappa} \| f''(\theta_1)^{-1/2} f'(\theta_1) \| + \sqrt{\rho} e^{\sqrt{ \kappa \rho} d_1/2}$.
\item[--] $\| x_n\|^2 \leqslant R^2 / 4 $   almost surely.
\EIT

We may thus apply the previous results, i.e., Theorem~\ref{theo:lms}, to obtain with the LMS algorithm a $\theta_3 \in \H$ such that, with $\gamma = \frac{1}{R^2}$:
\BEAS
\E \big[ \varepsilon_2 | \mathcal{G}_1 \big]
& \leqslant &  \frac{2}{n} 
\bigg[\sqrt{d} 
\sqrt{\kappa} \| f''(\theta_1)^{-1/2} f'(\theta_1) \| + \sqrt{d}  \sqrt{\rho} e^{\sqrt{ \kappa \rho} d_1/2}
+ \frac{R}{2} \| f''(\theta_1)^{-1} f'(\theta_1) \| 
\bigg]^2
\\
& \leqslant &  \frac{2}{n} 
\bigg[\sqrt{d} 
\sqrt{\kappa} \| H^{-1/2} f'(\theta_1) \| e^{\sqrt{ \kappa \rho} d_1/2}
 + \sqrt{d}  \sqrt{\rho} e^{\sqrt{ \kappa \rho} d_1/2}
+ \frac{R}{2} e^{\sqrt{ \kappa \rho} d_1}
 \| H^{-1} f'(\theta_1) \| 
\bigg]^2 \\
& & \mbox{ using Prop.~\ref{prop:hessians} }, \\
& \leqslant &  \frac{2}{n} 
\bigg[\sqrt{d} 
\sqrt{\kappa} ( \sqrt{3} + 2 \sqrt{\kappa \rho \varepsilon_1}) \sqrt{\varepsilon_1}
  e^{\sqrt{ \kappa \rho} d_1/2}
 + \sqrt{d}  \sqrt{\rho} e^{\sqrt{ \kappa \rho} d_1/2}
+ \frac{R}{2} e^{\sqrt{ \kappa \rho} d_1}
\frac{ e^{ \sqrt{ \kappa \rho} d_1}
- 1 }{ \sqrt{ \kappa \rho} d_1} \| \theta_1 - \theta_\ast\|
\bigg]^2 \\
& & \mbox{ using Prop.~\ref{prop:unweighted} and \eq{DE} from \mysec{proofprop1}}.
\EEAS
Thus, $\E \big[ \varepsilon_2 | \mathcal{G}_1 \big] \leqslant
\frac{2}{n} \bigg[
R \| \theta_1 - \theta_\ast\|  \triangle_2(t)  + \triangle_3(t) \sqrt{ d\rho}
\bigg]^2 $,
with  increasing functions
$$\triangle_2(t)= \frac{1}{2} e^{ \sqrt{3+t^2} t} \frac{ e^{\sqrt{3+t^2} t}-1}{\sqrt{3+t^2} t} ,$$
$$\triangle_3(t)=  \big[ ( \sqrt{3} + 2 t) t + 1 \big] e^{ \sqrt{3+t^2} t / 2},$$
which are such that $\triangle_2(t) \leqslant 0.6$   and $\triangle_3(t) \leqslant 1.2$ if $t \leqslant 1 / 16$.


We then get from \eq{newA}:
\BEA
\nonumber
 & & \E \big[ f(\theta_3) - f(\theta_\ast)  \big| \mathcal{G}_1 \big ] \\
\nonumber& \leqslant &  
57
 \kappa \rho \varepsilon_1^2
 +     \frac{2}{n}     ( 12 + 2 \sqrt{\rho} \sqrt{16 \kappa \rho} )
 \big[
 0.6 R \| \theta_1 - \theta_\ast\| + \sqrt{ d \rho} 1.2
 \big]^2
 \\
\nonumber& & 
+   1_{A_1^{\sf c}} \bigg(\varepsilon_1   +  \sqrt{ \frac{2}{n}}\big( 
R\| \theta_1 - \theta_\ast\|  \triangle_2(t)  + \triangle_3(t)\sqrt{ d \rho}
\big)  \sqrt{  2 \rho  } e^{ \sqrt{3 + t^2}t/  2}   +\sqrt{\rho}  e^{  \sqrt{3 + t^2} t   }  ( \sqrt{3} + 2 t ) \sqrt{\varepsilon_1}   \bigg) 
\\
\nonumber& \leqslant &  
57
 \kappa \rho \varepsilon_1^2
 +     \frac{1}{n}     ( 12 + 8 \sqrt{\rho} \sqrt{  \kappa \rho} )
 \big( 6 d\rho + \frac{3}{2} R^2\| \theta_1 - \theta_\ast\|^2 \big)
 \\
\nonumber& &
+   1_{A_1^{\sf c}} \bigg(  \sqrt{\varepsilon_1} \big[ \frac{\sqrt{t}}{ \sqrt{ \kappa \rho}}  
 +\sqrt{\rho}  e^{  \sqrt{3 + t^2} t   }  ( \sqrt{3} + 2 t ) 
\big] +  \sqrt{ \frac{2}{n}}\big( 
R\| \theta_1 - \theta_\ast\|  \triangle_2(t)  + \triangle_3(t)\sqrt{ d \rho}
\big)  \sqrt{  2 \rho  } e^{ \sqrt{3 + t^2}t/  2}     \bigg) \\
\nonumber&  =  &  
57
 \kappa \rho \varepsilon_1^2
 +     \frac{12}{n}     ( 3 + 2 \sqrt{\rho} \sqrt{  \kappa \rho} )
 \big( 2  d\rho + \frac{1}{2} R^2\| \theta_1 - \theta_\ast\|^2 \big)
 \\
\nonumber& &
+   1_{A_1^{\sf c}} \bigg(  \sqrt{\rho \varepsilon_1} \big[ \frac{\sqrt{t}}{ \rho \sqrt{ \kappa }}  
 +   e^{  \sqrt{3 + t^2} t   }  ( \sqrt{3} + 2 t ) 
\big] +  \sqrt{ \frac{4 \rho }{n}} 
 \triangle_2(t) e^{ \sqrt{3 + t^2}t/  2}  R\| \theta_1 - \theta_\ast\| 
 +  \sqrt{ \frac{4 \rho }{n}} 
 \triangle_3(t) e^{ \sqrt{3 + t^2}t/  2}   \sqrt{ \rho d}
    \bigg) 
\\
\nonumber & \leqslant &  
57
 \kappa \rho \varepsilon_1^2
 +     \frac{12}{n}     ( 3 + 2 \sqrt{\rho} \sqrt{  \kappa \rho} )
 \big( 2  d\rho + \frac{1}{2} R^2\| \theta_1 - \theta_\ast\|^2 \big)
 \\
 \label{eq:DD} & &
+   1_{A_1^{\sf c}} \bigg(  \sqrt{\rho \varepsilon_1}  \triangle_4(t) +  \sqrt{ \frac{  \rho }{n}} 
  R\| \theta_1 - \theta_\ast\|  \triangle_5(t)
 +   \sqrt{ \frac{  \rho }{n}} 
  \sqrt{\rho d} \triangle_6(t) \bigg),
\EEA
with (using $\rho \geqslant 4$):
\BEAS
 \triangle_4(t) & = &   \frac{\sqrt{t}}{  4}  
 +   e^{  \sqrt{3 + t^2} t   }  ( \sqrt{3} + 2 t )  \leqslant 5 \exp( 2 t^2 )
\\
  \triangle_5(t) & = &  2  e^{  \sqrt{3 + t^2} t /2   }   \Delta_2(t) \leqslant  4 \exp( 3 t^2 )\\
   \triangle_6(t) & = &  2  e^{  \sqrt{3 + t^2} t /2   }   \Delta_3(t) \leqslant  6 \exp( 3 t^2 ).
\EEAS
The last inequalities may be checked graphically.

By taking expectations and using $\E | XYZ| \leqslant  (\E |X|^2)^{1/2} (\E |X|^4)^{1/4}  (\E |X|^4)^{1/4} $, this leads to, from \eq{DD}:
 \BEA
 & & \nonumber \E \big[ f(\theta_3) - f(\theta_\ast)    \big ] \\
\nonumber & \leqslant &  
57
 \kappa \rho \E \big[ \varepsilon_1^2 \big]
 +     \frac{12}{n}     ( 3 + 2 \sqrt{\rho} \sqrt{  \kappa \rho} )
 \big( 2 d \rho+ \frac{1}{2} \E \big[   R^2\| \theta_1 - \theta_\ast\|^2 \big] \big)
 \\
\nonumber & & 
+   \sqrt{\rho} \sqrt{ \P ( A_1^{\sf c}) } \bigg(  \big( \E \big[ \varepsilon_1^2 \big] \big)^{1/4} \big( \E \big[  \triangle_4(t) \big]^4 \big)^{1/4} \\
\label{eq:close}
& &  +  \sqrt{ \frac{  1 }{n}} 
  \big( \E \big[ R^4 \| \theta_1 - \theta_\ast\|^4 \big] \big)^{1/4}  \big( \E \big[  \triangle_5(t) \big]^4 \big)^{1/4}
 +   \sqrt{ \frac{  1 }{n}} 
  \sqrt{\rho d}  \big( \E \big[  \triangle_6(t) \big]^4 \big)^{1/4}  \bigg).
  \EEA
  We now need to use bounds on the behavior of the first $n$ steps of regular  averaged stochastic gradient descent.
  
\paragraph{Fine results on  averaged stochastic gradient descent.}
In order to get error bounds on $\theta_1$, we run $n$ steps of averaged stochastic gradient descent with constant-step size $\gamma = 1 / ( 2 R^2 \sqrt{n} )$.
We need the following bounds from~\cite[Appendix E and Prop.~1]{bach2013adaptivity}:
\BEAS
\E [ (  f(\theta_1) - f(\theta_\ast) )^2 ] & \leqslant &  \frac{1}{n}  (  R^2 \| \theta_0 - \theta_\ast\|^2 + \frac{3}{4})^2\\
\E [ (  f(\theta_1) - f(\theta_\ast) )^3 ] & \leqslant &  \frac{1}{n^{3/2}}  (  R^2 \| \theta_0 - \theta_\ast\|^2 + \frac{3}{2})^3\\
\E \big[ R^4 \| \theta_1 - \theta_\ast\|^4  \big] & \leqslant &  (  R^2 \| \theta_0 - \theta_\ast\|^2 + \frac{3}{4})^2\\
\E \big[ R^2 \| \theta_1 - \theta_\ast\|^2  \big] & \leqslant &     R^2 \| \theta_0 - \theta_\ast\|^2 +  \frac{1}{4}\\
\P \bigg[  f(\theta_1) - f(\theta_\ast) \geqslant \frac{1}{16^2} (\kappa \rho)^{-1}  \bigg] & \leqslant & 
16^6 (\kappa \rho)^3 \E \big[  f(\theta_1) - f(\theta_\ast) \big]^3 \mbox{ using Markov's inequality},\\
& \leqslant & 
16^6 (\kappa \rho)^3 \frac{1}{n^{3/2}}  (  R^2 \| \theta_0 - \theta_\ast\|^2 + \frac{3}{2})^3.
\EEAS

However, we need a finer control of the deviations in order to bound quantities of the form $e^{ \alpha\varepsilon_1}$. In \mysec{control}, extending results 
from~\cite{bach2013adaptivity}, we show in Prop.~\ref{prop:exp} that if $ \frac{ \alpha   ( 10 + 2R^2 \| \theta_0 - \theta_\ast\|^2)  }{\sqrt{n}} \leqslant \frac{1}{2e }$, then $E e^{ \alpha ( f(\theta_1) - f(\theta_\ast) )} \leqslant 1$.

\paragraph{Putting things together.}
From \eq{close}, we then get, if $ \frac{ 6   ( 10 + 2R^2 \| \theta_0 - \theta_\ast\|^2)  }{\sqrt{n}} \leqslant \frac{1}{2e }$:
  \BEAS
& & \nonumber \E \big[ f(\theta_3) - f(\theta_\ast)    \big ] \\
  & \leqslant &  
57
 \kappa \rho  \frac{1}{n}  (  R^2 \| \theta_0 - \theta_\ast\|^2 + \frac{3}{4})^2  +     \frac{12}{n}     ( 3  + 2 \sqrt{\rho} \sqrt{  \kappa \rho} )
 \big(  2 d \rho + \frac{1}{4} +  \frac{1}{2}   R^2\| \theta_0 - \theta_\ast\|^2  \big)
 \\
& & 
+   \sqrt{\rho} \sqrt{16^6 (\kappa \rho)^3 \frac{1}{n^{3/2}}  (  R^2 \| \theta_0 - \theta_\ast\|^2 + \frac{3}{2})^3
  }  
\\
& & \times   \bigg(  5    \big( 
 \frac{1}{n}  (  R^2 \| \theta_0 - \theta_\ast\|^2 + \frac{3}{4})^2 \big)^{1/4}  +   4 \sqrt{ \frac{  1 }{n}} 
  \big(   (  R^2 \| \theta_0 - \theta_\ast\|^2 + \frac{3}{4})^2\big)^{1/4}   
 +   6  \sqrt{ \frac{  1 }{n}} 
  \sqrt{\rho d}     \bigg) 
  \EEAS
  Using the notation $D =  (  R^2 \| \theta_0 - \theta_\ast\|^2 + \frac{3}{2})$, we obtain:
 \BEAS
& & \nonumber \E \big[ f(\theta_3) - f(\theta_\ast)    \big ] \\
& \leqslant &  
57
 \kappa \rho  \frac{1}{n}  D^2  +     \frac{12}{n}     (3+ 2 \sqrt{\rho} \sqrt{  \kappa \rho} )
 \big( 2 d \rho+  \frac{D}{2}    \big)
 \\
& & 
+   \sqrt{\rho} \sqrt{16^6 (\kappa \rho)^3 \frac{1}{n^{3/2}}  D^3
  }  
 \times   \bigg(  5    \big( 
 \frac{1}{n}  D^2 \big)^{1/4}  +   4 \sqrt{ \frac{  1 }{n}} 
  \big(  D^2\big)^{1/4}   
 +   6  \sqrt{ \frac{  1 }{n}} 
  \sqrt{\rho d}     \bigg) \\
 & = &  
57
 \kappa \rho  \frac{1}{n}  D^2  +     \frac{12}{n}     (3+ 2 \sqrt{\rho} \sqrt{  \kappa \rho} )
 \big( 2 d \rho +  \frac{D}{2}    \big)
 \\
& & 
+   \sqrt{\rho} 16^3   (\kappa \rho)^{3/2} \frac{1}{n^{3/4}}  D^{3/2}
 \times   \bigg(  5    
 \frac{1}{n^{1/4} }  D^{1/2}  +   4  { \frac{  1 }{n^{1/2}}} 
   D^{1/2}  
 +   6   { \frac{  1 }{\sqrt{n}}} 
  \sqrt{\rho d}     \bigg) \\
 & \leqslant &  
 \frac{\kappa^{3/2} \rho^2}{n}
 \bigg[ 
\frac{57}{4}
  D^2  +     12    (\frac{3}{16} + \frac{2}{4} )
 \big( 2 d \rho +  \frac{D}{2}    \big)
+    16^3  D^{3/2}
 \times   \bigg(  5    
  D^{1/2}  +   4  { \frac{  1 }{n^{1/4}}} 
   D^{1/2}  
 +   6   { \frac{  1 }{n^{1/4}}} 
  \sqrt{\rho d}     \bigg) \bigg]\\
  & & \hspace*{4cm} \mbox{ using } \rho \geqslant 4 \mbox{ and } \kappa \geqslant 1, \\
   & \leqslant &  
 \frac{\kappa^{3/2} \rho^2 D^2 }{n}
 \bigg[ 
\frac{57}{4}
  1 +     12    (\frac{3}{16} + \frac{2}{4} )
 \big( 2 d \rho\frac{4}{9} +  \frac{1}{2} \frac{2}{3}    \big)
+    16^3  
 \times   \bigg(  5    
   +   4   
 +   6   { \frac{  1 }{n^{1/4}}}  \frac{\sqrt{2}}{\sqrt{3}}
  \sqrt{\rho d}     \bigg) \bigg] \mbox{ using } D \geqslant \frac{3}{2},\\
& \leqslant &  
 \frac{\kappa^{3/2} \rho^2 D^2    }{n}
 \bigg[   36881  +  20067 \frac{\sqrt{\rho d}}{n^{1/4}   } + 17 d \rho \bigg]
 \leqslant  \frac{\kappa^{3/2} \rho^3 d   }{n}  56965 (    R^2 \| \theta_0 - \theta_\ast\|^2 + \frac{3}{2})^2  \\
 & \leqslant & \frac{\kappa^{3/2} \rho^3 d   }{n}    (    16 R  \| \theta_0 - \theta_\ast\| + 19)^4  .
\EEAS

The condition $ \frac{ 6   ( 10 + 2R^2 \| \theta_0 - \theta_\ast\|^2)  }{\sqrt{n}} \leqslant \frac{1}{2e }$ is implied
by $n \geqslant ( 19  +  9 R \| \theta_0 - \theta_\ast\| )^4$.

\section{Higher-order bounds for stochastic gradient descent}
\label{sec:control}

In this section, we provide high-order bounds for averaged stochastic gradient for logistic regression. The first proposition gives a finer result than~\cite{bach2013adaptivity}, with a simpler proof, while the second proposition is new.

\begin{proposition}
Assume  \textbf{(B1-4)}. Consider the  stochastic gradient recursion
$ \theta_n = \theta_{n-1} - \gamma \ell'(y_n, \langle \theta_{n-1},x_n \rangle ) x_n$ and its averaged version $\bar{\theta}_{n-1}$. 
We have, for all real $p\geqslant 1$,
\BEQ
\big\|
f(\bar{\theta}_{n-1}) - f(\theta_\ast) 
\big\|_p \leqslant \frac{17 \gamma R^2}{2} ( \sqrt{p  } + \frac{p}{\sqrt{n}} )^2
+ \frac{1}{\gamma n}  \| \theta_{0 }- \theta_\ast \|^2
\EEQ
\BEQ
\big\|
  \| \theta_n - \theta_\ast  \|^2
\big\|_p \leqslant {17 \gamma^2 R^2 n }  ( \sqrt{p  } + \frac{p}{\sqrt{n}} )^2
+ 2  \| \theta_{0 }- \theta_\ast \|^2.
\EEQ
\end{proposition}
\begin{proof}
Following \cite{bach2013adaptivity}, we   have the recursion:
\BEAS
2 \gamma\big[  f(\theta_{n-1}) - f(\theta_\ast)  \big] +   \| \theta_{n }- \theta_\ast \|^2
& \leqslant &   \| \theta_{n-1 }- \theta_\ast \|^2
+ \gamma^2   R^2 + M_n,
\EEAS
     with 
      $$
      M_n =  -  2 \gamma  \langle
 \theta_{n-1} - \theta_\ast ,  f_n'(\theta_{n-1}) -  f'(\theta_{n-1}) \rangle.$$

This leads to
$$ 2\gamma n  f \bigg( \frac{1}{n} \sum_{k=1}^{n} \theta_{k-1} \bigg) - 2\gamma n f(\theta^\ast)  +\| \theta_n - \theta_\ast \|^{2}  \leqslant A_n,$$
with
$
A_n =  \| \theta_{0 }- \theta_\ast \|^2
+ n \gamma^2   R^2 + \sum_{k=1}^n M_k.
$ Note that $\E(M_k | \F_{k-1})=0$ and 
$
 |M_k|   \leqslant    4 \gamma R \| \theta_{k-1} - \theta_\ast\| \leqslant 4\gamma R A_{k-1}^{1/2}
$ almost surely. We may now use BRP's inequality in \eq{brp} to get:
\BEAS
\bigg\| \sup_{k \in \{0,\dots, n\} } A_k \bigg\|_p
& \leqslant & \| \theta_{0 }- \theta_\ast \|^2
+ n \gamma^2   R^2 
+ \sqrt{p} \bigg\|
16 \gamma^2 R^2 \sum_{k=1}^n \| \theta_{k-1} - \theta_\ast \|^2
\bigg\|_{p/2}^{1/2} \\
&& 
+ p \bigg\| \sup_{k \in \{1,\dots, n\} } 4 \gamma R \| \theta_{k-1} - \theta_\ast\|  \bigg\|_p
\\
& \leqslant & \| \theta_{0 }- \theta_\ast \|^2
+ n \gamma^2   R^2 
+ \sqrt{p} 4 \gamma R  \sqrt{n}  \bigg\|
  \sup_{k \in \{0,\dots, n-1\} } A_k
\bigg\|_{p/2}^{1/2} \\
& &
+ p 4 \gamma R \bigg\| \sup_{k \in \{0,\dots, n-1\} }  A_{k}^{1/2}  \bigg\|_p
\\
& \leqslant & \| \theta_{0 }- \theta_\ast \|^2
+ n \gamma^2   R^2 
+ 4 \gamma R     \bigg\|
  \sup_{k \in \{0,\dots, n-1\} } A_k
\bigg\|_{p/2}^{1/2} \big( \sqrt{p n } + p \big)
.\EEAS
Thus if $B= \bigg\| \sup_{k \in \{0,\dots, n\} } A_k \bigg\|_p$, we have
$$
B \leqslant  \| \theta_{0 }- \theta_\ast \|^2
+ n \gamma^2   R^2 
+   4 \gamma R   B^{1/2}  \big( \sqrt{p n } + p \big) .$$
By solving this quadratic inequality, we get:
$$
\big( B^{1/2}  - 2 \gamma R ( \sqrt{p n } + p) \big)^2
\leqslant  \| \theta_{0 }- \theta_\ast \|^2
+ n \gamma^2   R^2 
+   4 \gamma^2 R^2  \big( \sqrt{p n } + p \big)^2 $$
$$
 B^{1/2}   \leqslant 2 \gamma R ( \sqrt{p n } + p)  
+ \sqrt{  \| \theta_{0 }- \theta_\ast \|^2
+ n \gamma^2   R^2 
+   4 \gamma^2 R^2  \big( \sqrt{p n } + p \big)^2} $$
\BEAS
 B    & \leqslant &  8 \gamma^2 R^2 ( \sqrt{p n } + p)^2
+ 2 \| \theta_{0 }- \theta_\ast \|^2
+ 2 n \gamma^2   R^2 
+   8 \gamma^2 R^2  \big( \sqrt{p n } + p \big)^2
\\ 
&\leqslant &  16 \gamma^2 R^2 ( \sqrt{p n } + p)^2
+ 2 \| \theta_{0 }- \theta_\ast \|^2
+ 2 n \gamma^2   R^2  \\
&\leqslant &  17 \gamma^2 R^2 ( \sqrt{p n } + p)^2
+ 2 \| \theta_{0 }- \theta_\ast \|^2.
\EEAS
The previous statement leads to the desired result if $p \geqslant 2$. For $p \in [1,2]$, we may bound it by the value at $p=2$, and a direct calculation shows that the bound is still correct.
\end{proof}

\begin{proposition}
\label{prop:exp}
Assume  \textbf{(B1-4)}. Consider the  stochastic gradient recursion
$ \theta_n = \theta_{n-1} - \gamma \ell'(y_n, \langle \theta_{n-1},x_n \rangle ) x_n$ and its averaged version $\bar{\theta}_{n-1}$. 
If  $\displaystyle \frac{ \alpha e   ( 10 + 2R^2 \| \theta_0 - \theta_\ast\|^2)  }{\sqrt{n}} \leqslant \frac{1}{2}$, then
$ \E  \exp \big(  \alpha \big[ f(\bar{\theta}_{n-1}) - f(\theta_\ast)  \big] \big) 
\leqslant 1$.
\end{proposition}
\begin{proof}
Using that almost surely, $\| \bar{\theta}_{n-1} - \theta_\ast\| \leqslant \| \theta_0 - \theta_\ast\| + n \gamma R$ we obtain that  almost surely
$
f(\bar{\theta}_{n-1} ) - f(\theta_\ast)
\leqslant  R\| \theta_0 - \theta_\ast\| + n \gamma R^2
$.

Moreover, from the previous proposition,  we have for $p \leqslant \frac{n}{4}$,
$$\big\|
f(\bar{\theta}_{n-1}) - f(\theta_\ast) 
\big\|_p \leqslant \frac{17 \gamma R^2}{2}  \frac{9}{4}p
+ \frac{1}{\gamma n}  \| \theta_{0 }- \theta_\ast \|^2.$$
For $\gamma = \frac{1}{ 2R^2 \sqrt{n}}$, we get:
$$\big\|
f(\bar{\theta}_{n-1}) - f(\theta_\ast) 
\big\|_p \leqslant \frac{10 p}{\sqrt{n}}
+ \frac{2R^2}{\sqrt{n}}  \| \theta_{0 }- \theta_\ast \|^2,$$
and
$$
f(\bar{\theta}_{n-1} ) - f(\theta_\ast)
\leqslant  R\| \theta_0 - \theta_\ast\| + \frac{\sqrt{n}}{2} \mbox{ almost surely}.
$$
This leads to the bound valid for all $p$:
$$\big\|
f(\bar{\theta}_{n-1}) - f(\theta_\ast) 
\big\|_p \leqslant \frac{ p}{\sqrt{n}}
( 10 +  2R^2  \| \theta_{0 }- \theta_\ast \|^2).$$
 We then get
 \BEAS
\E  \exp \big(  \alpha \big[ f(\bar{\theta}_{n-1}) - f(\theta_\ast)  \big] \big) 
& = & 
\sum_{k=0}^\infty \frac{ \alpha^p}{p!} \E [ | f(\bar{\theta}_{n-1}) - f(\theta_\ast) 
 |^p] \\
& \leqslant & 
\sum_{k=0}^\infty \frac{ \alpha^p}{p!}  \frac{p^p}{n^{p/2}} \bigg(   
 10 +  2R^2  \| \theta_{0 }- \theta_\ast \|^2   \bigg)^p\\
& \leqslant & 
\sum_{k=0}^\infty \frac{ \alpha^p}{2 (p/e)^p}  \frac{p^p}{n^{p/2}} \bigg(   
 10 +  2R^2  \| \theta_{0 }- \theta_\ast \|^2   \bigg)^p \mbox{ using Stirling's formula},\\
& \leqslant & 
\frac{1}{2} \sum_{k=0}^\infty \frac{(e\alpha)^p}{n^{p/2}} \bigg(   
 10 +  2R^2  \| \theta_{0 }- \theta_\ast \|^2   \bigg)^p\\
& \leqslant &  \frac{1}{2} \frac{1}{1 - 1/2} = 1 \mbox{ if } \frac{ \alpha e   ( 10 + 2R^2 \| \theta_0 - \theta_\ast\|^2)  }{\sqrt{n}} \leqslant \frac{1}{2}.
\EEAS
\end{proof}

 \section{Properties of self-concordance functions}
In this section, we review various properties of self-concordant functions, that will prove useful in proving Theorem~\ref{theo:log}. All these properties rely on bounding the third-order derivatives by second-order derivatives. More precisely, from assumptions \textbf{(B3-4)}, we have for any $\theta, \delta, \eta \in \H$,
where $f^{(r)}[\delta_1,\dots,\delta_k]$ denotes the $k$-th order differential of $f$:
\BEAS
f'''(\theta)[\delta,\delta,\eta]
& = & \E \big[ \ell'''(y_n, \langle \theta,x_n \rangle)  \langle \delta, x_n \rangle^2 
\langle \eta, x_n \rangle \big] \\
| f'''(\theta)[\delta,\delta,\eta] |
& \leqslant & \E \big[ \ell''(y_n, \langle \theta,x_n \rangle) \langle \delta, x_n \rangle^2 
| \langle \eta, x_n \rangle | \big] \\
& \leqslant & \sqrt{ \E \big[ \ell''(y_n, \langle \theta,x_n \rangle)^2  \langle \delta, x_n \rangle ^4\big] }  \sqrt{ \E \big[    \langle \eta, x_n \rangle^2\big] } 
\mbox{ using Cauchy-Schwarz},
\\
& \leqslant & \sqrt{\kappa \rho }  {f''(\theta)[\delta,\delta] }\sqrt{ \langle \eta, H \eta \rangle } \mbox{ using the two assumptions}.   \EEAS

\subsection{Global Taylor expansions}
\label{sec:selfc1}

In this section, we derive global non-asymptotic Taylor expansions for self-concordant functions, which show  that they  behave similarly to like quadratic functions.

The following proposition shows that having a small excess risk $f(\theta) -f (\theta_\ast) $ implies that the weighted distance to optimum 
$ \langle \theta - \theta_\ast, H ( \theta - \theta_\ast) \rangle $ is small. Note that for quadratic functions, these two quantities are equal and that throughout this section, we always consider norms weighted by the matrix $H$ (Hessian at optimum).

\begin{proposition}[Bounding weighted distance to optimum from function values]
\label{prop:wd}
Assume \textbf{(B3-4)}. Then, for any $\theta \in \H$:
\BEQ
\label{eq:value}
  \langle \theta - \theta_\ast, H ( \theta - \theta_\ast) \rangle  \leqslant 3 \big[ f(\theta) -f (\theta_\ast) \big] + \kappa \rho \big[ f(\theta) -f (\theta_\ast) \big]^2.
\EEQ
\end{proposition}
\begin{proof}
Let $\varphi: t \mapsto f\big[\theta_\ast + t (\theta - \theta_\ast) \big]$.
Denoting $d = \sqrt{ \langle \theta - \theta_\ast, f''(\theta_\ast) ( \theta - \theta_\ast) \rangle }$, we have:
\BEAS
|\varphi'''(t)| & \leqslant &  \E \big[ \ell'''(y, \langle \theta_\ast + t (\theta - \theta_\ast) , x \rangle)|\langle \theta - \theta_\ast, x \rangle|^3\big]  \\
& \leqslant & \E \big[ \ell''(y, \langle \theta_\ast + t (\theta - \theta_\ast) , x \rangle) \langle \theta - \theta_\ast, x \rangle^2\big] \sqrt{ \kappa \rho }d
= \sqrt{\kappa \rho} d \varphi''(t),
\EEAS
from which we obtain $\varphi''(t) \geqslant \varphi''(0) e^{-\sqrt{\kappa \rho}dt}$. Following~\cite{bach2010self}, by integrating twice (and noting that $\varphi'(0)=0$ and $\varphi''(0) = d^2$), we get
\BEAS
f(\theta) = \varphi(1) & \geqslant & \varphi(0) + \varphi''(0) \frac{1}{S^2 d^2 }
\big( e^{-\sqrt{\kappa \rho}d} +\sqrt{\kappa \rho} d - 1\big)\\
& \geqslant & f(\theta_\ast)  +  \frac{1}{\kappa \rho  }
\big( e^{-\sqrt{\kappa \rho}d} +\sqrt{\kappa \rho} d - 1\big).
\EEAS
Thus $$    e^{-\sqrt{\kappa \rho}d} + \sqrt{\kappa \rho} d - 1 \leqslant \kappa \rho \big[ f(\theta) -f (\theta_\ast) \big].
$$
The function $\kappa: u \mapsto e^{-u}+u-1$ is an increasing bijection from $\rb_+$ to itself. Thus this implies $d \leqslant \frac{1}{\sqrt{\kappa \rho}} \kappa^{-1} \bigg(
\kappa \rho \big[ f(\theta) -f (\theta_\ast) \big] \bigg)$. We  show below that $\kappa^{-1}(v) \leqslant \sqrt{3v + v^2}$, leading to the desired result.

The identity $\kappa^{-1}(v) \leqslant \sqrt{3v + v^2}$ is equivalent to
$e^{-u} + u - 1 \geqslant \sqrt{ u^2 + \alpha^2} - \alpha$, for $\alpha = \frac{3}{2}$. It then suffices to show that 
$1-e^{-u} \geqslant  \frac{u}{ \sqrt{ u^2 + \alpha^2} }$. This can be shown by proving the monotonicity of $u \mapsto 
e^{-u} + u - 1 - \sqrt{ u^2 + \alpha^2} + \alpha$, and we leave this exercise to the reader.
\end{proof}

The next proposition shows that Hessians between two points which are close in weighted distance are close to each other, for the order between positive semi-definite matrices.

\begin{proposition}[Expansion of Hessians]
\label{prop:hessians}
Assume \textbf{(B3-4)}. Then, for any $\theta_1,\theta_2 \in \H$:
\BEQ
f''(\theta_1) e^{\sqrt{\kappa \rho} \sqrt{ \langle \theta_2 - \theta_1, H ( \theta_2 - \theta_1) \rangle }} \succcurlyeq
 f''(\theta_2) \succcurlyeq f''(\theta_1) e^{-\sqrt{\kappa \rho} \sqrt{ \langle \theta_2 - \theta_1, H ( \theta_2 - \theta_1) \rangle } }
,
\EEQ
\BEQ
\big \|
f''(\theta_1)^{-1/2} f''(\theta_2) f''(\theta_1)^{-1/2} - \idm
\big\|_{\rm op} \leqslant e^{\sqrt{\kappa \rho} \sqrt{ \langle \theta_2 - \theta_1, H ( \theta_2 - \theta_1) \rangle } }- 1.
\EEQ
\end{proposition}
\begin{proof}
Let $z \in \H$ and $\psi(t) = z^\top f''( \theta_1 + t (\theta_2- \theta_1) ) z $. We have:
\BEAS
|\psi'(t)| & = & |f'''( \theta_1 + t (\theta_2 - \theta_1) ) [z,z,\theta_2 - \theta_1] |\\
& \leqslant & f''( \theta_1 + t (\theta_2 - \theta_1) ) [z,z]  \sqrt{\kappa \rho}  d  = \psi(t) \sqrt{\kappa \rho} d,
\EEAS
with $d_{12} =  \sqrt{ \langle \theta_2 - \theta_1, H ( \theta_2 - \theta_1) \rangle }$.
Thus $ \psi(0) e^{\sqrt{\kappa \rho} d_{12} t} \geqslant \psi(t) \geqslant \psi(0) e^{-\sqrt{\kappa \rho} d_{12} t}$. This implies, for $t=1$, that
$$
 f''(\theta_1) e^{\sqrt{\kappa \rho} d_{12}} \succcurlyeq
 f''(\theta_2) \succcurlyeq f''(\theta_1) e^{-\sqrt{\kappa \rho} d_{12}}
,
$$
which implies the desired results.   $\| \cdot \|_{\rm op}$ denotes the operator norm (largest singular value).
\end{proof}

The following proposition gives an approximation result bounding the first order expansion of gradients by the first order expansion of function values.
\begin{proposition}[Expansion of gradients]
Assume \textbf{(B3-4)}. Then, for any $\theta_1,\theta_2 \in \H$ and $\Delta \in \H$:
\BEQ
\langle \Delta, f'(\theta_2) - f'(\theta_1) - f''(\theta_1) ( \theta_2 - \theta_1) \rangle
\leqslant \sqrt{\kappa \rho} \langle \Delta, H \Delta \rangle^{1/2}
\big[ f(\theta_2) - f(\theta_1) - \langle f'(\theta_1) ,\theta_2- \theta_1\rangle
\big].
\EEQ
\end{proposition}
\begin{proof}
Let $\varphi(t) = \langle \Delta, f'(\theta_1 + t(\theta_2 - \theta_1)) - f'(\theta_1) - t f''(\theta_1) ( \theta_2 - \theta_1) \rangle$. We have
$\varphi'(t) = \langle \Delta,  f''(\theta_1 + t(\theta_2 - \theta_1))(\theta_1 - \theta_2 \rangle
 - \langle \Delta,  f''(\theta_1  )(\theta_1 - \theta_2 ) \rangle
$ and
$|\varphi''(t) |=| f'''(\theta_1 + t(\theta_2 - \theta_1))[\theta_2 - \theta_1,\theta_2 - \theta_1,\Delta]| \leqslant  \sqrt{\kappa \rho}  \langle \Delta, H \Delta \rangle^{1/2}
\langle \theta_1 - \theta_2,  f''(\theta_1 + t(\theta_2 - \theta_1))(\theta_1 - \theta_2)\rangle  $. This leads to
$$
\langle \Delta, f'(\theta_2) - f'(\theta_1) - f''(\theta_1) ( \theta_2 - \theta_1) \rangle
\leqslant \sqrt{\kappa \rho}  \langle \Delta, H \Delta \rangle^{1/2} 
\big[ f(\theta_2) - f(\theta_1) - \langle f'(\theta_1) ,\theta_2 - \theta_1\rangle
\big] .
$$
Note that one may also use the bound
\BEAS
|\varphi'(t)| & \leqslant & \| \theta_1 - \theta_2 \| \langle \Delta, f''(\theta_1)^2 \Delta \rangle ^{1/2} \big[ e^{ t \sqrt{\kappa \rho}  \| H^{1/2} ( \theta_1  - \theta_2) \|}
-1 \big], 
\EEAS
leading to 
\BEA
\nonumber & & \langle \Delta, f'(\theta_2) - f'(\theta_1) - f''(\theta_1) ( \theta_2 - \theta_1) \rangle
\\
& \leqslant &
 \| \theta_1 - \theta_2 \| \langle \Delta, f''(\theta_1)^2 \Delta \rangle ^{1/2} 
 \frac{ e^{   \sqrt{\kappa \rho} \| H^{1/2} ( \theta_1  - \theta_2) \|}
 - 1 -   \sqrt{\kappa \rho} \| H^{1/2} ( \theta_1  - \theta_2) \|}{  \sqrt{\kappa \rho}  \| H^{1/2} ( \theta_1  - \theta_2) \|}.
\EEA
\end{proof}

The following proposition considers a global Taylor expansion of function values. Note that when $ {\kappa \rho}    { \langle \theta_2 - \theta_1, H ( \theta_2 - \theta_1) \rangle }$ tends to zero, we obtain \emph{exactly} the second-order Taylor expansion. For more details, see~\cite{bach2010self}. This is followed by a corrolary that upper bounds excess risk by distance to optimum (this is thus the other direction than~Prop.~\ref{prop:wd}).

\begin{proposition}[Expansion of function values]
\label{prop:functionvalues}
Assume \textbf{(B3-4)}. Then, for any $\theta_1,\theta_2 \in \H$ and $\Delta \in \H$:
\BEA
\nonumber
\!\!\!\!\!\!\!\!\!\! & & 
 f (\theta_2) - f (\theta_1) - \langle f'(\theta_1) , \theta_2 - \theta_1  \rangle \\
\!\!\!\!\!\!\!\!\!\!& \!\!\!\!\! \leqslant \!\!\!\!\! &   \langle \theta_2 - \theta_1, f''(\theta_1) (\theta_2 - \theta_1) \rangle \frac{e^{\sqrt{\kappa \rho}  \sqrt{ \langle \theta_2 - \theta_1, H ( \theta_2 - \theta_1) \rangle }} - 1 - 
\sqrt{\kappa \rho}  \sqrt{ \langle \theta_2 - \theta_1, H ( \theta_2 - \theta_1) \rangle }}{  {\kappa \rho}    { \langle \theta_2 - \theta_1, H ( \theta_2 - \theta_1) \rangle }} .
\label{eq:functionvalues}\EEA
\end{proposition}
\begin{proof}
Let $\varphi(t) = f \big[ \theta_1 + t(\theta_2 - \theta_1) \big] - f(\theta_1) - t  \langle f'(\theta_1) , \theta_2 - \theta_1  \rangle$. We have
$\varphi'(t) =  \langle f'\big[ \theta_1 + t(\theta_2 - \theta_1) \big] , \theta_2 - \theta_1  \rangle-    \langle f'(\theta_1) , \theta_2 - \theta_1  \rangle
$ and $\varphi''(t) = \langle \theta_2 - \theta_1,  f''\big[ \theta_1 + t(\theta_2 - \theta_1) \big] ( \theta_2 - \theta_1) \rangle$. Moreover,
$\varphi'''(t) \leqslant \sqrt{\kappa \rho}  \varphi''(t) \sqrt{ \langle \theta_2 - \theta_1, H ( \theta_2 - \theta_1) \rangle }$, leading to
$\varphi''(t) \leqslant e^{\sqrt{\kappa \rho}  t \sqrt{ \langle \theta_2 - \theta_1, H ( \theta_2 - \theta_1) \rangle }} \varphi''(0) $. Integrating twice between 0 and 1 leads to the desired result.
\end{proof}

\begin{corollary}[Excess risk]
\label{prop:value_from_opt}
Assume \textbf{(B3-4)}, and $\theta_1 \in \H$ and $\theta_2 = \theta_1 - f''(\theta_1)^{-1} f'(\theta_1)$. Then
\BEQ
f(\theta) - f(\theta^\ast) \leqslant \frac{ e^{\sqrt{\kappa \rho} d} - \sqrt{\kappa \rho} d - 1}{ {\kappa \rho} },
\EEQ
where $d = \sqrt{ \langle \theta - \theta_\ast, H( \theta - \theta_\ast) \rangle}$.\end{corollary}
\begin{proof}
Applying Prop.~\ref{prop:functionvalues} to $\theta_2 = \theta$ and $\theta_1 = \theta_\ast$, we get the desired result.
\end{proof}

The following proposition looks at a similar type of bounds than Prop.~\ref{prop:functionvalues}; it is weaker when $\theta_2$ and $\theta_1$ are close (it does not converge to the second-order Taylor expansion), but stronger for large values (it does not grow exponentially fast).

\begin{proposition}[Bounding function values with fewer assumptions]
\label{prop:values_few}
Assume \textbf{(B3-4)}, and $\theta1, \theta_2 \in \H$. Then
\BEQ
f(\theta_2) - f(\theta_1) \leqslant   \sqrt{\rho} \| H^{1/2} ( \theta_1- \theta_2) \|.
\EEQ
\end{proposition}
\begin{proof}
Let $\varphi(t) = f(\theta_1 + t( \theta_2- \theta_1) ) - f(\theta_1)$. We have
$|\varphi'(t)| = |\E \ell'(y_n, \langle x_n, \theta1 + t( \theta_2- \theta_1) \rangle )
\langle \theta_2- \theta_1t, x_n \rangle  | \leqslant \sqrt{ \rho } \| H^{1/2} ( \theta_1 - \theta_2) \|$. Integrating between 0 and 1 leads to the desired result.
\end{proof}

\subsection{Analysis of Newton step}
\label{sec:selfc2}
Self-concordance has been traditionally used in the analysis of Newton's method~(see~\cite{boyd,self}). In this section, we adapt classical results to our specific notion of self-concordance (see also~\cite{bach2010self}). A key quantity is the so-called ``Newton decrement'' at a certain point $\theta_1$, equal to
$\langle f'(\theta_1), f''(\theta_1)^{-1} f'(\theta_1) \rangle$, which governs the convergence behavior of Newton methods (this is the quantity which is originally shown to be quadratically convergent). In this paper, we consider a slightly different version where the Hessian is chosen to be the one at $\theta_\ast$, i.e., $\langle f'(\theta_1), H^{-1} f'(\theta_1) \rangle$.

The following proposition shows how a full Newton step improves the Newton decrement (by taking a square).

\begin{proposition}[Effect of Newton step on Newton decrement]
\label{prop:effect_newton}
Assume \textbf{(B3-4)}, and $\theta_1 \in \H$ and $\theta_2 = \theta_1 - f''(\theta_1)^{-1} f'(\theta_1)$. Then
\BEQ
\langle f'(\theta_2), H^{-1} f'(\theta_2) \rangle
\leqslant
 {\kappa \rho}   e^{2 \sqrt{\kappa \rho} d_1}  
\langle f'(\theta_1),  H^{-1}   f'(\theta_1) \rangle^2
\bigg(
\frac{ e^{\sqrt{\kappa \rho}  d_{12}} - \sqrt{\kappa \rho}  d_{12} - 1 }{ {\kappa \rho} d_{12}^2}\bigg)^2,
\EEQ
where $d_{12}^2 =  { \langle \theta_2 - \theta_1, H ( \theta_2  - \theta_1 ) \rangle}
\leqslant e^{\sqrt{\kappa \rho}  d_1} \langle f'(\theta_1) , H^{-1} f'(\theta_1)\rangle$ 
and $d_1 =  \langle \theta_1 - \theta_\ast, H ( \theta_1  - \theta_\ast ) \rangle^{1/2}$.
\end{proposition}

\begin{proof}
When applying the two previous propositions to  the Newton step $\theta_2 = \theta_1 - f''(\theta_1)^{-1} f'(\theta_1)$, we get:
\BEAS
\langle \Delta, f'(\theta_2)  \rangle
& \leqslant & \sqrt{\kappa \rho}  \langle \Delta,H \Delta \rangle^{1/2}
\langle f'(\theta_1),  f''(\theta_1)^{-1}   f'(\theta_1) \rangle
\frac{ e^{\sqrt{\kappa \rho}  d_{12}} - \sqrt{\kappa \rho}  d_{12} - 1 }{ {\kappa \rho}   d_{12}^2}\\
& \leqslant & \sqrt{\kappa \rho} S e^{\sqrt{\kappa \rho} d_1} \langle \Delta,H \Delta \rangle^{1/2}
\langle f'(\theta_1),  H^{-1}   f'(\theta_1) \rangle
\frac{ e^{\sqrt{\kappa \rho}  d_{12}} - \sqrt{\kappa \rho}  d_{12} - 1 }{ {\kappa \rho} d_{12}^2}.
\EEAS
We then optimize with respect to  $\Delta$ to obtain the desired result.
\end{proof}

The following proposition shows how the Newton decrement is upper bounded by a function of the excess risk.

\begin{proposition}[Newton decrement]
\label{prop:decrement}
Assume \textbf{(B3-4)}, and $\theta_1 \in \H$, then,
\BEQ
\langle f'(\theta_1) , H^{-1} f'(\theta_1)\rangle
  \leqslant   \bigg(
\frac{1}{2}  \sqrt{\kappa \rho}  \Delta_1
+ \sqrt{
   d_1^2
+ \sqrt{\kappa \rho}  d_1  \Delta_1
+ \frac{1}{4}  {\kappa \rho}  \Delta_1^2
}
\bigg)^2,
\EEQ
with $d_1  = \sqrt{ \langle \theta_1 - \theta_\ast, H( \theta_1 - \theta_\ast) \rangle}$
and $\Delta_1 = f(\theta_1) - f(\theta_\ast)$.
\end{proposition}
\begin{proof}
We may   bound the Newton decrement
 as follows:
\BEA
\nonumber
\langle f'(\theta_1) , H^{-1} f'(\theta_1)\rangle
& \leqslant &
\langle f'(\theta_1)   - H( \theta_1 - \theta_\ast), H^{-1} f'(\theta_1)\rangle
+ \langle   H( \theta_1 - \theta_\ast), H^{-1} f'(\theta_1)\rangle
\\
\label{eq:decrement} & \leqslant &
  \sqrt{\kappa \rho}  \big[ f(\theta_1)- f(\theta_\ast) \big] \langle f'(\theta_1) , H^{-1} f'(\theta_1)\rangle^{1/2}
+\langle f'(\theta_1) , \theta_1 - \theta_\ast\rangle
.
\EEA
This leads to
\BEAS
\langle f'(\theta_1) , H^{-1} f'(\theta_1)\rangle
& \leqslant & \bigg(
\frac{1}{2}  \sqrt{\kappa \rho}  \big[ f(\theta_1)- f(\theta_\ast) \big]
+ \sqrt{
\langle f'(\theta_1) , \theta_1 - \theta_\ast\rangle
+ \frac{1}{4}  {\kappa \rho}  \big[ f(\theta_1)- f(\theta_\ast) \big]^2
}
\bigg)^2.
\EEAS
Moreover,
\BEAS
\langle f'(\theta_1) , \theta_1 - \theta_\ast\rangle
& = & 
\langle H ( \theta_1 - \theta_\ast) , \theta_1 - \theta_\ast\rangle
+ \langle f'(\theta_1) - H ( \theta_1 - \theta_\ast), \theta_1 - \theta_\ast\rangle \\
& \leqslant & 
\langle H ( \theta_1 - \theta_\ast) , \theta_1 - \theta_\ast\rangle
+ \sqrt{\kappa \rho}  \langle \theta_1 -\theta_\ast, H(\theta_1 - \theta_\ast) \rangle^{1/2} \big[ f(\theta_1) - f(\theta_\ast) \big] \\
& \leqslant & 
d_1^2
+ \sqrt{\kappa \rho}  d_1  \big[ f(\theta_1) - f(\theta_\ast) \big].
\EEAS
Overall, we get
\BEAS
&&\langle f'(\theta_1) , H^{-1} f'(\theta_1)\rangle \\
& \leqslant & \bigg(
\frac{1}{2}  \sqrt{\kappa \rho}  \big[ f(\theta_1)- f(\theta_\ast) \big]
+ \sqrt{
   d_1^2
+ \sqrt{\kappa \rho}  d_1  \big[ f(\theta_1) - f(\theta_\ast) \big]
+ \frac{1}{4}  {\kappa \rho}  \big[ f(\theta_1)- f(\theta_\ast) \big]^2
}
\bigg)^2.
\EEAS
\end{proof}

The following proposition provides a bound on a quantity which is not the Newton decrement. Indeed, this is (up to the difference in the Hessians), the norm of the Newton step. This will be key in the following proofs.

\begin{proposition}[Bounding gradients from unweighted distance to optimum]
\label{prop:unweighted}
Assume \textbf{(B3-4)}, and $\theta_1 \in \H$, then,
\BEQ
\| H^{-1} f'(\theta_1) 
\| \leqslant 
\frac{ e^{ \sqrt{\kappa \rho}d_1 } -1 }{\sqrt{\kappa \rho}d_1} \| \theta_1 - \theta_\ast\|
,\EEQ
with $d_1  = \sqrt{ \langle \theta_1 - \theta_\ast, H( \theta_1 - \theta_\ast) \rangle}$.
\end{proposition}
\begin{proof}
We have:
\BEAS
\| H^{-1} f'(\theta_1) 
\|
& \leqslant & \| H^{-1} \big[ f'(\theta_1) - H( \theta_1 - \theta_\ast )\big]
\|
+ \| H^{-1} \big[  H( \theta_1 - \theta_\ast )\big]
\| \\ 
& \leqslant & \|  \theta_1 - \theta_\ast \|
\bigg( 1 +  \frac{ e^{ \sqrt{\kappa \rho}d_1 } -1 - \sqrt{\kappa \rho}d_1  }{\sqrt{\kappa \rho}d_1}
\bigg).
\EEAS
\end{proof}

The next proposition shows that having a small Newton decrement implies that the weighted distance to optimum is small.

\begin{proposition}[Weighted distance to optimum]
\label{prop:asd}
Assume \textbf{(B3-4)}.
If we have $ {  \sqrt{ \kappa \rho} e^{ \sqrt{ \kappa \rho} d}   \langle f'(\theta),H^{-1} f'(\theta) \rangle^{1/2}} \leqslant \frac{1}{2}$, 
with $d = \sqrt{ \langle \theta - \theta_\ast, H ( \theta - \theta_\ast) \rangle}$, then
$$d  
 \leqslant 4 e^{\sqrt{ \kappa \rho} d }    \langle f'(\theta),H^{-1} f'(\theta) \rangle^{1/2}.$$
\end{proposition}
\begin{proof}
For any $\Delta \in \H$ such that $\langle \Delta, H \Delta \rangle = 1$,  and $t \geqslant 0$, we have, following the same reasoning than for Prop.~\ref{prop:functionvalues}:
\BEAS
f(\theta  + t\Delta) & \geqslant & f(\theta) + t \langle \Delta,  f'(\theta) \rangle
+  \langle \Delta, f''(\theta) \Delta\rangle \frac{e^{ -vt} + vt - 1}{v^2} \\
& \geqslant & f(\theta) + 
\frac{\langle \Delta, f''(\theta) \Delta\rangle}{v^2}
\bigg[
e^{-vt} - 1 + tv \big(
1 - s \big)
\bigg] 
\EEAS
with $v = \sqrt{ \kappa \rho} \sqrt{ \langle \Delta , H \Delta \rangle }   = \sqrt{ \kappa \rho}$  and
$$s = \frac{  v  | \langle \Delta,  f'(\theta) \rangle|}{   \langle \Delta, f''(\theta) \Delta\rangle}
\leqslant  \frac{  \sqrt{ \kappa \rho}    \langle f'(\theta),  f''(\theta)^{-1} f'(\theta) \rangle^{1/2}}{   \langle \Delta, f''(\theta) \Delta\rangle^{1/2}}
\leqslant  {  \sqrt{ \kappa \rho}  e^{ \sqrt{ \kappa \rho} d}   \langle f'(\theta),H^{-1} f'(\theta) \rangle^{1/2}}.$$

It is shown in~\cite{bach2010self} that if $s  \in [0,1)$, then
$$
e^{-2s /(1-s)} + ( 1- s) 2s (1-s)^{-1} -1 \geqslant 0.
$$
This implies that  if $s \leqslant 1/2$, for $t = \frac{  2 \sqrt{ \kappa \rho}^{-1}s }{1-s}$, $f(\theta_2 + t \Delta) \geqslant f(\theta_2)$. Thus,
\BEQ
\label{eq:d2}
d = \sqrt{ \langle \theta  - \theta_\ast, H(\theta  - \theta_\ast) \rangle} \leqslant
t \leqslant
 {  4 \sqrt{ \kappa \rho}^{-1}s }
 \leqslant 4 e^{\sqrt{ \kappa \rho}d }    \langle f'(\theta),H^{-1} f'(\theta) \rangle^{1/2}.
 \EEQ
 Note that the quantity $d$ appears twice in the result above.
\end{proof}

 \subsection{Proof of Prop.~\ref{prop:approx_newton}}
 In this section, we prove Prop.~\ref{prop:approx_newton} using tools from self-concordance analysis. These tools are described in the previous Sections~\ref{sec:selfc1} and~\ref{sec:selfc2}. In order to understand the proof, it is preferable to read these sections first.
 
 \label{sec:proofprop1}
 
  We use the notation  $t^2 = \kappa \rho \varepsilon_1$. We then get
$d_1^2 \defeq\langle \theta_1 - \theta_\ast, H(\theta_1 - \theta_\ast) \rangle \leqslant  ( 3 + t^2 ) \varepsilon_1$ from Prop.~\ref{prop:wd}.

 \paragraph{Proof of   \eq{always}.}
  We have, from Prop.~\ref{prop:values_few},
 \BEA
\nonumber  f(\theta_3) - f(\theta_\ast)
 & \leqslant &  f(\theta_2) - f(\theta_\ast) +  \sqrt{\rho} \| H^{1/2} ( \theta_3- \theta_2) \|    \\
 \label{eq:AS}
 & \leqslant &  f(\theta_2) - f(\theta_\ast) +  \sqrt{  2 \rho  \varepsilon_2} e^{ \sqrt{ \kappa \rho}d_1/2}    .
\EEA
Moreover, we have, also
from Prop.~\ref{prop:values_few},
$  f(\theta_2) - f(\theta_\ast) \leqslant  \sqrt{\rho} \| H^{1/2} f''(\theta_1)^{-1} f'(\theta_1) \|$, and using Prop.~\ref{prop:hessians}, we get
\BEQ
\label{eq:AT}
 f(\theta_2) - f(\theta_\ast) \leqslant  e^{ \sqrt{\kappa \rho} d_1 } \sqrt{\rho} \| H^{-1/2}  f'(\theta_1) \|.
 \EEQ

We may now use Prop.~\ref{prop:decrement} and use the bound:
\BEA
\nonumber
\langle f'(\theta_1) , H^{-1} f'(\theta_1)\rangle
& \leqslant & \bigg(
\frac{1}{2}  \sqrt{\kappa \rho} \varepsilon_1
+ \sqrt{
(3+t^2) \varepsilon_1
+ \sqrt{\kappa \rho} \sqrt{ (3+t^2) \varepsilon_1}  \varepsilon_1+ \frac{1}{4}  {\kappa \rho}  \varepsilon_1^2
}
\bigg)^2 \\
\nonumber
& \leqslant & \bigg(
\frac{1}{2}  t \sqrt{\varepsilon_1}
+   \sqrt{
 (3+t^2)\varepsilon_1
+ t \sqrt{(3+t^2) }  \varepsilon_1+ \frac{1}{4}  t^2 \varepsilon_1
}
\bigg)^2 \\
\label{eq:AD}& = & \bigg(
\frac{1}{2}  t  
+  \sqrt{
(3+t^2)
+ t \sqrt{(3+t^2)  }   + \frac{1}{4}  t^2  
}
\bigg)^2  \varepsilon_1  \defeq \square_1(t)^2 \varepsilon_1.
\EEA
A simple plot shows that for all $t>0$, 
\BEQ
\label{eq:DE}
\square_1(t) = \frac{1}{2}  t  
+  \sqrt{
(3+t^2)
+ t \sqrt{(3+t^2)  }   + \frac{1}{4}  t^2  
} \leqslant \sqrt{3} + 2 t.
\EEQ
Combining with \eq{AS} and \eq{AT}, we get
$$ f(\theta_3) - f(\theta_\ast)
  \leqslant   e^{ \sqrt{3+t^2} t} \sqrt{\rho \varepsilon_1} ( \sqrt{ 3 } + 2t) +  \sqrt{  2 \rho  \varepsilon_2} e^{\sqrt{3+t^2} t/2} ,
  $$
which is  exactly \eq{always}.

\paragraph{Proof of \eq{almost1} and \eq{almost2}.} For these two inequalities, the starting point is the same.
Using \eq{AD} (i.e., the Newton decrement at $\theta_1$), we first show that  the distances $d_{12}$ and $d_2$ are bounded.
Using $f''(\theta_1) \succcurlyeq e^{ -\sqrt{  \kappa \rho }d_1  }  H$ (Prop.~\ref{prop:hessians}):
$$d_{12}^2 \leqslant e^{ \sqrt{  \kappa \rho }d_1  } \langle f'(\theta_1) , H^{-1} f'(\theta_1)\rangle
\leqslant 
 e^{ t \sqrt{3+t^2} } 
  \square_1(t)^2 \varepsilon_1 \defeq  \square_2(t)^2 \varepsilon_1   ,
 $$
 and thus
  $$d_2 \leqslant d_1 + d_{12} \leqslant  \bigg[  \sqrt{3+t^2} + 
    \square_2(t)
   \bigg] \sqrt{\varepsilon_1}.
  $$

Now, we can bound the Newton decrement at $\theta_2$, using Prop.~\ref{prop:effect_newton}:
\BEAS
\langle f'(\theta_2) , H^{-1} f'(\theta_2)\rangle
& \leqslant & \kappa \rho e^{ 2 \sqrt{\kappa \rho} d_1 } 
\langle f'(\theta_1) , H^{-1} f'(\theta_1)\rangle^2 \bigg(\frac{e^{ \sqrt{\kappa \rho } d_{12} } -  \sqrt{\kappa \rho } d_{12}
- 1 }{ (\sqrt{\kappa \rho } d_{12} )^2}\bigg)^2
 \\
 & \leqslant & 
 \kappa \rho  \varepsilon_1^2 \square_2(t)^4
 \bigg(\frac{e^{ t \square_2(t)} -  t \square_2(t)
- 1 }{ (t \square_2(t) )^2}\bigg)^2
\defeq  \square_3(t) \kappa \rho  \varepsilon_1^2.   \EEAS

Thus, using Prop.~\ref{prop:asd}, if $
\kappa \rho e^{ 2 \sqrt{ \kappa \rho} d_2} 
\langle f'(\theta_2) , H^{-1} f'(\theta_2)\rangle
\leqslant
t^4  e^{2  t  [  \sqrt{3+t^2} + 
    \square_2(t)
   ]  }
\square_3(t)   \leqslant \frac{1}{4}$,
then
$$d_2 \leqslant 4 e^{ \sqrt{\kappa \rho}  d_2} 
\sqrt{\square_3(t) \kappa \rho  \varepsilon_1^2} 
\leqslant
 4 e^{  t  [  \sqrt{3+t^2} + 
    \square_2(t)
   ] } 
\sqrt{\square_3(t) \kappa \rho  \varepsilon_1^2} 
\defeq 
 \square_4(t) \sqrt{ \kappa \rho \varepsilon_1^2}.$$
We then have
\BEAS
d_3 = \sqrt{ \langle \theta_3 - \theta_\ast, H ( \theta_3 -\theta_\ast ) \rangle}
& \leqslant & 
\sqrt{ \langle \theta_3 - \theta_2, H ( \theta_3 -\theta_2 ) \rangle}
+\sqrt{ \langle \theta_2 - \theta_\ast, H ( \theta_2 -\theta_\ast ) \rangle}
= d_{23} + d_2 \\
& \leqslant & \square_4(t) \sqrt{ \kappa \rho \varepsilon_1^2}
+ \sqrt{ 2 \varepsilon_2 }  e^{ t\sqrt{ 3 + t^2 } /2}
\\
d_3 \sqrt{\kappa \rho}
& \leqslant 
& \square_4(t) t^2 
+ \sqrt{ 2 \varepsilon_2  \kappa \rho}  e^{ t\sqrt{ 3 + t^2 } /2}
\leqslant \square_4(t) t^2 
+ \sqrt{ 2 u^2}  e^{ t\sqrt{ 3 + t^2 } /2}
\EEAS
where $\varepsilon_2  \kappa \rho \leqslant u^2$.

We now have two separate paths to obtain \eq{almost1} and \eq{almost2}. 

 If we assume that $\varepsilon_2$ is bounded, i.e., with $t = 1/16$ and $u=1/4$, then, one can check computationally that  we obtain $d_3 \sqrt{\kappa \rho}  \leqslant 0.41$ and thus
$b=0.576$ below:
\BEAS
f(\theta_3) - f(\theta^\ast) 
& \leqslant &  \frac{ e^{\sqrt{\kappa \rho} d_3} - \sqrt{\kappa \rho} d_3 - 1}{ {\kappa \rho} } \mbox{ using Prop.~\ref{prop:value_from_opt}}, \\
& \leqslant & d_3^2 \max_{ \alpha \in [0 ,  0.41 } \frac{e^{\alpha} - \alpha - 1 }{\alpha^2}
\leqslant 0.576
\big( \square_4(t) \sqrt{ \kappa \rho \varepsilon_1^2}
+ \sqrt{ 2 \varepsilon_2 }  e^{ t\sqrt{ 3 + t^2 } /2}
\big)^2
\\
& \leqslant &  
0.576 (1+1/c)   \square_4(t)^2 \kappa \rho \varepsilon_1^2
+ 2 \times 0.576 (1+c) e^{ t\sqrt{ 3 + t^2 } } \varepsilon_2  \\
& \leqslant &  57
 \kappa \rho \varepsilon_1^2
+ 
12
\varepsilon_2 , \mbox{ with } c=8.1, \EEAS
which is exactly \eq{almost2}. 


If we only assume $\varepsilon_1$ bounded, then we have (from the beginning of the proof):
\BEAS
  f(\theta_3) - f(\theta_\ast)
 & \leqslant &  f(\theta_2) - f(\theta_\ast) +  \sqrt{\rho} \| H^{1/2} ( \theta_3- \theta_2) \|    \\
 & \leqslant &  57
 \kappa \rho \varepsilon_1^2
 +  \sqrt{  2 \rho  \varepsilon_2}  e^{ t \sqrt{3+t^2} / 2}   \leqslant    57
 \kappa \rho \varepsilon_1^2
 +  2 \sqrt{    \rho  \varepsilon_2}  ,
\EEAS
because we may use the earlier reasoning with $\varepsilon_3=0$. This is \eq{almost1}.

 \newpage

 \section{Additional experiments}
 
 In Table~\ref{tab:data2}, we describe the datasets we  have used in experiments and where they were downloaded from.

\urlstyle{same}

\begin{table}
\caption{Datasets used in our experiments \label{tab:data2}. We report the proportion of non-zero entries.}
\begin{center}
\begin{tabular}{|l|r|r|r|l|}
\hline
Name & $d$ & $n$ & sparsity &  \\
\hline
\emph{quantum} & 79 & 50 000 &    100 \% & {\small\url{osmot.cs.cornell.edu/kddcup/} }\\
\emph{covertype} & 55 & 581 012 &   100 \% &  {\small\url{www.csie.ntu.edu.tw/~cjlin/libsvmtools/datasets/}} \\
\emph{alpha} & 501 & 500 000 &   100 \% & {\small\url{ftp://largescale.ml.tu-berlin.de/largescale/}}\\
\emph{sido} &   4 933& 12 678 &    10 \% & {\small\url{www.causality.inf.ethz.ch/}}\\
\emph{rcv1} &         47 237& 20 242 &   0.2 \%  & {\small\url{www.csie.ntu.edu.tw/~cjlin/libsvmtools/datasets/}}\\
\emph{news} &       1 355 192 & 19 996  &   0.03 \% &   {\small\url{www.csie.ntu.edu.tw/~cjlin/libsvmtools/datasets/}}\\
\hline
\end{tabular}
\end{center}
\end{table}

 In \myfig{addtest}, we provide similar results than in \mysec{experiments}, for two additional datasets, \emph{quantum} and \emph{rcv1}, while in \myfig{train1}, \myfig{train2} and \myfig{train3}, we provide training objectives for all methods. 
 We can make the following observations:
 \BIT
 \item[--] For non-sparse datasets, SAG manages to get the smallest training error, confirming the results of~\cite{sag}.
 \item[--] For the high-dimensional sparse datasets, constant step-size SGD is performing best (note that as shown in \mysec{logistic}, it is not converging to the optimal value in general, this happens notably for the \emph{alpha} dataset).
 \EIT
 
\begin{figure}

\begin{center}
 
 \hspace*{-.5cm}
\includegraphics[scale=.43]{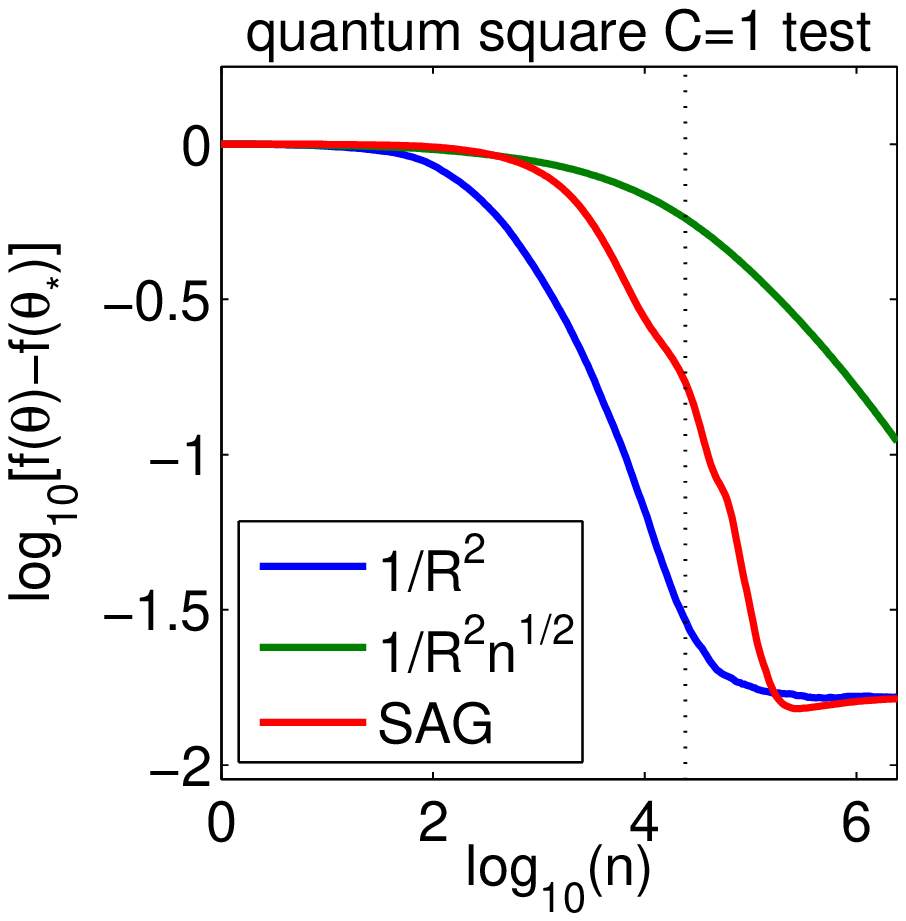}
\hspace*{-.25cm}
\includegraphics[scale=.43]{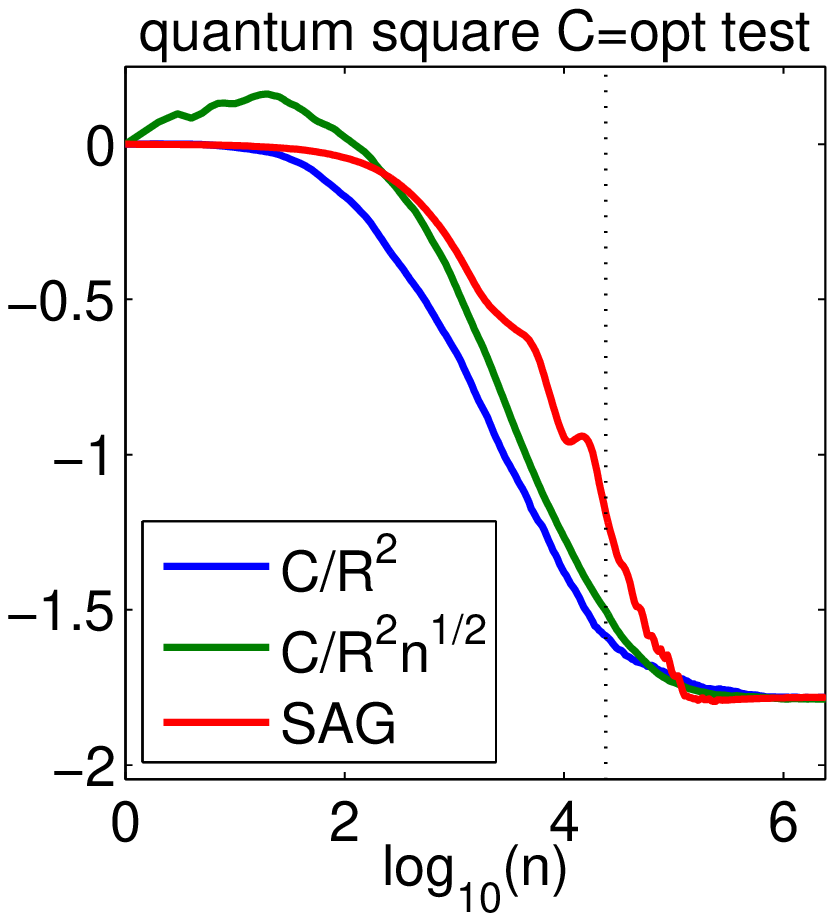}
\includegraphics[scale=.43]{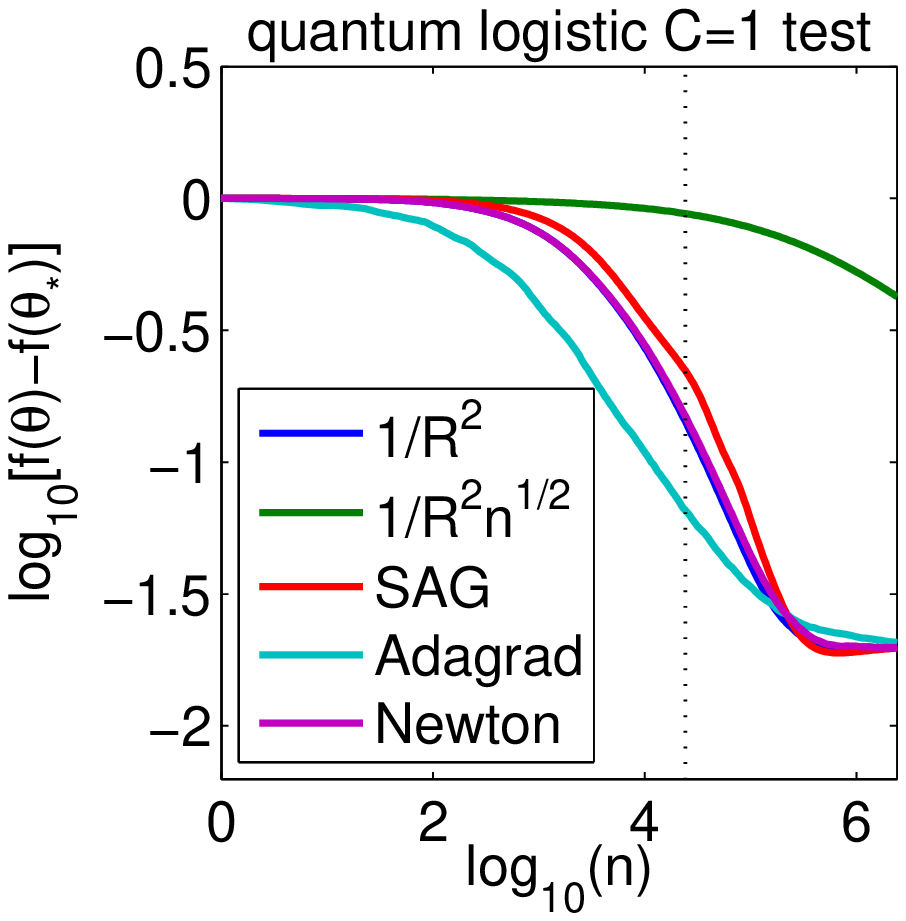}
\hspace*{-.25cm}
\includegraphics[scale=.43]{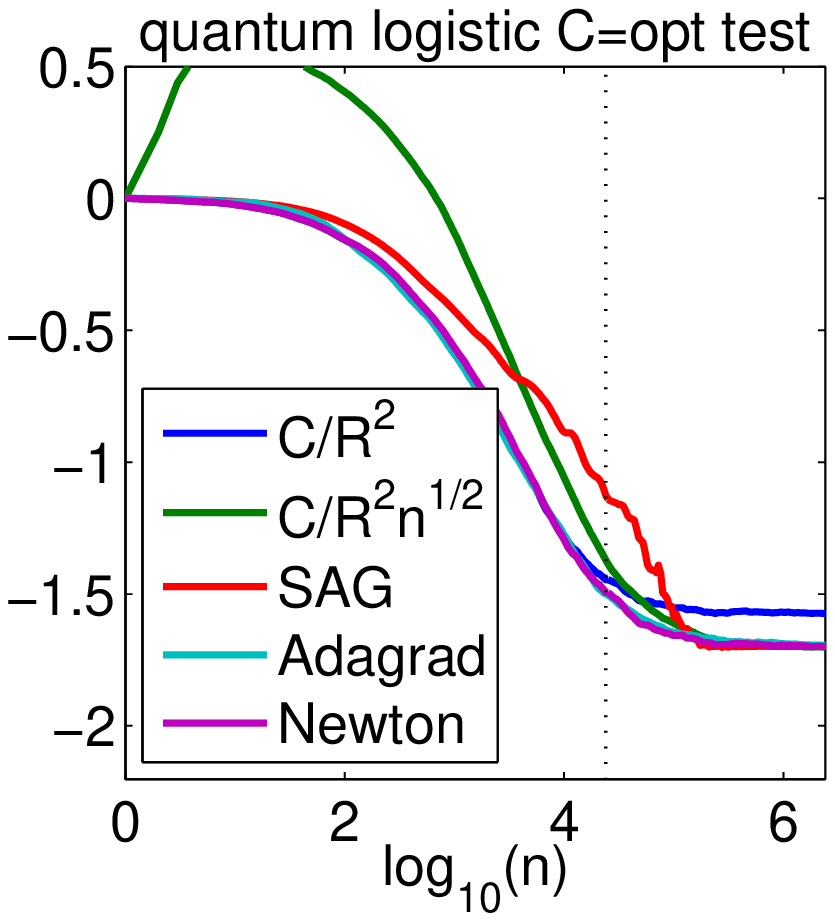}
\hspace*{-.5cm}

\hspace*{-.5cm}
\includegraphics[scale=.43]{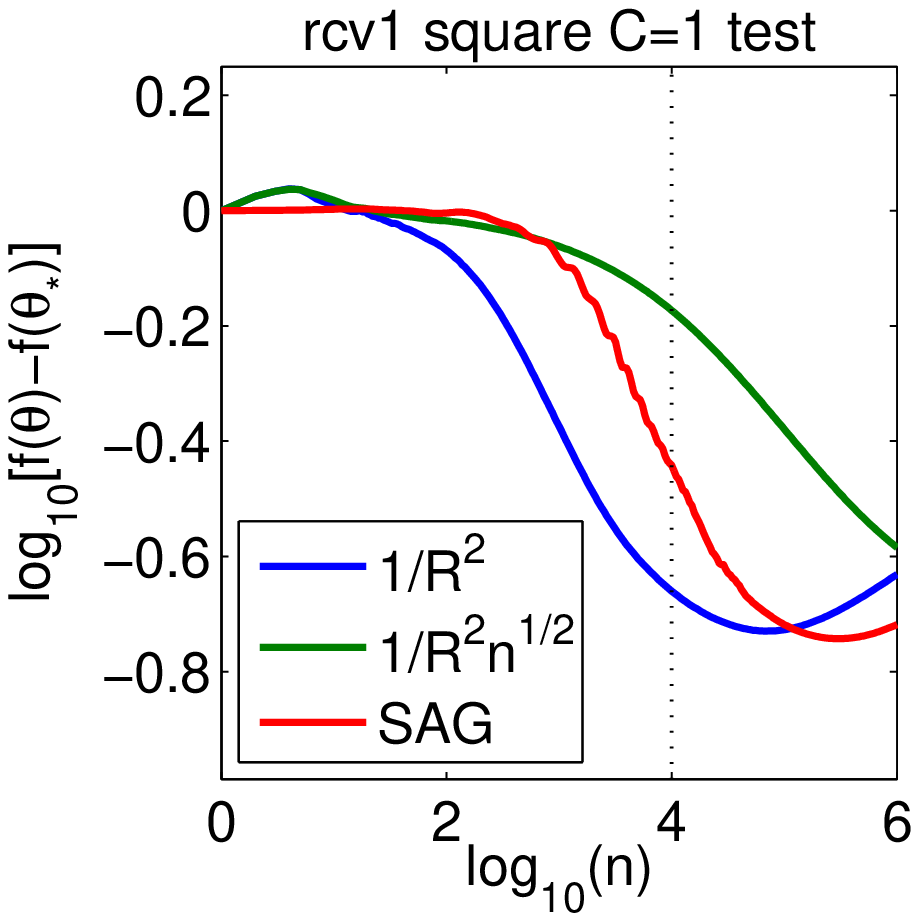}
\hspace*{-.25cm}
\includegraphics[scale=.43]{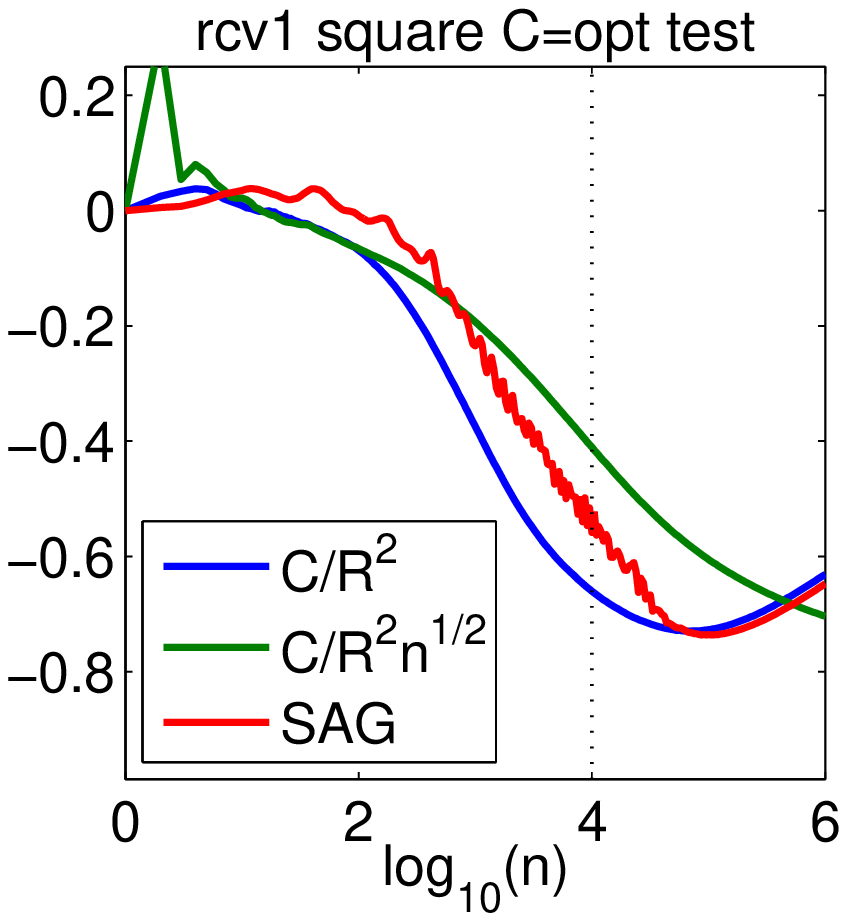}
\includegraphics[scale=.43]{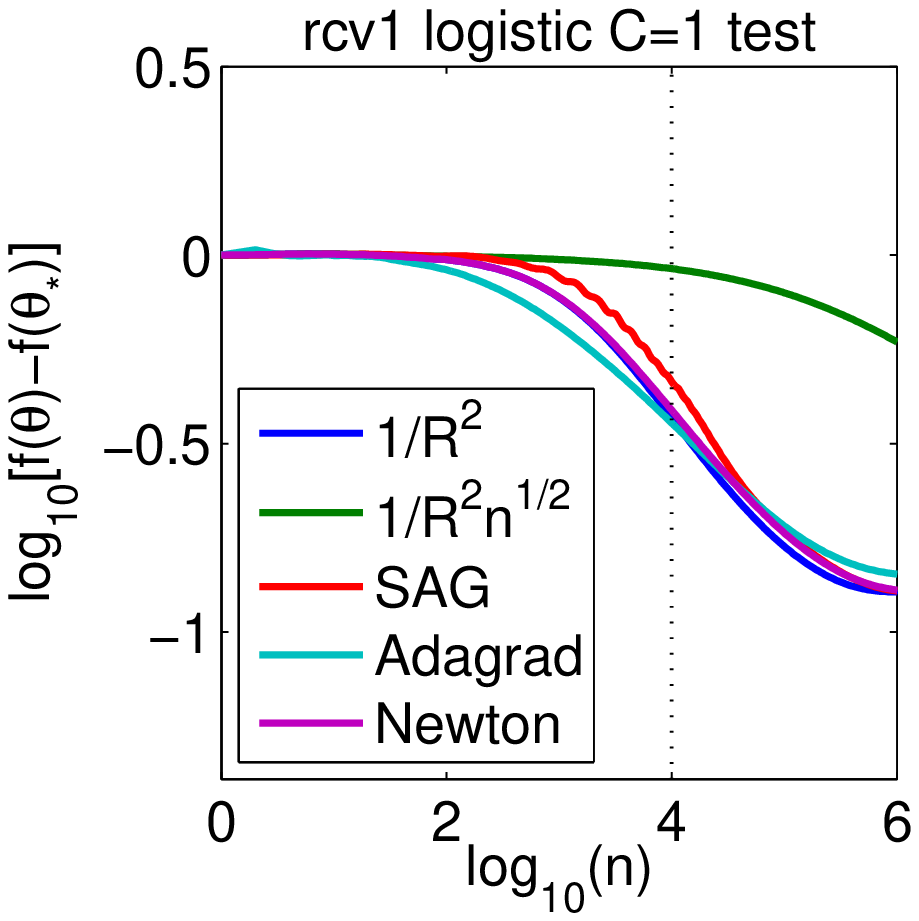}
\hspace*{-.25cm}
\includegraphics[scale=.43]{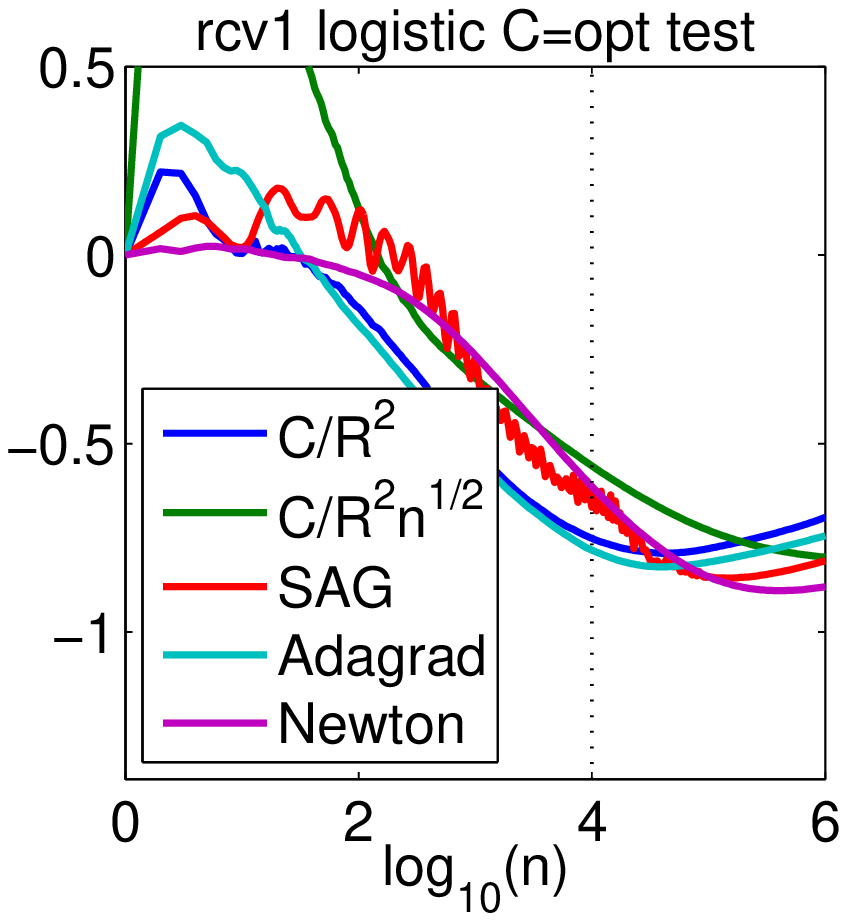}
\hspace*{-.5cm}

\end{center}

\caption{Test  performance for least-square regression (two left  plots) and  logistic regression (two right plots). From top to bottom:  \emph{quantum}, \emph{rcv1}. Left: theoretical steps, right: steps optimized for performance after one effective pass through the data. }
\label{fig:addtest}
\end{figure}

\begin{figure}

\begin{center}

\hspace*{-.5cm}
\includegraphics[scale=.43]{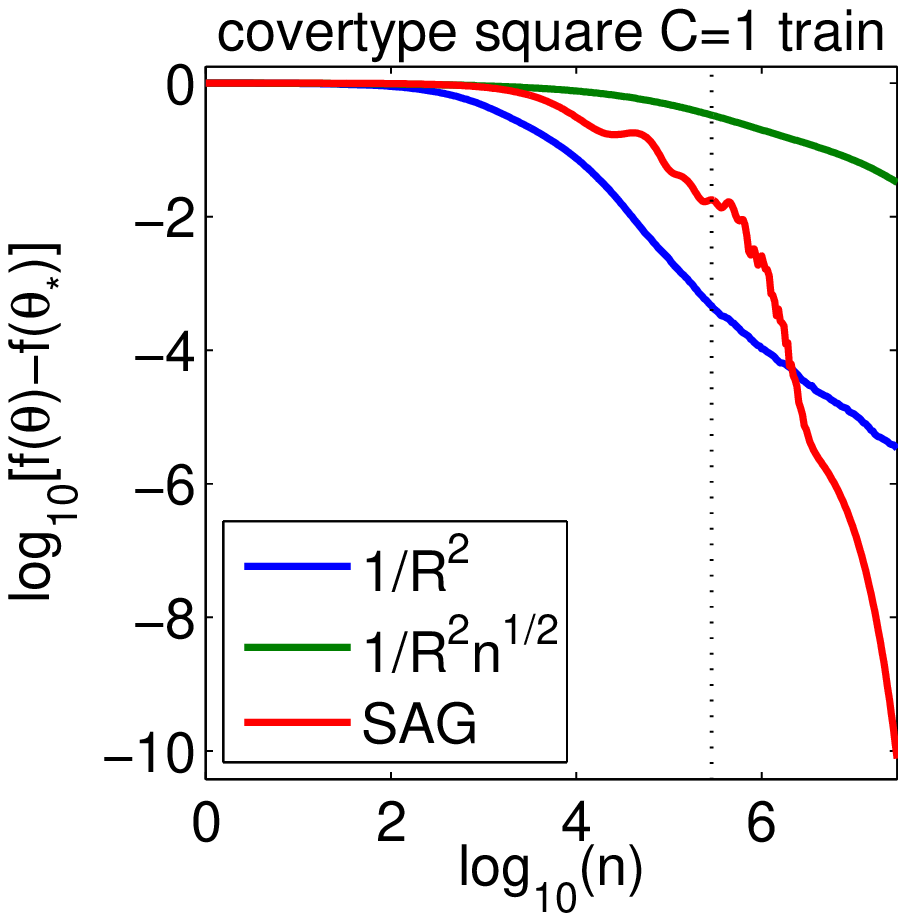}
\hspace*{-.25cm}
\includegraphics[scale=.43]{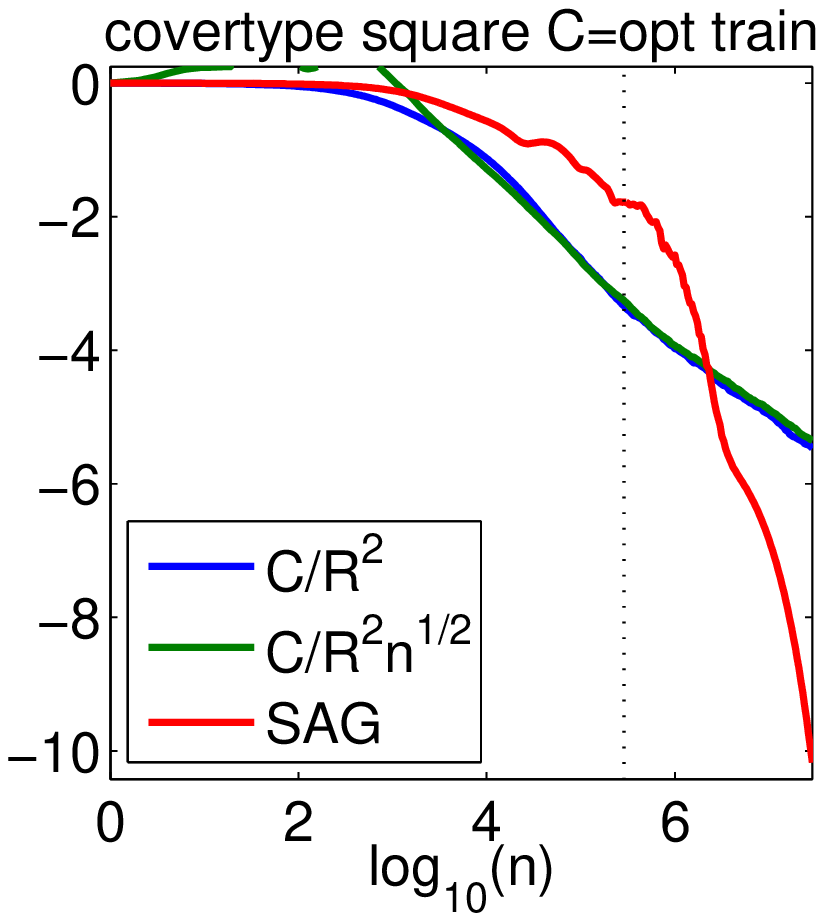}
\includegraphics[scale=.43]{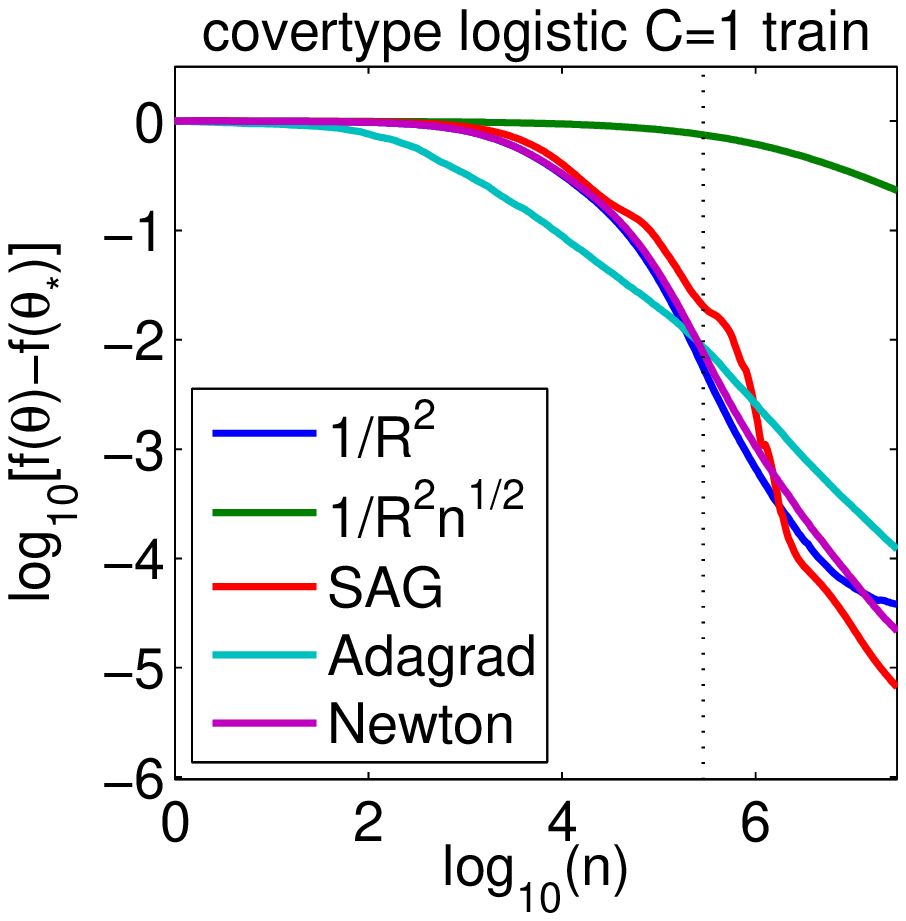}
\hspace*{-.25cm}
\includegraphics[scale=.43]{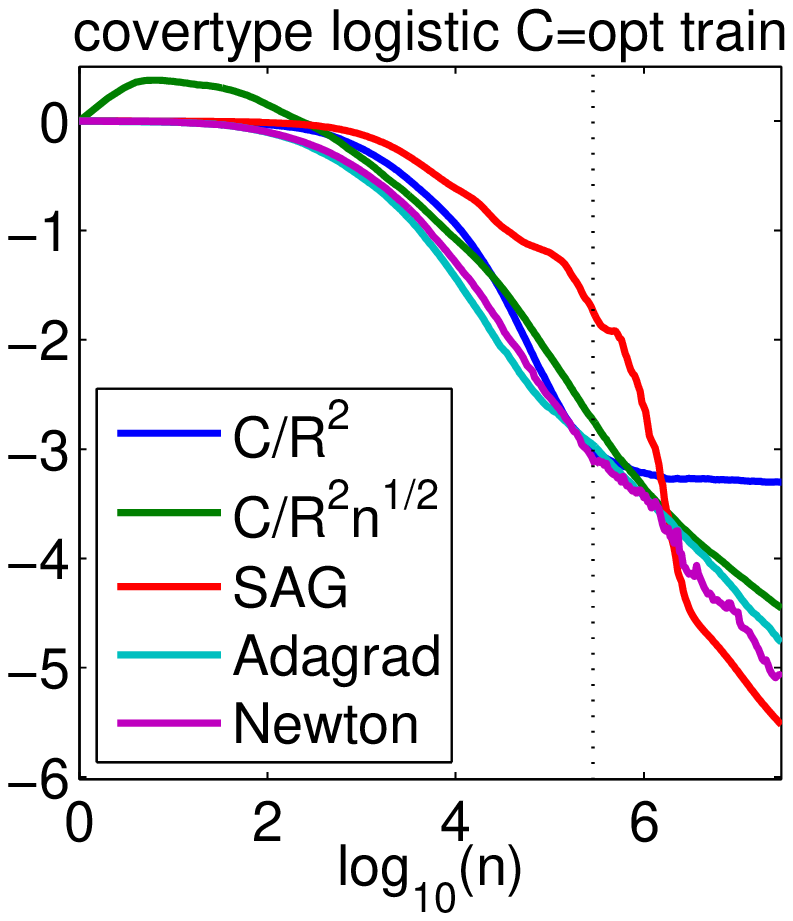}
\hspace*{-.5cm}

\hspace*{-.5cm}
\includegraphics[scale=.43]{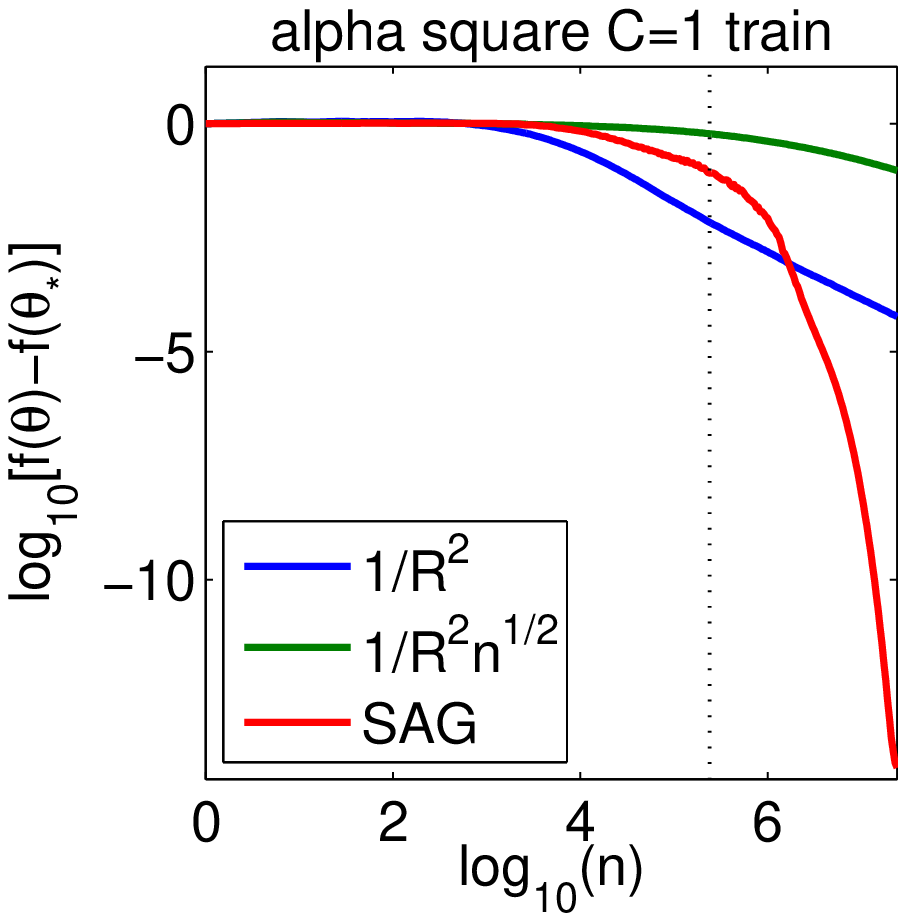}
\hspace*{-.25cm}
\includegraphics[scale=.43]{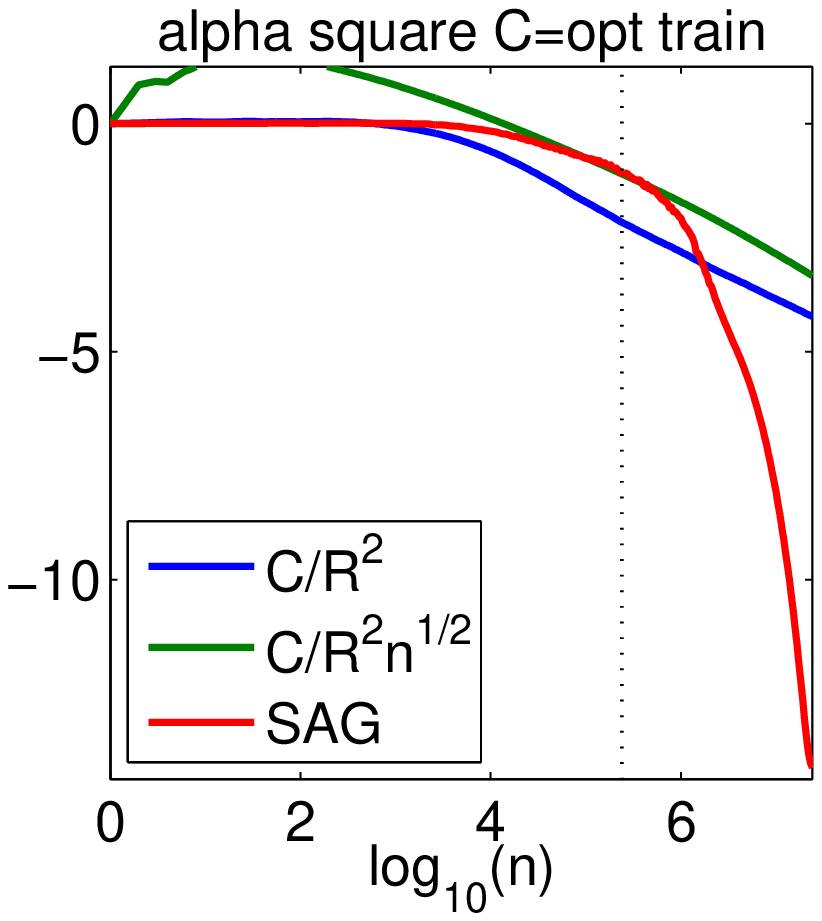}
\includegraphics[scale=.43]{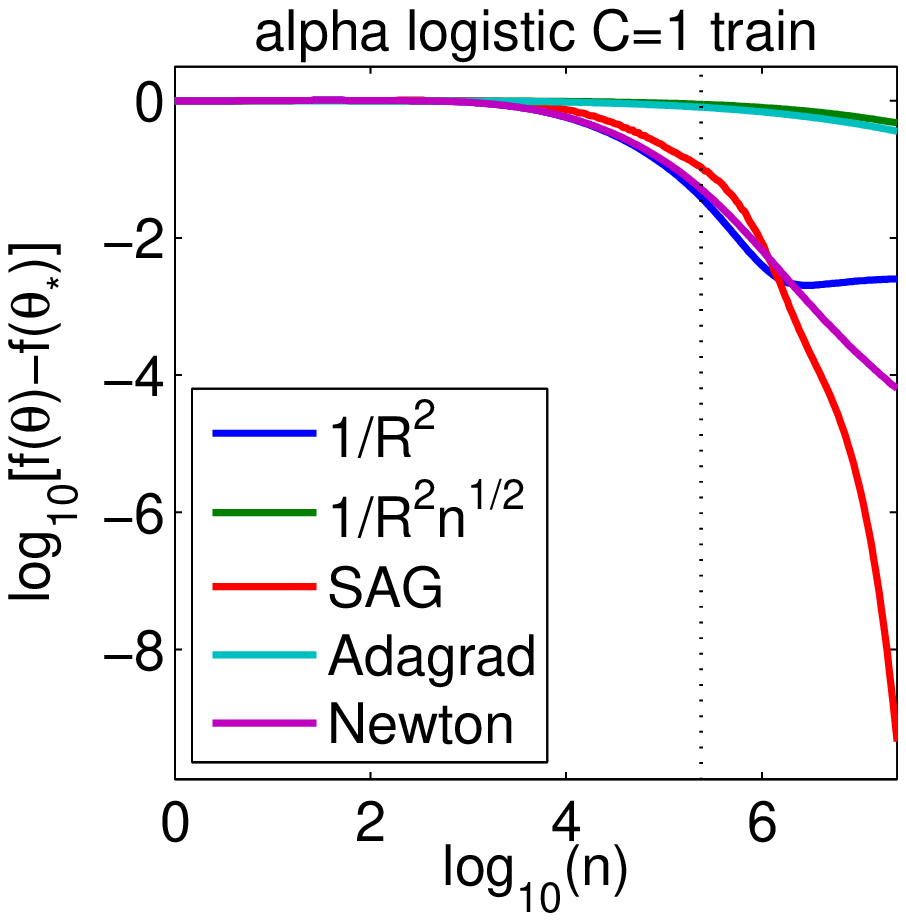}
\hspace*{-.25cm}
\includegraphics[scale=.43]{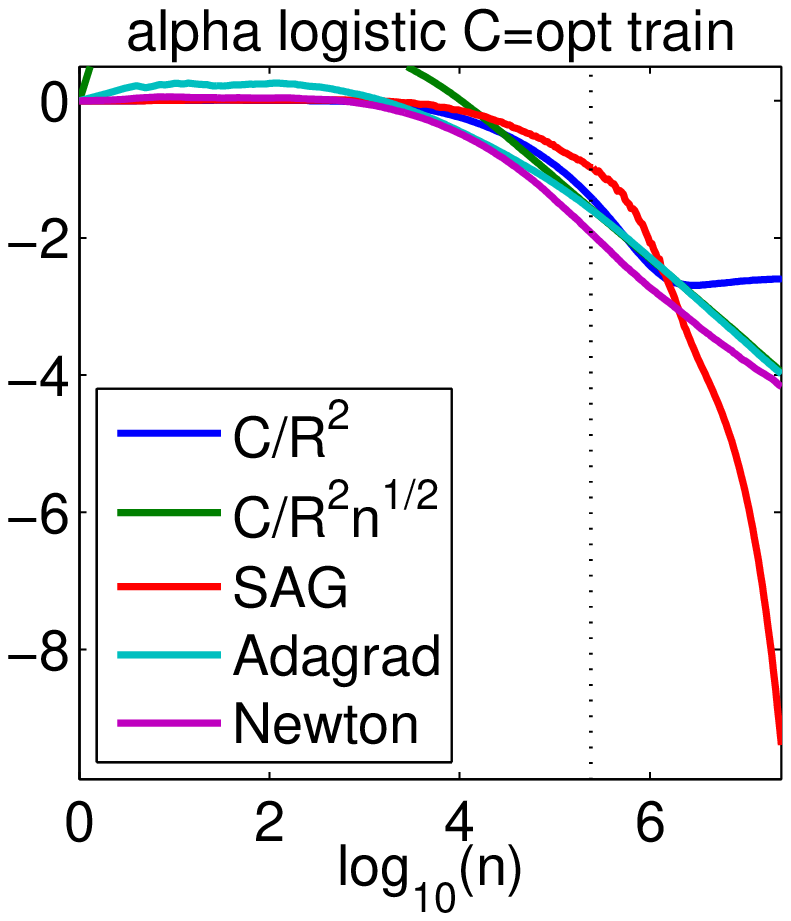}
\hspace*{-.5cm}

\end{center}

\caption{Training objective for least square regression (two left plots) and logistic regression (two right plots).  From top to bottom:  \emph{covertype}, \emph{alpha}. Left: theoretical steps, right: steps optimized for performance after one effective pass through the data. }
\label{fig:train1}
\end{figure}

\begin{figure}

\begin{center}
\hspace*{-.5cm}
\includegraphics[scale=.43]{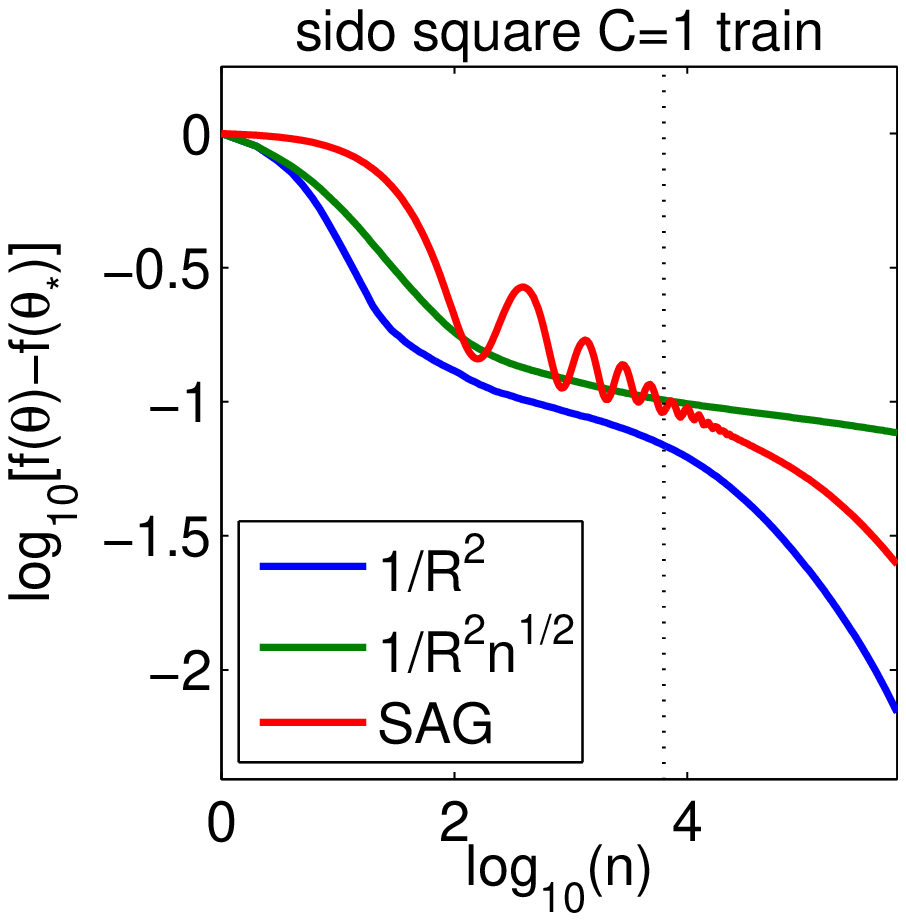}
\hspace*{-.25cm}
\includegraphics[scale=.43]{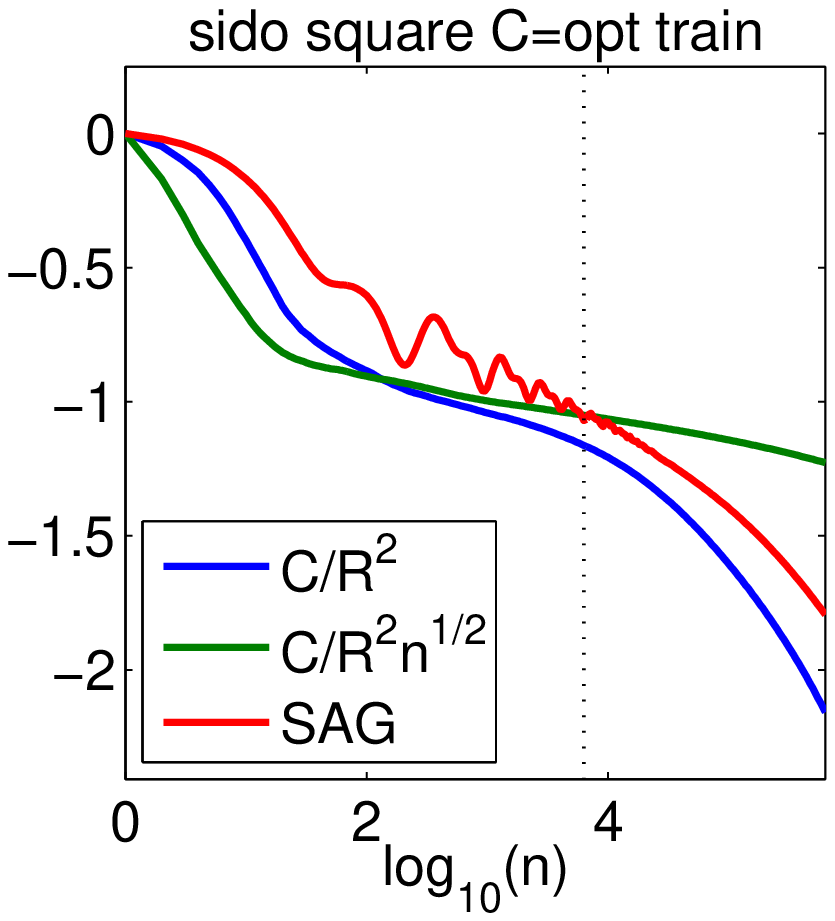}
\includegraphics[scale=.43]{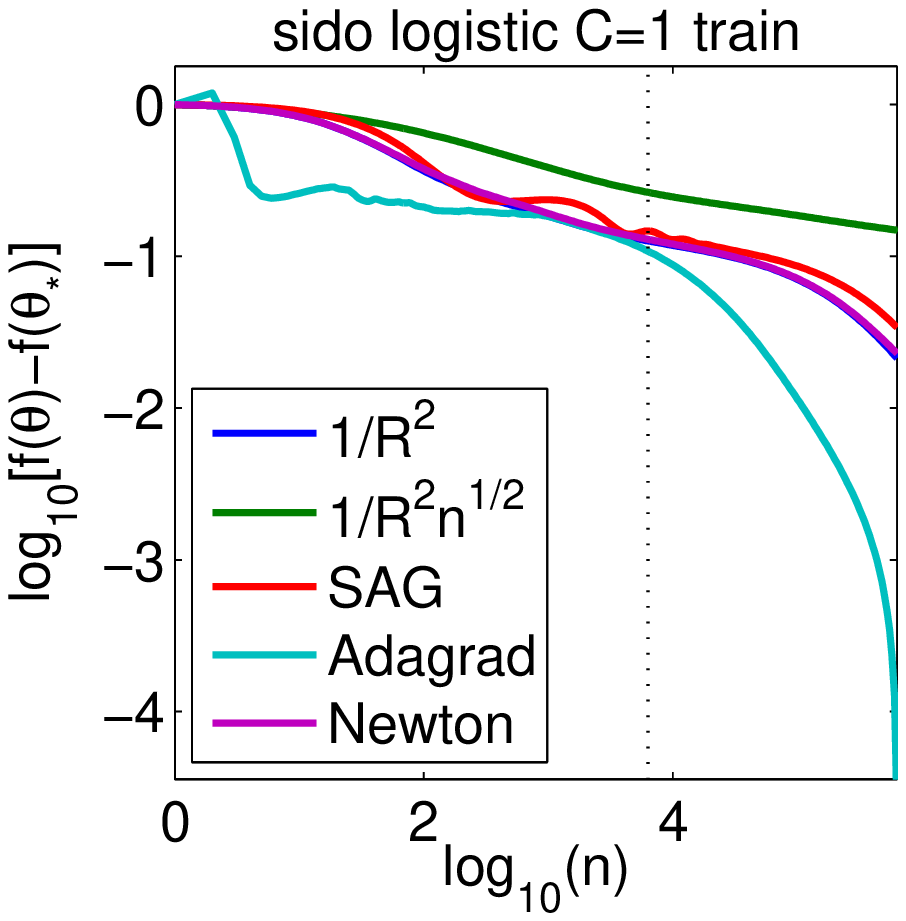}
\hspace*{-.25cm}
\includegraphics[scale=.43]{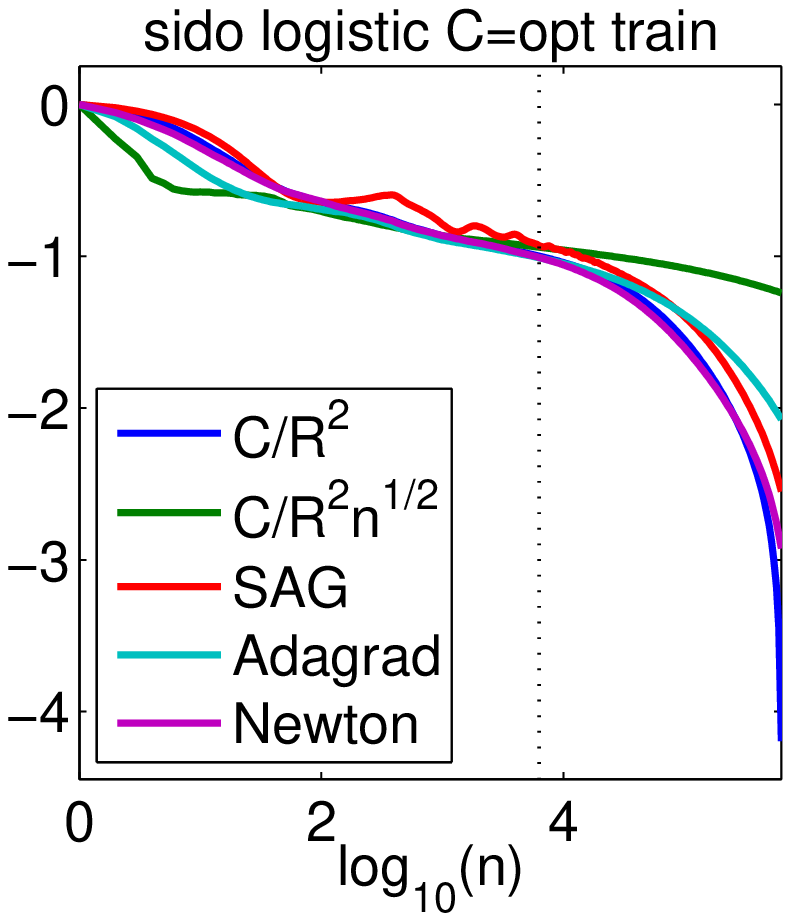}
\hspace*{-.5cm}

\hspace*{-.5cm}
\includegraphics[scale=.43]{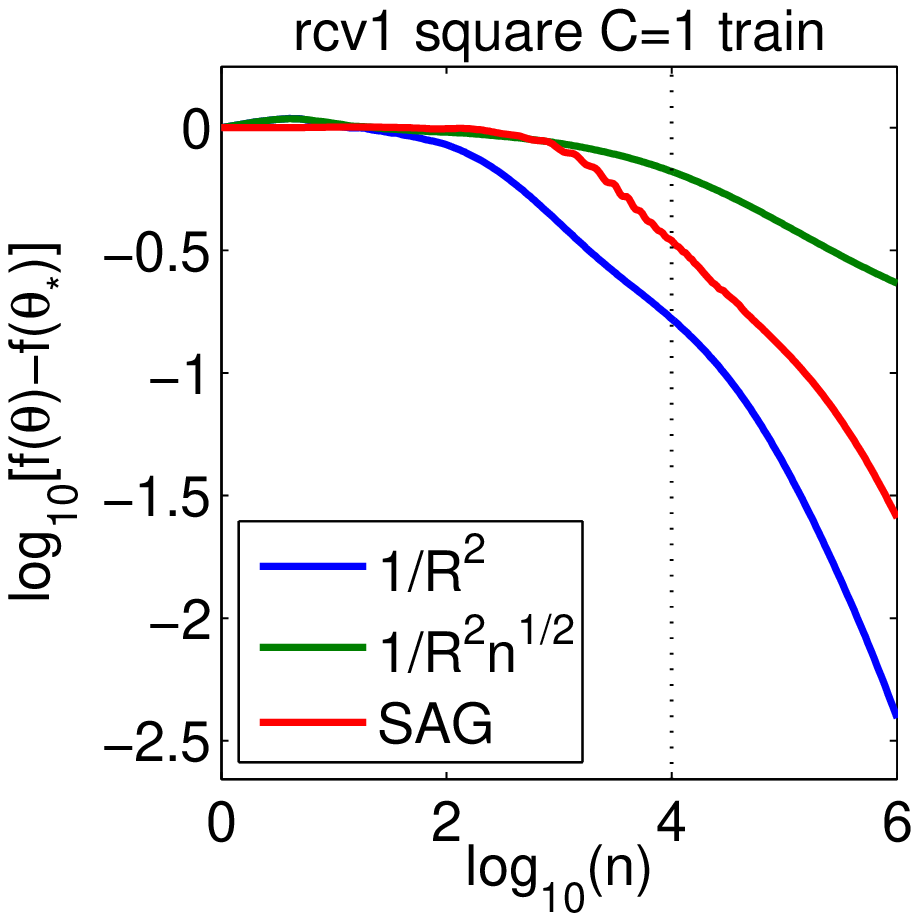}
\hspace*{-.25cm}
\includegraphics[scale=.43]{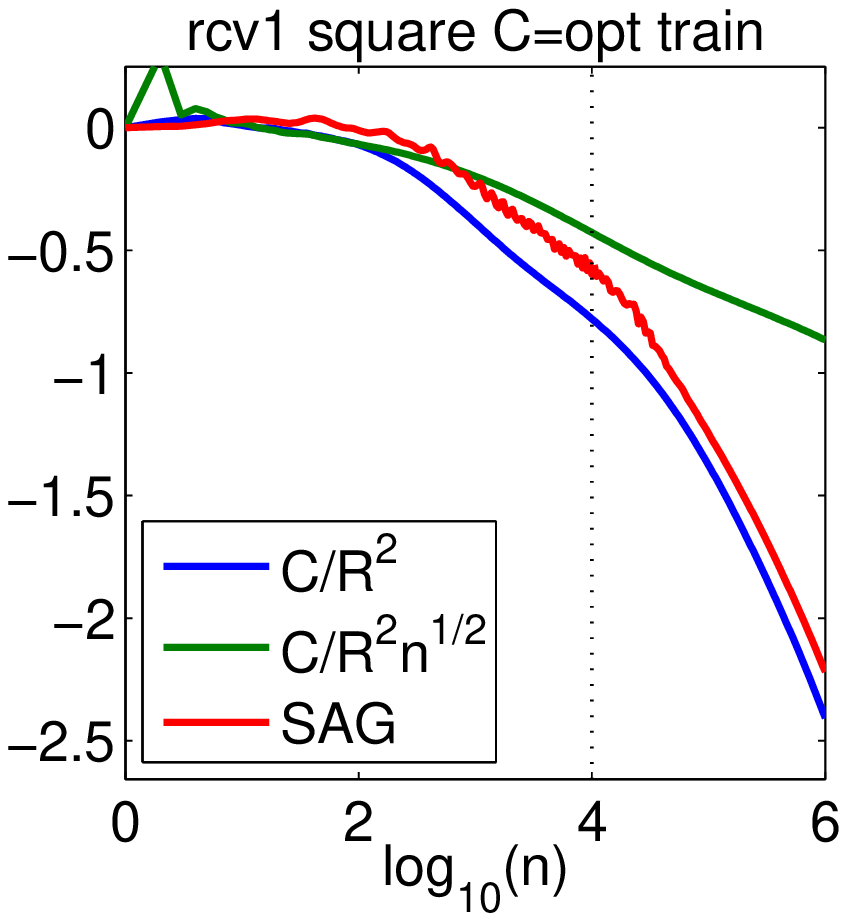}
\includegraphics[scale=.43]{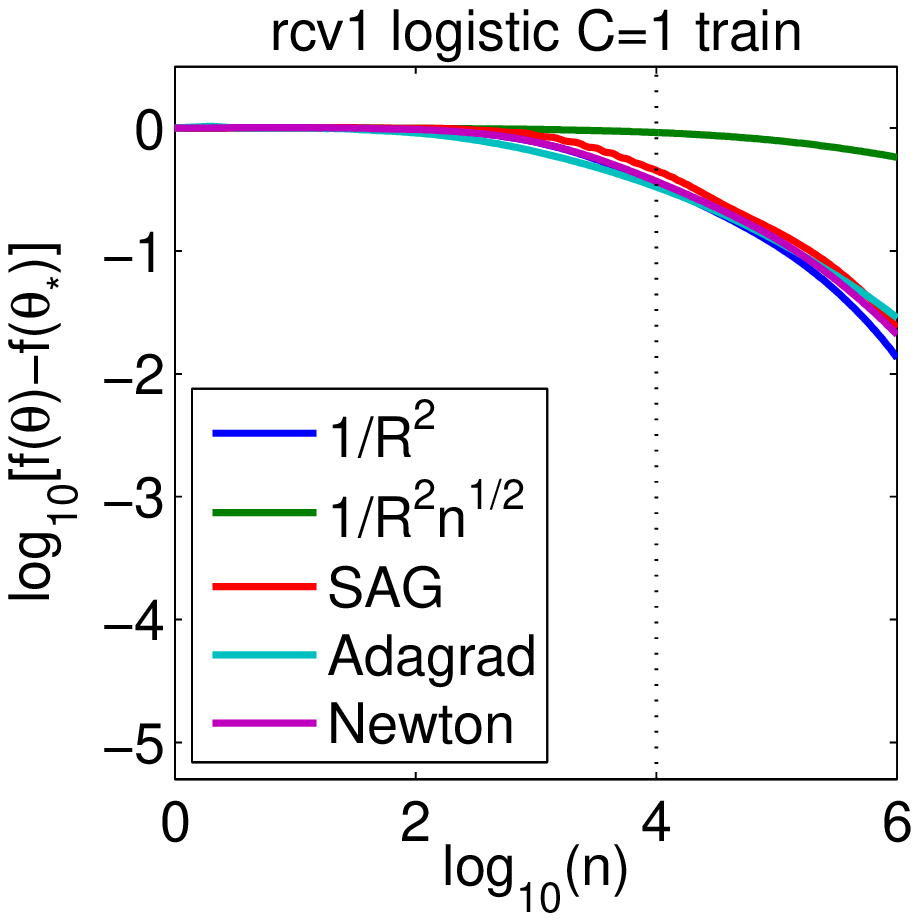}
\hspace*{-.25cm}
\includegraphics[scale=.43]{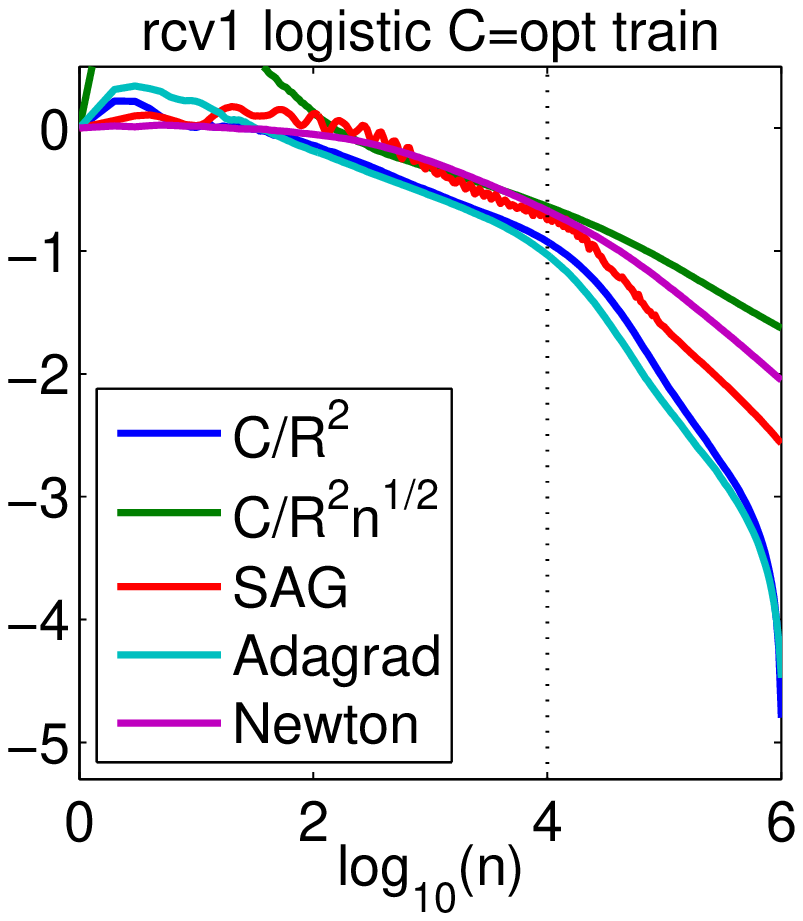}
\hspace*{-.5cm}

\hspace*{-.5cm}
\includegraphics[scale=.43]{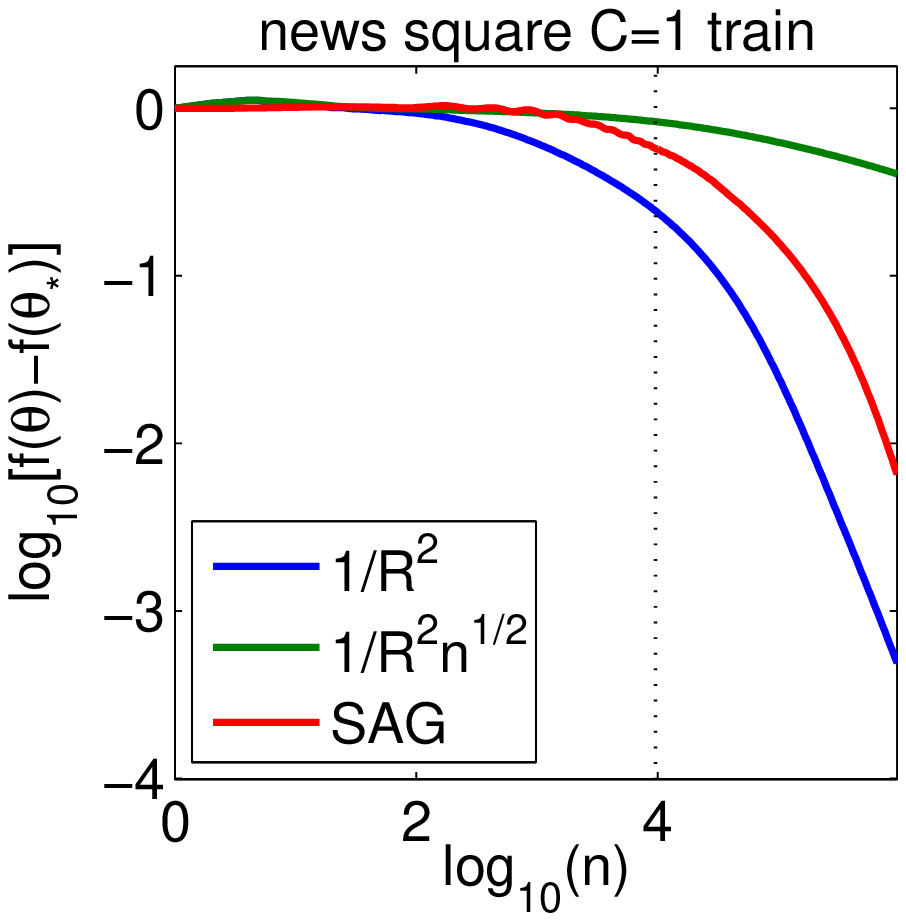}
\hspace*{-.25cm}
\includegraphics[scale=.43]{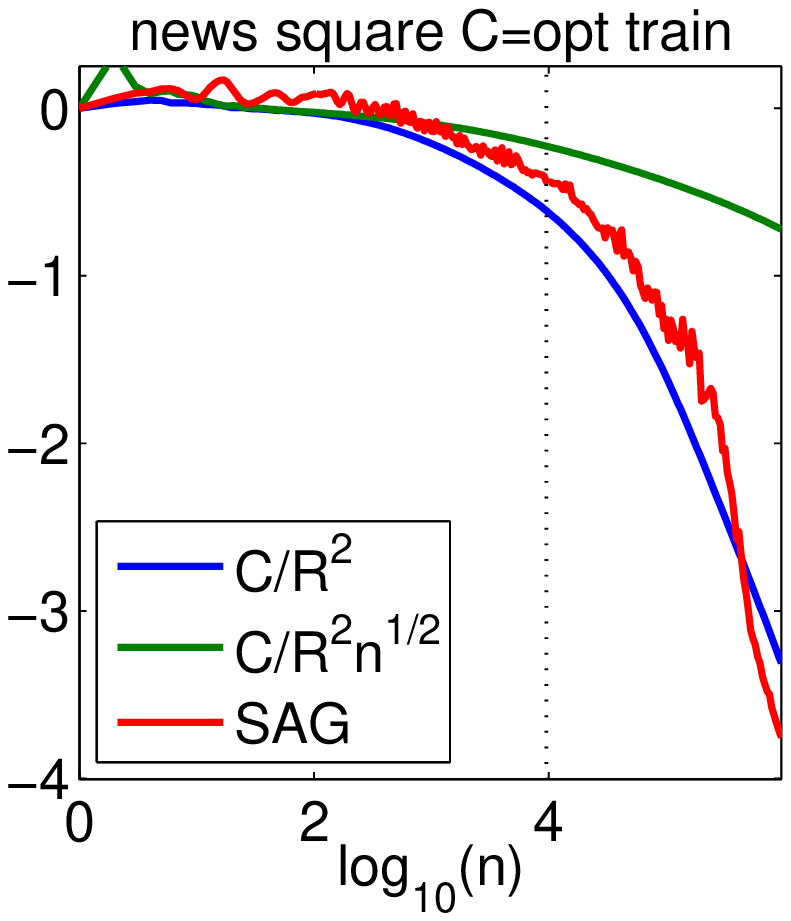}
\includegraphics[scale=.43]{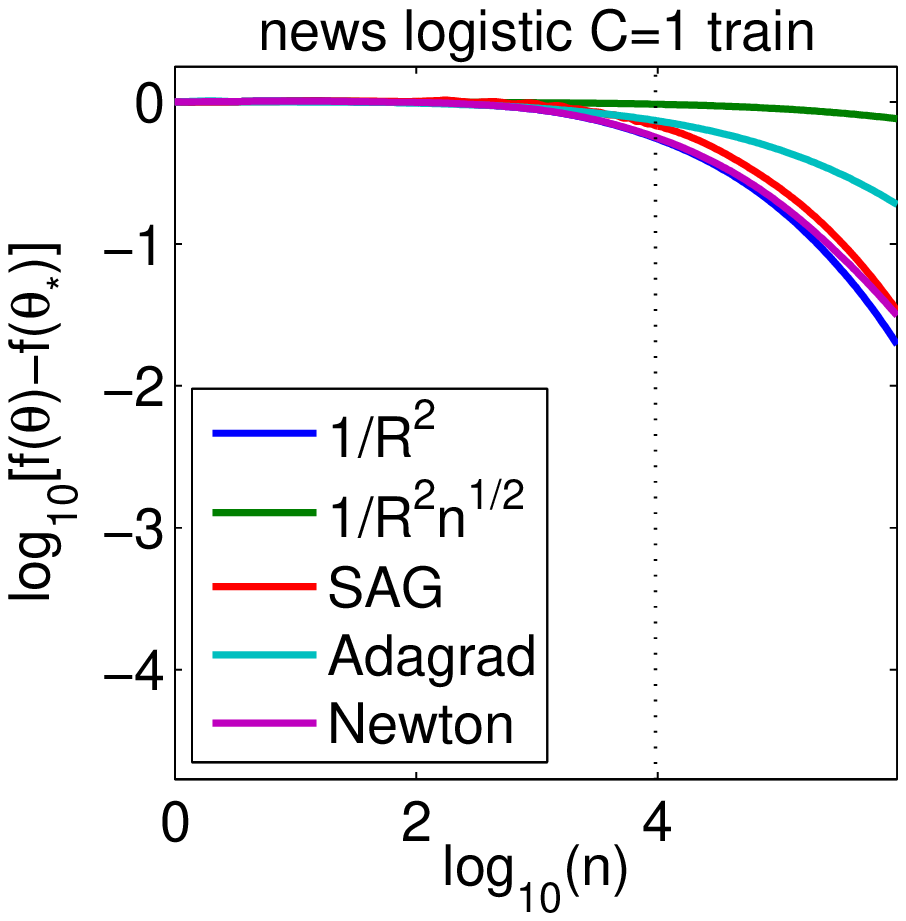}
\hspace*{-.25cm}
\includegraphics[scale=.43]{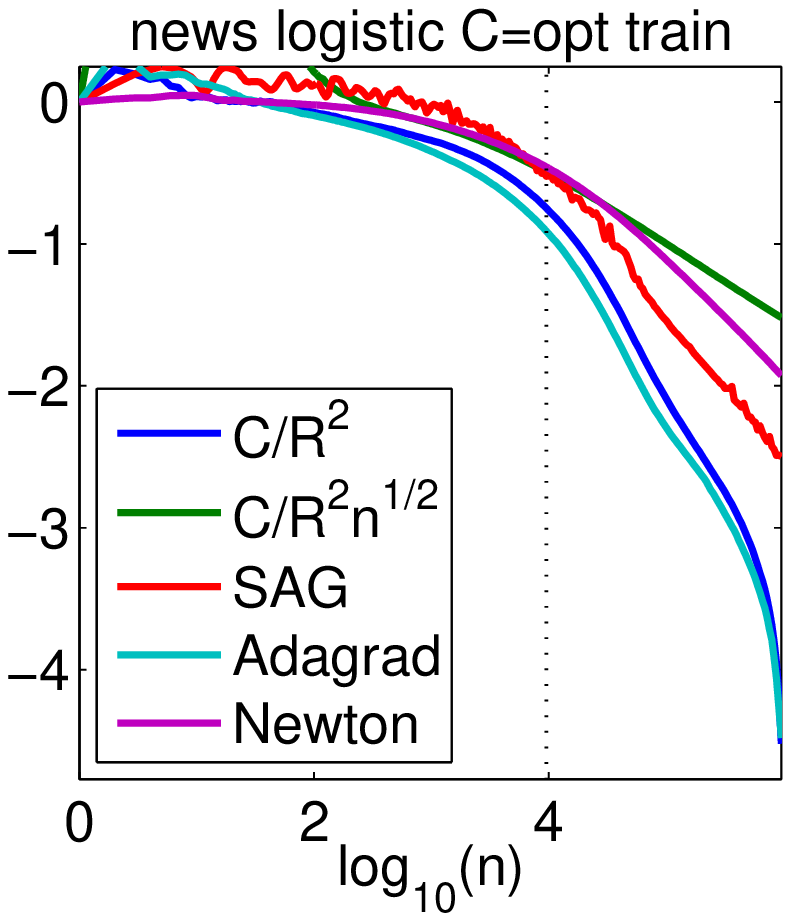}
\hspace*{-.5cm}

\end{center}

\caption{Training objective for least square regression (two left plots) and logistic regression (two right plots).  From top to bottom: \emph{sido}, \emph{news}. Left: theoretical steps, right: steps optimized for performance after one effective pass through the data. }
\label{fig:train2}
\end{figure}

\begin{figure}

\begin{center}

\hspace*{-.5cm}
\includegraphics[scale=.43]{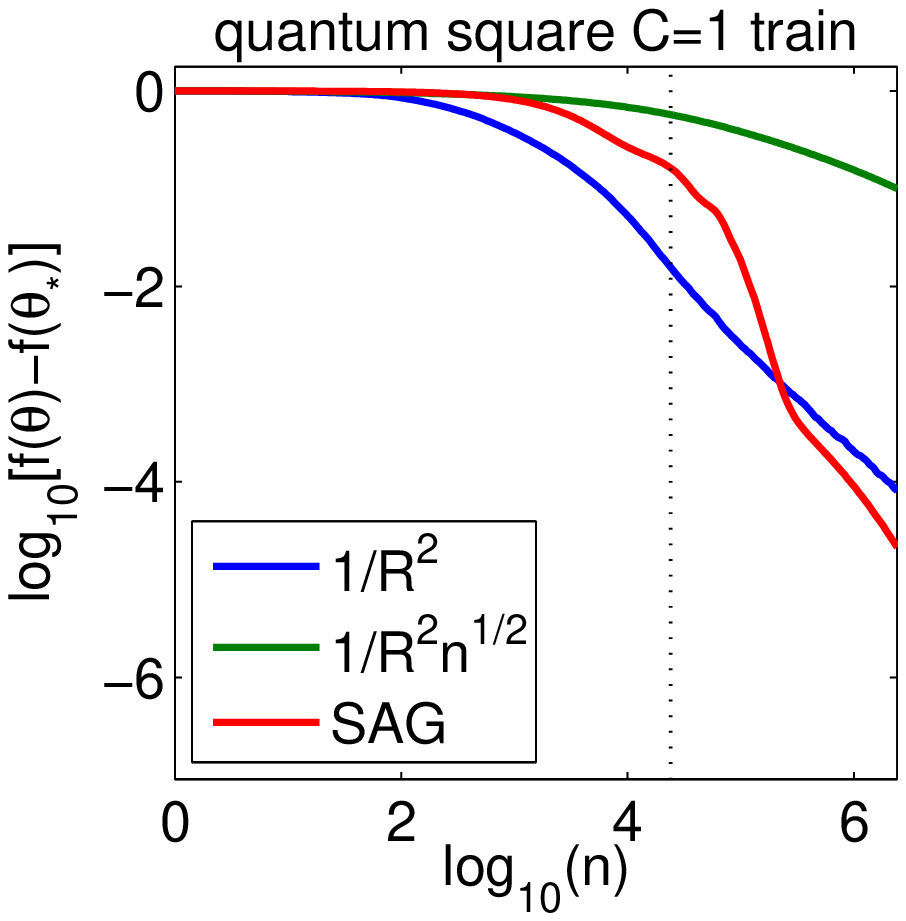}
\hspace*{-.25cm}
\includegraphics[scale=.43]{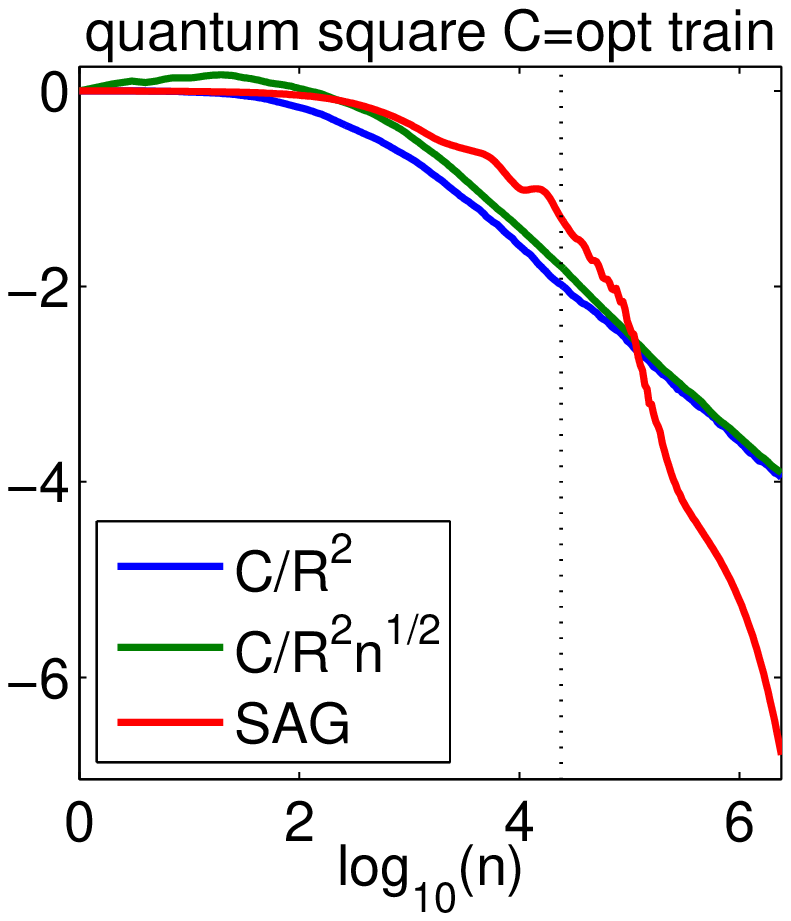}
\includegraphics[scale=.43]{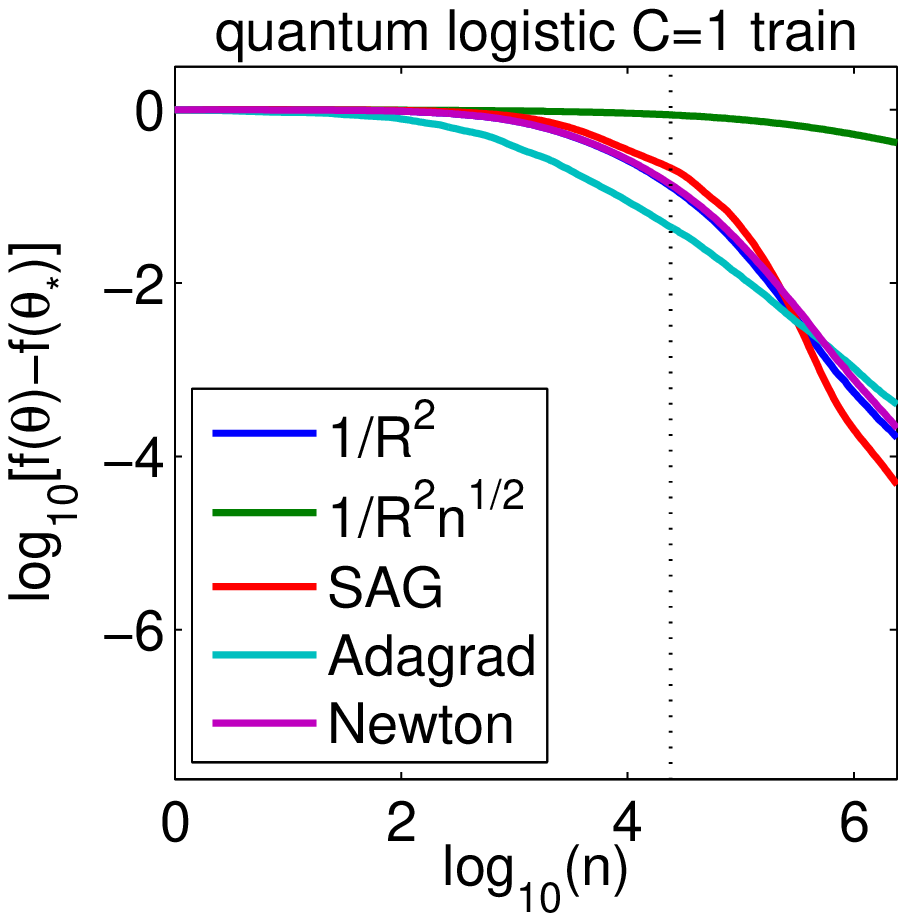}
\hspace*{-.25cm}
\includegraphics[scale=.43]{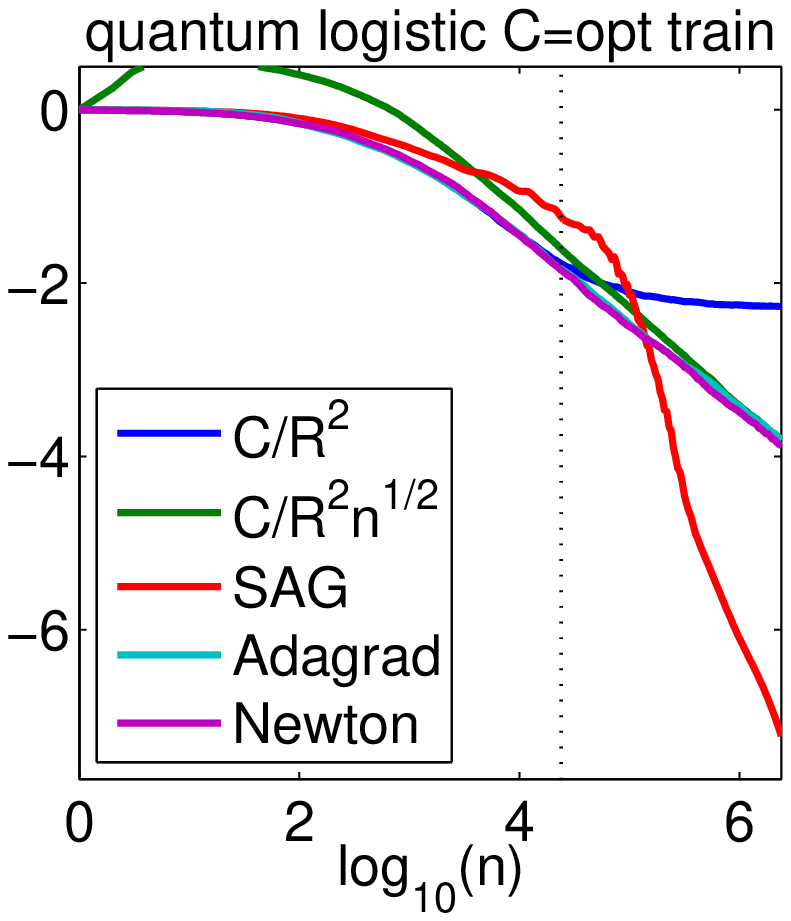}
\hspace*{-.5cm}

 \hspace*{-.5cm}
\includegraphics[scale=.43]{final_square_comparison_rcv1_train_theory_step_reduced.eps}
\hspace*{-.25cm}
\includegraphics[scale=.43]{final_square_comparison_rcv1_train_best_step_reduced.eps}
\includegraphics[scale=.43]{final_logistic_comparison_rcv1_train_theory_step_reduced.eps}
\hspace*{-.25cm}
\includegraphics[scale=.43]{final_logistic_comparison_rcv1_train_best_step_reduced.eps}
\hspace*{-.5cm}
 
\end{center}

\caption{Training objective for least square regression (two left plots) and logistic regression (two right plots). From top to bottom: \emph{quantum}, \emph{rcv1}. Left: theoretical steps, right: steps optimized for performance after one effective pass through the data. }
\label{fig:train3}
\end{figure}

\bibliography{twilight}

       \end{document}